\documentclass[twoside]{article}

\usepackage[preprint]{aistats2027}
\usepackage{amsmath, amssymb, amsthm, booktabs, comment, makecell, subfigure, stfloats}
\usepackage{xcolor, todonotes, tikz, enumitem, hyperref, pgfplots, multirow, tabularx}
\usetikzlibrary{decorations.pathreplacing}
\pgfplotsset{compat=1.18}

\usepackage{thmtools, thm-restate}

\declaretheorem[name=Proposition,numberwithin=section]{prop}
\declaretheorem[name=Assumption,numberwithin=section]{ass}
\declaretheorem[name=Corollary,numberwithin=section]{coroll}
\declaretheorem[name=Lemma,numberwithin=section]{lemma}

\declaretheorem[name=Remark,numberwithin=section]{remark}

\newenvironment{proofsketch}{%
  \proof}{\endproof}

\usepackage[ruled, vlined, linesnumbered, norelsize]{algorithm2e}
\IncMargin{1.5em}
\usepackage{macros}

\usepackage[round]{natbib}
\begin{document}



\runningauthor{Cesani, Papini, Restelli}

\twocolumn[
    \aistatstitle{Variance-Reduced Policy Gradient is Nearly Minimax-Optimal}

    \aistatsauthor{
        Leonardo Cesani\footnotemark[1]
        \And Matteo Papini
        \And Marcello Restelli
    }

    \aistatsaddress{
        Politecnico di Milano
        \And Università degli Studi di Milano
        \And Politecnico di Milano
    }
]

\footnotetext[1]{Correspondence to:
    \texttt{leonardo.cesani@polimi.it}.}

\begin{abstract}
  %We identify a family of first order policy optimization algorithms that are optimal under gradient domination for general parametric policies. They do not require Hessian-vector products nor a Fisher information matrix. We prove $\widetilde{\mathcal{O}}(\epsilon^{-2})$ last-iterate global convergence under the $1$-Polyak-Łojasiewicz condition without any assumptions on the variance of the importance weights. Our result encompasses Softmax and Fisher-non-degenerate policies. An adaptation of the classic lower bound by~\cite{azar2013minimax} to parametric policy optimization confirms that the sample complexity is optimal up to logarithmic factors.
  %Policy Gradient (PG) is a class of Reinforcement Learning algorithms widely used in continuous control problems, yet standard gradient estimators suffer from high variance and poor sample efficiency. 
  Standard Policy Gradient (PG) algorithms suffer from high estimation variance because they rely on Monte Carlo updates. While off-policy variance reduction mitigates this problem, establishing optimal sample complexity currently requires either restrictive assumptions, such as uniformly bounded importance weights, or impractical second-order operations. In this work, we introduce Loopless Variance-Reduced PG (\alg), a unified theoretical framework that abstracts a class of first-order estimators that integrate recursive updates with defensive importance sampling, including known estimators such as \texttt{PAGE-PG} and \texttt{STORM-PG}. We prove that this algorithmic class achieves $\tilde{\mathcal{O}}(\epsilon^{-2})$ last-iterate global convergence under weak gradient domination for general parametric policies. We establish this guarantee using purely first-order updates, with minimal assumptions and without gradient preconditioners. By extending state-of-the-art lower bounds to continuous parametric policies, we establish an $\Omega(\epsilon^{-2})$ sample complexity limit, confirming our framework is minimax-optimal up to logarithmic factors and horizon terms. %Finally, our experimental campaign provides a unified view on recent variance-reduced PGs, exploring the LVR-PG framework to identify the best empirical algorithm within this class.
\end{abstract}
%%%%%%%%%%%%%%%%%%%%%%%%%%%%%%%%%%%%%%%%%%%%%%%%%%%%%%%%%%%%
%%% MAIN PAPER
%%%%%%%%%%%%%%%%%%%%%%%%%%%%%%%%%%%%%%%%%%%%%%%%%%%%%%%%%%%%

\begin{table*}[t]
    \centering
    \setlength{\aboverulesep}{0pt}
    \setlength{\belowrulesep}{0pt}
    \setlength{\extrarowheight}{1pt} 
    \setlength{\tabcolsep}{7pt}
    \renewcommand{\arraystretch}{1.3}
    \caption{Comparison of sample complexities and theoretical assumptions used for variance-reduced PG algorithms. \emph{Complexity} denotes the number of trajectories required to reach the global optimal policy. \emph{Convergence Type}, last-iterate or average-iterate guarantees, \dash~denotes the necessary number of trajectories for any algorithm to return an $\epsilon$-optimal policy. \emph{Gradient Only} denotes an algorithm that does not use second-order or Fisher information. \emph{No 2nd-Order Asm.} denotes algorithms which do \emph{not} require assumptions on the Hessian matrix of the log-policy. \emph{No Bounded IW Asm.} indicates algorithms not requiring an assumption on the boundedness of the variance of the importance weights. $^{\dagger}$ The $\mathcal{O}(\cdot)$ hides dependencies on $\gamma, C_1, L_S, \chi$ and $\psi$, $\widetilde{\mathcal{O}}(\cdot)$ hides logarithmic factors. $^{\ddagger}$ Requires an additional second-order smoothness assumption.}
    \label{tab:algorithm_comparison}
    \resizebox{\textwidth}{!}{
        \begin{tabular}{|l| c c c c c|}
            \toprule
            \textbf{Algorithm} & \textbf{Complexity}$^{\dagger}$ & \makecell{\textbf{Convergence}\\ \textbf{Type}} & \makecell{\textbf{Gradient}\\ \textbf{Only}} & \makecell{\textbf{No 2nd-Order}\\ \textbf{Asm.}} & \makecell{\textbf{No Bounded}\\ \textbf{IW Asm.}} \\
            \midrule
            \gpomdp~\citep{montenegro2024learning} & $\tilde{\mathcal{O}}(\epsilon^{-3})$ & Last-iterate & \cmark & \xmark & \cmark \\
            \texttt{SRVR-PG}~\citep{liu2020improved} & $\mathcal{O}(\epsilon^{-3})$ & Average-iterate & \cmark & \xmark & \xmark \\
            \texttt{STORM-PG-S}~\citep{ding2022global} & $\mathcal{O}(\epsilon^{-3})$ & Average-iterate & \cmark & \xmark & \xmark \\
            \texttt{N-PG-IGT}$^{\ddagger}$~\citep{fatkhullin2023stochastic} & $\tilde{\mathcal{O}}(\epsilon^{-2.5})$ & Last-iterate & \cmark & \xmark & \cmark \\
            \texttt{HARPG}$^{\ddagger}$~\citep{fatkhullin2023stochastic} & $\tilde{\mathcal{O}}(\epsilon^{-2})$ & Last-iterate & \xmark & \xmark & \cmark \\
            \texttt{ANPG}~\citep{mondal2024improved} & $\mathcal{O}(\epsilon^{-2})$ & Average-iterate & \xmark & \xmark & \cmark \\
            \midrule
            \rowcolor{lightBlue!40}
            \alg \textbf{(Theorem~\ref{thm:lvr_pg_convergence})} & $\tilde{\mathcal{O}}(\epsilon^{-2})$ & Last-iterate & \cmark & \cmark & \cmark \\
            \midrule
            \textbf{Lower Bound (Theorem~\ref{thr:lower_bound})}  & $\Omega(\epsilon^{-2})$ & \dash  & \cmark & \cmark & \cmark \\
            \bottomrule
        \end{tabular}
    }
\end{table*}

%\mat{(tabella) per distinguere i metodi "veramente" first order da second-order E natural, chatgpt suggerisce "Euclidean first-order" o "gradient-only". La distinzione importante è che il nostro algoritmo usa solo il gradiente (indipendentemente da come è stimato)}
%\mat{(tabella) cambierei "bounded hessian" in "no second-order assumptions" e "bounded IW" in "no assumptions on IW variance" o simili, in modo tale che il tick sia sempre una cosa positiva}

\section{Introduction}
\label{sec:intro}

In recent years, Reinforcement Learning~\cite[RL, ][]{sutton1998reinforcement} has been used as the central tool for aligning and scaling the reasoning capabilities of Large Language Models~\cite[LLMs, ][]{brown2020language, achiam2023gpt}. Methods based on Proximal Policy Optimization~\cite[\texttt{PPO},][]{schulman2017proximal, ouyang2022training} and, more recently, Group Relative Policy Optimization~\cite[\texttt{GRPO}, ][]{shao2024deepseekmath} have achieved enormous success on reasoning tasks~\citep{guo2025deepseek} and aligning model outputs with human preferences~\citep{ziegler2019fine, bai2022training}. At their core, the algorithms powering these methods reduce to the Policy Gradient~\cite[PG, ][]{deisenroth2013survey} framework.

Despite the empirical success, their \emph{learning efficiency} is a matter of debate~\citep{rafailov2023direct}. Standard PG methods estimate the gradient of the expected return using sampled trajectories~\citep{sutton1999policy}, and the resulting Monte Carlo estimates can exhibit high estimation variance~\citep{zhao2011analysis}. Actor-critic algorithms~\citep{konda2000actor} address this difficulty through a learned critic, whose approximation error can introduce bias. For LLMs, maintaining and updating an additional critic also entails substantial memory and computational costs, motivating actor-only methods that avoid a separate critic network~\citep{yuan2023rrhf, ahmadian2024back}. Improving trajectory efficiency without introducing a learned critic is therefore relevant to scalable policy optimization.

Reducing gradient-estimation variance has long motivated theoretical research on the \emph{sample complexity} of PG methods, defined as the number of trajectories $n$ required to learn an $\epsilon$-optimal policy, where $\epsilon$ is the required accuracy. Recent actor-only applications give this established problem renewed practical relevance. A central theoretical objective is to determine whether purely first-order, actor-only methods can achieve the optimal $\Omega(\epsilon^{-2})$ sample-complexity lower bound~\citep{azar2013minimax}. However, the existing guarantees for such methods do not match this lower bound and require additional assumptions on the Hessian of the log-policy or on the variance of the importance weights.

Within the class of actor-only PG methods, variance reduction seeks to improve sample complexity by reusing historical gradient information while retaining gradient-only updates. Unlike in standard stochastic optimization, the trajectory distribution depends on the policy parameters and therefore changes over the course of optimization. To account for this distribution shift, these methods employ Importance Sampling (IS), weighting trajectory-based estimates by likelihood ratios between the target and sampling policies. Inspired by stochastic non-convex optimization, early methods such as \texttt{SVR-PG}~\citep{papini2018stochastic} and \texttt{SRVR-PG}~\citep{xu2019sample} employ double-loop architectures. These compute large-batch gradient estimates at a fixed anchor policy to correct subsequent inner-loop mini-batch updates, making the refresh schedule dependent on the chosen epoch length. Loopless methods instead incorporate historical information through recursive momentum updates or probabilistic refreshes, as in \texttt{STORM-PG}~\citep{yuan2020stochastic,ding2022global} and \texttt{PAGE-PG}~\citep{gargiani2022page}, respectively.

Existing analyses of these importance-sampling-based methods impose strong statistical assumptions, notably a uniform bound on the variance of Importance Weights (IW)~\citep{liu2020improved,ding2022global}. This condition is not generally guaranteed in continuous state-action spaces, where trajectory-level likelihood ratios accumulate the distribution shifts induced by policy updates. \cite{paczolay2024sample} addresses this by using \emph{defensive} IS~\citep{owen2000safe}, but it only studies convergence to a first-order stationary point. In contrast, \texttt{N-PG-IGT}~\citep{fatkhullin2023stochastic} circumvents the IW assumption through implicit gradient transport, retaining historical gradient information without using importance sampling. Its established sample complexity of $\tilde{\mathcal{O}}(\epsilon^{-2.5})$, however, does not attain the optimal accuracy dependence. Other approaches attain this dependence through different information mechanisms. \texttt{HARPG} uses Hessian-vector products~\citep{fatkhullin2023stochastic}, whereas \texttt{ANPG} estimates natural-gradient directions through an auxiliary stochastic optimization problem involving Fisher geometry~\citep{mondal2024improved}. Despite their theoretical appeal, these approaches require additional computational overhead beyond policy-gradient estimation, motivating the study of gradient-only methods. Their analyses also impose regularity assumptions on policy derivatives, as summarized in Table~\ref{tab:algorithm_comparison}.
%\leo{Qua dire che usare hessiane e FIM non è feasible in pratica.}

The question is therefore whether existing gradient-only PG methods can achieve last-iterate global convergence with the optimal accuracy dependence, without assuming uniformly bounded importance-weight variance or additional second-order regularity.

\paragraph{Original Contributions.}
We study a unified theoretical framework encompassing existing gradient-only PG estimators equipped with defensive IS. In the remainder of the paper, we present our main contributions as follows:
\begin{itemize}[topsep=-4pt,itemsep=0pt,parsep=0pt,partopsep=0pt]
    \item In Section~\ref{sec:limitations}, we adapt the hard-instance construction of~\citet{azar2013minimax} to continuously parameterized policies satisfying WGD, establishing that any algorithm guaranteeing expected $\epsilon$-optimality requires $\Omega(\epsilon^{-2})$ trajectories in the worst-case.
    \item In Section~\ref{sec:algorithm}, we introduce the unified \alg framework, abstracting loopless estimators like \texttt{PAGE-PG}~\citep{gargiani2022page} and \texttt{STORM-PG}~\citep{yuan2020stochastic}. By integrating \emph{defensive} importance sampling~\citep{paczolay2024sample} to control off-policy variance, we avoid assumptions on the variance of the importance weights.
    \item In Section~\ref{sec:convergence}, we prove \alg achieves last-iterate global convergence with a sample complexity of $\tilde{\mathcal{O}}(\epsilon^{-2})$, matching our lower bound up to logarithmic factors, without assuming uniformly bounded importance-weight variance or additional second-order regularity.
    \item In Section~\ref{sec:experiments}, we empirically benchmark the \alg framework against state-of-the-art baselines, identifying the defensive variant of \texttt{STORM-PG} as the dominant algorithmic configuration with superior sample efficiency. The full experimental evaluation is provided in Appendix~\ref{apx:exp}.
\end{itemize}

\section{Preliminaries}
\label{sec:preliminaries}

\paragraph{Notation.} For a measurable set $\mathcal{X}$, we denote by $\Delta_{\mathcal{X}}$ the set of probability measures over $\mathcal{X}$. For $P \in \Delta_{\mathcal{X}}$, if it admits a density, we denote it by $p$. We will use $x \sim P$ to express that the random variable $x$ is distributed according to $P$. For $n\in \mN$, we let $\dsb{n} \coloneqq \{1, \dots, n\}$. Vectors are indicated with bold letters. Given two symmetric matrices $A$ and $B$, we say that $A \succeq B$ if the matrix $A-B$ is positive semidefinite. 

\paragraph{Interaction Protocol.} We consider an agent interacting with a Markov Decision Process~\citep[MDP, ][]{puterman2014markov} $\mathcal{M} = \left(\cS, \cA, p, r, \rho_0, \gamma \right)$ over independent, indefinite-length episodes. In each timestep $h = 0, 1, \dots$, the agent observes a state $\bs_h \in \cS \subseteq \mR^{d_{\cS}}$, where $\bs_0 \sim \rho_0$ is the initial state and $\rho_0 \in \Delta_{\cS}$ is the initial-state distribution, and executes an action $\ba_h \in \cA \subseteq \mR^{d_{\cA}}$ sampled from a parametric stochastic policy $\pt:\cS\rightarrow\Delta_{\cA}$. Here, $\bt\in \Theta\subseteq\mR^{d_{\Theta}}$ is the parameter vector defining the policy space $\Pi_{\Theta}=\{\pt \mid \bt \in \Theta\}$. The agent then receives a reward $r_{h+1} \coloneqq r(\bs_h, \ba_h) \in [-\rmax, \rmax]$ and transitions to a next state $\bs_{h+1} \sim p(\cdot\mid\bs_h, \ba_h) \in \Delta_{\cS}$. For a discount factor $\gamma \in [0, 1)$, the performance of $\bt$ is assessed via the \emph{expected discounted return} under the assumption that $\bt \mapsto \pt$ is differentiable. The latter is defined as:
\begin{align*}
    J(\bt) \coloneqq \mE\left[\sum_{h\ge0}\gamma^h r(\bs_h, \ba_h) \,\bigg|\, \pt\right],
\end{align*}
where the expectation is taken over the trajectory distribution induced by the policy and the environment dynamics. In the following, we will refer to $J:\Theta\rightarrow\mathbb{R}$ as the \emph{performance function}, which is generally non-convex. Because the per-step rewards are bounded, it satisfies $|\perf| \leq \rmax/(1-\gamma)$ for all $\bt \in \Theta$. We denote the \emph{optimal} performance as $J^* \coloneqq \max_{\bt \in \Theta} \perf$.

\paragraph{Policy Gradients.} We consider learning algorithms that iteratively update a sequence of policies $\pi_{\bt_1}, \dots, \pi_{\bt_K}$. At each iteration $k$, the agent collects a dataset $\mathcal{D}_k = \{\tau_1^k, \dots, \tau^k_{n_k}\}$ of $n_k$ independent trajectories drawn from the trajectory distribution induced by the policy $p_{\bt_k}(\tau)$. To accommodate structural or safety constraints common in practical applications, we consider PG methods that restrict the parameter space to a closed convex set $\Theta \subseteq \mathbb{R}^{d_{\Theta}}$ and, similarly to \cite{xu2019sample}, update the parameters via projected stochastic gradient ascent: $\bt_{k+1} \leftarrow \mathcal{P}_\Theta\left(\bt_k + \eta \widehat{\nabla}_{\bt}J(\bt_k)\right)$. Here, $\eta > 0$ is the learning rate, and $\widehat{\nabla}_{\bt}J(\bt_k)$ is an \emph{unbiased} estimator of the true gradient $\gradK$. We formalize the optimization geometry over the parameter domain $\Theta$ using the generalized gradient mapping $\mathcal{G}_\eta(\bt) \coloneqq \frac{1}{\eta} \left( \mathcal{P}_\Theta(\bt + \eta \nabla J(\bt)) - \bt \right)$. This mapping acts as the projected gradient at $\bt$, recovering the standard gradient $\nabla_{\bt}\perf$ in the unconstrained case $\Theta = \mathbb{R}^{d_{\Theta}}$. To compute the stochastic gradient estimator $\widehat{\nabla}_{\bt}J(\bt)$ from finite interaction data without introducing approximation or truncation bias, we evaluate the infinite-horizon discounted objective using a randomized geometric stopping rule, in which the interaction with the environment is terminated after each transition with probability $1-\gamma$, independently of the interaction history. For each trajectory $\tau_i$, let $H_{\gamma, i}\in\{1,2,\ldots\}$ denote the number of executed transitions, with stopping randomness independent across trajectories. In this way, the expected length of each trajectory is $\mE[H_{\gamma, i}]=(1-\gamma)^{-1}$, and, since $\Pr(H_{\gamma, i}>h)=\gamma^h$, the expected undiscounted return accumulated before stopping matches the infinite-horizon discounted objective function $\perf$. For theoretical analysis, we adopt the geometric stopping rule to obtain unbiased estimates of the infinite-horizon objective. In the empirical evaluation, we use a fixed horizon. Using the geometric horizons $H_{\gamma,i}$, the estimate $\widehat{\nabla}_{\bt}J(\bt_k)$ is computed from the on-policy dataset $\mathcal{D}_k$ via the \gpomdp~\citep{baxter2001infinite} estimator:
\begin{align}
    \label{eq:estimator}
        \widehat{\nabla}_{\bt}J(\bt_k)
        &= \frac{1}{n_k}\sum_{i=1}^{n_k}
        \bm{g}(\tau_i \mid \bt_k),\\
        \bm{g}(\tau_i \mid \bt)
        &\coloneqq
        \sum_{h=0}^{H_{\gamma,i}-1}
        \left(
            \sum_{l=0}^{h}
            \nabla_{\bt}\log \pt(\ba_l\mid\bs_l)
        \right)r_{h+1}.\nonumber
\end{align}
While the variance of this on-policy estimator is bounded under standard regularity conditions~\citep{zhao2011analysis, papini2022smoothing}, it is often very large in practice, compromising learning efficiency.

\paragraph{Off-Policy Gradient Estimation.} To mitigate large variance of the estimator, \emph{variance-reduced} algorithms reuse historical data to improve sample efficiency, requiring the evaluation of the gradient at a \emph{target} policy $\bt$ using trajectories collected by a different \emph{behavioral} policy $\bt'$. To maintain an unbiased gradient estimate, these methods rely on Importance Sampling~\citep[IS, ][]{metelli2020importance}. Given a trajectory $\tau \sim p(\cdot \mid \bt')$, the off-policy \gpomdp estimator is defined as $\widehat{\nabla}_{\bt}J(\bt) = \frac{1}{n}\sum_{i=1}^{n} \omega(\bt, \tau_i) \bm{g}(\tau_i\mid \bt)$, where $\omega(\bt, \tau_i)$ is the importance weight defined by the trajectory density ratio:
\begin{align*}
    \omega(\bt, \tau_i) \coloneqq \frac{p_{\bt}(\tau_i)}{p_{\bt'}(\tau_i)} = \prod_{h=0}^{H_{\gamma, i}-1} \frac{\pi_{\bt}(\ba_h \mid \bs_h)}{\pi_{\bt'}(\ba_h \mid \bs_h)},
\end{align*}
where the environment transition dynamics and the trajectory's stopping probability cancel out. While this off-policy estimator is unbiased, its variance scales with the $\chi^2$-divergence between the trajectory distributions induced by $\bt$ and $\bt'$~\citep{cortes2010learning}. Since the $\chi^2$-divergence is generally unbounded, the importance-weight variance can easily diverge even for simple continuous policy classes, such as Gaussians~\citep{zhang2021convergence}.

\section{Learning Framework}
\label{sec:limitations}

To study the sample complexity of learning a \emph{globally optimal policy}, we must define the structural and geometric properties of the objective landscape and establish the statistical limits of policy optimization.

\paragraph{Assumptions.} To study variance-reduced PG algorithms within a unified framework, we define the structural and geometric properties of the objective landscape that we will use in our analysis. First, we enforce the following standard regularity conditions on the policy parametrization.

\begin{ass}[Regularity]
    \label{ass:LS}
    The policy class admits finite constants $C_1, L_S > 0$ such that, for all $\bt, \bt' \in \Theta$, and for any $\bs \in \cS$, we have:
    \begin{align*}
        &\textit{(i)}~~\sup_{\ba \in \cA}\big\| \nabla_{\bt} \log \pt(\ba \mid \bs) \big\|_2 \leq C_1\textbf{}, \\
        &\textit{(ii)}~\mathop{\mE}_{\ba\sim\pt} \left[ \big\| \nabla_{\bt} \log \pt(\ba \mid \bs) - \nabla_{\bt} \log \pi_{\bt'}(\ba \mid \bs) \big\|_2^2 \right]\\
        &\quad\quad\leq L_S^2 \|\bt - \bt'\|_2^2.
    \end{align*}
\end{ass}

Assumption~\ref{ass:LS}~\emph{(ii)}, introduced by~\cite{paczolay2024sample}, relaxes the uniform Lipschitz score condition standard in prior literature~\citep{papini2018stochastic, xu2019sample, yuan2020stochastic, gargiani2022page, fatkhullin2023stochastic} by requiring only mean-squared smoothness under the on-policy distribution, rather than a worst-case bound across all state-action pairs. Condition~\emph{(i)} requires bounded score gradients, which softmax policies satisfy, for example. Under these regularity conditions, the performance function $\perf$ and its stochastic gradient estimator are smooth and have bounded variance, respectively, as formalized in the following proposition (proof in Appendix~\ref{apx:lemmas}).

\begin{prop}
    \label{prop:properties}
    Under Assumption~\ref{ass:LS}, the following properties hold.
    \begin{align*}
        &\textit{(i)}~~\mathbb{V}\text{ar}\left(\bm{g}(\tau\mid\bt)\right) \leq \frac{4 \rmax^2 C_1^2}{(1-\gamma)^3} \eqqcolon \var;\\
        &\textit{(ii)}~\|\nabla_{\bt}^2 J(\bt)\|_2 \le \frac{R_{\max}(C_1^2 + L_S)}{(1-\gamma)^2} \eqqcolon L.
    \end{align*}
\end{prop}

Finally, to ensure global convergence and prevent the learning trajectory from getting trapped in suboptimal solutions, we impose a standard geometric assumption on the objective landscape~\citep{yuan2022general, montenegro2024learning}.

\begin{ass}[Weak Gradient Domination (WGD)]
    \label{ass:WGD}
    There exist $\chi > 0$ and $\psi \geq 0$ such that, for every $\bt \in \Theta$ and a step size $\eta \leq 1/L$, it holds that:
    \begin{align*}
        J^* - \perf \leq \chi \|\mathcal{G}_\eta(\bt)\|_2 + \psi
    \end{align*}
\end{ass}

Assumption~\ref{ass:WGD} is the standard for global convergence of constrained stochastic optimization~\citep{yuan2022general, masiha2022stochastic, ding2022global, fatkhullin2023stochastic, montenegro2024learning}. Note that in the unconstrained setting, where $\Theta = \mathbb{R}^{d_{\Theta}}$, the mapping strictly recovers the true gradient $\mathcal{G}_\eta(\bt) = \nabla_{\bt}\perf$, meaning this condition reduces to the standard WGD%\mat{dove hai definito l'acronimo WGD?} 
 assumption on the gradient norm. When $\psi = 0$, we recover the (strong) \emph{gradient domination} (GD) property: $J^* - \perf \leq \chi \|\mathcal{G}_\eta(\bt)\|_2$. GD is stricter than WGD since it requires that the constrained performance function $\perf$ has no suboptimal stationary points within the feasible domain. Instead, WGD admits local maxima as long as their performance is $\psi$-close to the global optimum~\citep{montenegro2024learning}.

\paragraph{Lower Bound for Policy Optimization.} Having established the learning framework, we now investigate the statistical limits of the policy optimization problem.

\begin{restatable}[Sample Complexity Lower Bound]{thr}{lowerBound}
    \label{thr:lower_bound}
    For any discount factor $\gamma \in [2/3, 1)$ and sufficiently small target accuracy $0 < \epsilon < \mathcal{O}\left(\frac{\rmax}{\gamma(1 - \gamma)}\right)$, there exists a class of MDPs $\mathcal{M}$ and a policy parametrization satisfying Assumptions~\ref{ass:LS} and~\ref{ass:WGD}, such that any algorithm guaranteeing an expected sub-optimality gap $J^* - \mathbb{E}[J(\bt_{K})] \leq \epsilon$ must collect a number of trajectories $n$ satisfying:
    \begin{align*}
        n &= \Omega\left( \frac{\rmax^2}{(1-\gamma)^2}\epsilon^{-2} \right).
    \end{align*}
\end{restatable}
Theorem~\ref{thr:lower_bound} establishes an $\Omega(\epsilon^{-2})$ statistical barrier under our regularity and WGD assumptions. In the context of RL, however, our result shows that these geometric properties do not trivialize the learning problem. Even when WGD holds with zero bias, so that GD excludes suboptimal stationary points, it cannot eliminate the irreducible estimation variance induced by the stochastic transition dynamics, thereby preserving the $\Omega(\epsilon^{-2})$ statistical barrier. This accuracy dependence is consistent with variance-reduced stochastic optimization under GD, where an $\mathcal{O}(\epsilon^{-2})$ sample complexity is achievable \cite{fatkhullin2022sharp}.

However, existing lower bounds do not directly show this barrier in our setting. First, lower bounds for bounded-variance stochastic-gradient oracles concern algorithms that observe stochastic gradients~\cite{masiha2026sharp}, whereas MDP interaction provides additional information through states, actions, and rewards, enabling importance-weighted gradient evaluations across different policies using the same trajectories. Second, the tabular lower bound of \cite{azar2013minimax} does not automatically apply after restricting the MDP and policy class to satisfy GD. Theorem~\ref{thr:lower_bound} addresses these limitations by establishing an $\Omega(\epsilon^{-2})$ sample complexity lower bound for policy optimization under the stated regularity and WGD assumptions.

Finally, Theorem~\ref{thr:lower_bound} establishes a statistical limit for policy optimization, independently of the update rule. Although existing methods achieve near-optimal sample complexity using Hessian or Fisher information, the lower bound does not imply that such information is necessary to reach this statistical limit.

\paragraph{Matching the Lower Bound.} Theorem~\ref{thr:lower_bound} establishes the statistical limit, leaving open whether first-order policy optimization methods attain the optimal accuracy dependence. In non-convex stochastic optimization under WGD,~\citet{masiha2026sharp} establishes an $\Omega(\epsilon^{-3})$ lower bound for the gradient-only, bounded-variance oracle model. Improving this dependence requires additional oracle structure, which recursive variance-reduction methods exploit through ``double sampling'': gradient evaluations at distinct parameters using shared randomness. In PG, implementing this coupling requires correcting the distribution shift induced by policy updates. Alternative approaches avoid importance sampling: \texttt{N-PG-IGT} provides a $\tilde{\mathcal{O}}(\epsilon^{-2.5})$ guarantee, whereas \texttt{HARPG} and \texttt{ANPG} attain the $\epsilon^{-2}$ dependence through Hessian-vector products and Fisher geometry, respectively~\citep{fatkhullin2023stochastic,mondal2024improved}. We therefore investigate whether we can establish the optimal accuracy dependence of this first-order recursive framework without relying on Hessian-vector products or Fisher-based updates, potentially without assuming uniformly bounded importance-weight variance.

%\leo{Mettere una discussione sul perchè non esiste ancora un algoritmo ottimo: cosa fanno bene e cosa non fatto bene gli algoritmi che esistono. Basarci su quello che si sa da stochastic optimization. Li, se vogliamo al sample complexity ottima, bisogna fare doppio sampling (per algoritmi del primo ordine). Questo non è necessariamente vero in PO perchè c'è molta più ottimizzazione. Gli algoritmi che esistono o fanno doppio sampling (di due o più parametri sulla stessa traiettoria, tipo SRVPG, STORM, HARP, e che non fa N-PG-IGT), oppure sfruttano la informazione del MDP in un moodo diverso, tipo i natural, in modo diverso da optimization. Ad ognuno di questi algoritmi manca qualcosa.}

\section{The \alg Framework}
\label{sec:algorithm}

To address the limitations identified in Section~\ref{sec:limitations}, we introduce \emph{Loopless Variance-Reduced Policy Gradient} (\alg), a unified framework for analyzing loopless, gradient-only PG estimators. By suitably setting its parameters, \alg recovers state-of-the-art estimators such as \texttt{STORM-PG}~\citep{yuan2020stochastic} and \texttt{PAGE-PG}~\citep{gargiani2022page}, equipped with the defensive IS scheme of~\cite{paczolay2024sample}, bounding the IW and removing the need to assume their variance to be uniformly bounded. Our goal is to establish that, with an appropriate choice of hyperparameters, an algorithm within \alg achieves last-iterate global convergence with a sample complexity matching the $\Omega(\epsilon^{-2})$ lower bound of Theorem~\ref{thr:lower_bound}, up to logarithmic factors and horizon terms.% \mat{un po' di ripetizioni qua}

\paragraph{Defensive Importance Sampling}
Classical variance-reduced PG evaluates past policies using on-policy data, leading to IW to diverge exponentially over the trajectory horizon. To avoid this blow-up in the IW, we sample trajectories from a mixture trajectory distribution and employ~\emph{defensive IS}~\citep{owen2000safe, paczolay2024sample}. Let $p_{\bt}(\tau)$ be the probability density of a trajectory $\tau$ induced by a policy $\pt$. During off-policy estimation at iteration $k$, the algorithm collects a mini-batch of $B$ trajectories. Instead of drawing these exclusively from the current policy, each trajectory is independently drawn from the target policy $\pi_{\bt_{k+1}}$ with probability $\xi \in [0, 1]$ and from the previous policy $\pi_{\bt_k}$ with probability $1 - \xi$. This yields a mixture trajectory distribution, defined as $q_{\xi}(\tau) = \xi p_{\bt_{k+1}}(\tau) + (1-\xi)p_{\bt_k}(\tau)$. To estimate the gradient at either $\bt_{k+1}$ or $\bt_k$ using trajectories sampled from $q_{\xi}$, we use the defensive IS weight $\omega_{\xi}(\bt,\tau)\coloneqq p_{\bt}(\tau)/q_{\xi}(\tau)$, where $\bt\in\{\bt_{k+1},\bt_k\}$. For $\xi\in(0,1)$, each target trajectory density enters $q_\xi$ with a positive coefficient, so its importance weight is bounded by the reciprocal of that coefficient. In particular, when evaluating the current policy, we have $\omega_{\xi}(\bt_{k+1}, \tau) \leq 1/\xi$, and similarly $\omega_{\xi}(\bt_{k}, \tau) \leq 1/(1-\xi)$. One can easily show that the gradient estimate computed using this mixture distribution remains unbiased.

\paragraph{The Unified Estimator.} Building upon the loopless variance reduction principles introduced in~\cite{yuan2020stochastic} and~\cite{gargiani2022page}, our framework regulates full-batch and mini-batch updates through a randomized mechanism governed by $p \in [0, 1]$, which we call \emph{refresh probability} parameter. At iteration $k \in \dsb{K}$, the evaluation of the target policy $\bt_{k+1}$ and the construction of the new gradient estimator $\estGradKone$ depend on a Bernoulli realization $c_k \sim \mathcal{B}(p)$. If $c_k = 1$, the framework executes a full refresh, sampling $N$ trajectories directly from the target policy $\pi_{\bt_{k+1}}$ to compute a full-batch gradient estimate with Equation~\eqref{eq:estimator} with $n_k = N$. If $c_k = 0$, the framework bypasses the full-batch estimation in favor of a recursive mini-batch update. To evaluate the gradients of both $\pi_{\bt_{k+1}}$ and $\pi_{\bt_k}$ simultaneously, \alg collects $B \ll N$ trajectories independently drawn from the mixture behavior distribution $q_\xi$. This allows us to compute the Defensive IS operator $\omega_{\xi}(\bt, \tau)$ via the density ratio. Using this mini-batch and the mixture IS operator, the gradient estimator is recursively updated using a \emph{momentum} parameter $\alpha \in [0, 1]$ to control off-policy variance:
\begin{equation}
    \label{eq:recursive_estimator}
    \begin{split}
        \estGradKone &= \bm{g}_{B}^{\omega_{\xi}}(\bt_{k+1})\\&\quad + (1-\alpha) \left[ \estGradK - \bm{g}_{B}^{\omega_{\xi}}(\bt_{k})\right],
    \end{split}
\end{equation}
where $\bm{g}_{B}^{\omega_{\xi}}(\bt) \coloneqq \frac{1}{B} \sum_{i=1}^B \omega_{\xi}(\bt, \tau_i)\bm{g}(\tau_i \mid \bt)$. The momentum parameter $\alpha$ regulates the degree to which past gradient errors are discounted, and enables to control of the variance-reduction strategy to use: setting $\alpha = 0$ and $p \in [0, 1]$ recovers the update rule of \texttt{PAGE-PG}~\citep{gargiani2022page}, while setting $\alpha > 0$ yields the momentum update of \texttt{STORM-PG}~\citep{yuan2020stochastic}. In the rest of the paper, we will refer to the defensive versions of these two algorithms as \defpagepg and \defstormpg, respectively. Finally, once we compute the variance-reduced estimator $\widehat{\nabla}_{\bt}^\dagger J(\bt_k)$, we update the policy parameters via projected stochastic gradient ascent. The algorithm aims to find a parameter $\bt \in \Theta$ that maximizes the performance function subject to this constraint. 
\begin{algorithm}[t!]
    \caption{\alg}
    \label{alg:main_uvr_pg}
    \SetKwInput{Input}{Input}
    \SetKw{Return}{Return}

    \Input{$K, N, B, p, \alpha, \xi, \eta, \Theta, \bt_0$}
    
    Init. $\widehat{\nabla}_{\bt}J(\bt_0)$ using $N$ trajectories from $\pi_{\bt_0}$\;
    
    \For{$k = 0, \dots, K-1$}{
        $\bt_{k+1} \leftarrow \mathcal{P}_\Theta \left( \bt_k + \eta \estGradK \right)$\;
        Draw refresh variable $c_k \sim \mathcal{B}(p)$\;
        
        \eIf{$c_k = 1$}{
            Sample $N$ trajectories $\tau_{1:N} \sim \pi_{\bt_{k+1}}$\;
            Estimate gradient with Equation~\eqref{eq:estimator}\;
        }{
            Sample $B$ trajectories $\tau_{1:B}$ from $q_{\xi}$\;
            Estimate gradient with Equation~\eqref{eq:recursive_estimator}\;
        }
    }
    \vspace{1mm}
    \Return $\bt_K$
\end{algorithm}

\begin{comment}
\begin{table*}[t]
\centering
{% Open local group so changes do not leak
    \setlength{\aboverulesep}{0pt}
    \setlength{\belowrulesep}{0pt}
    \setlength{\extrarowheight}{1pt} 
    \setlength{\tabcolsep}{10pt} 
    \renewcommand{\arraystretch}{1.3} 
    \caption{Hyperparameter configurations for \alg and its instantiations. To achieve the minimax optimal sample complexity, all configurations share a constant mini-batch size $B \ge 1$ and an initial batch size $N_0 = \Theta(\epsilon^{-2})$. For \alg, the parameters $p$ and $\alpha$ are bounded from above but must jointly satisfy the minimum decay threshold $p + \alpha = \Omega(B\epsilon^2)$. The corners of this continuous area recovers both \defpagepg and \defstormpg.}
    \label{tab:convergence_instantiations}
    \begin{tabular}{|l|cccc|>{\columncolor{lightBlue!40}}c|}
        \toprule
        \textbf{Algorithm} & \textbf{$N$} & \textbf{$p$} & \textbf{$\alpha$} & \textbf{$\eta$} & \textbf{Complexity} \\
        \midrule
        \alg & $\Theta(\epsilon^{-2})$ & $\mathcal{O}(B\epsilon^2)$ & $\mathcal{O}(B\epsilon^2)$ & $\Theta(B\epsilon)$ & \\
        \defpagepg & $\Theta(\epsilon^{-2})$ & $\Theta(\epsilon^2)$ & $0$ & $\Theta(\epsilon)$ & \\
        \defstormpg & $-$ & $0$ & $\Theta(\epsilon^2)$ & $\Theta(\epsilon)$ & \large\multirow{-3}{*}{$\widetilde{\mathcal{O}}(\epsilon^{-2})$} \\
        \bottomrule
    \end{tabular}
}
\end{table*}
\end{comment}

\section{Convergence and Optimality}
\label{sec:convergence}

In this section, we establish the sample complexity required by an instantiation of the \alg framework to obtain an $\epsilon$-approximate global optimum. Specifically, we demonstrate that the \alg framework defines a class of minimax--optimal first-order algorithms, up to logarithmic factors. The proofs for the results presented in this section are deferred to Appendix~\ref{apx:convergence}.

\begin{restatable}[Global Convergence of \alg]{thr}{algConvergence}
    \label{thm:lvr_pg_convergence}
    Under Assumptions~\ref{ass:LS} and~\ref{ass:WGD}, for any target accuracy $\epsilon > 0$ and  fixed mini-batch size $B \ge 1$, let the batch sizes be $N, N_0 = \Theta(\epsilon^{-2})$ and the learning rate be $\eta = \Theta(B\epsilon)$. For appropriately chosen  hyperparameters satisfying $p, \alpha = \mathcal{O}(B\epsilon^2)$ and $p + \alpha = \Omega(B\epsilon^2)$ and choosing $\xi=1/2$, \alg guarantees an expected suboptimality gap $J^* - \mE[J(\bt_K)] \le \psi + \epsilon$ with a total expected sample complexity of at most:
    \begin{align*}
        %N_0 + K \left(pN + (1-p)B\right) = \widetilde{\mathcal{O}}\left( \frac{R_{\max}^2}{\mu (1-\gamma)^4 \epsilon^2} \right),
        n = \widetilde{\mathcal{O}}\left( \frac{\chi^3 R_{\max}^2}{(1-\gamma)^4 \epsilon^2} \right),
    \end{align*}
    where $\widetilde{\mathcal{O}}(\cdot)$ hides logarithmic factors.
\end{restatable}
\begin{proofsketch}
    The proof establishes the minimax-optimal sample complexity for \alg by coupling the policy's suboptimality descent with the variance recurrence of the gradient estimator.

    \textbf{Part I: The Descent Mechanics.} We begin by characterizing the estimation error $\Delta_k \coloneqq \mE[\|\estGradK - \gradK\|_2^2]$. Lemma~\ref{lemma:lvr_pg_variance} quantifies how this error propagates through recursive updates. Letting $\lambda \coloneqq 1-(1-p)(1-\alpha)^2$ and $\Psi_k \coloneqq \mE\left[\|\bt_{k+1} - \bt_k\|_2^2\right]$ we have:
    \begin{align*}
        \Delta_{k+1} \le (1-\lambda) \Delta_k + \frac{2(1-\lambda)\Calg}{B}\Psi_k  + C_{\nu}.
    \end{align*}
    Here, $\Calg>0$ is the mean-squared smoothness constant from Lemma~\ref{lemma:consecutive_etstimation_bound}, which controls the estimation error induced by policy updates. The additive variance term $C_\nu$ combines the contributions of full-batch refreshes and recursive mini-batch updates, weighted by $p$ and $\alpha^2$, respectively. From the $L$-smoothness of the objective obtained in Proposition~\ref{prop:properties}, letting $r_k \coloneqq \mE [\tilde{\constrPerf}^* - \constrPerf(\bt_k)]$, we obtain the recursion of the suboptimality gap:
    \begin{align*}
        r_{k+1} \leq r_k - \frac{\eta}{8} \mE\left[\|\mathcal{G}_\eta(\bt_k)\|_2^2\right] + \frac{3\eta}{4}\Delta_k - \left(\frac{1}{4\eta} - \frac{L}{2}\right)\Psi_k.
    \end{align*}
    To establish convergence, we study the potential $\Phi_k \coloneqq r_k + c\Delta_k$, combining suboptimality and estimation error. Choosing $c = \frac{3\eta}{2\lambda}$ and a sufficiently small step size $\eta$ yields a recursion in which both objective improvement and estimator-error decay contribute to reducing $\Phi_k$.
    
    \textbf{Part II: WGD and Recurrence.} Because $\|\mathcal{G}_\eta(\bt_k)\|_2^2$ provides only local stationarity guarantees, we use Assumption~\ref{ass:WGD} to bound the global suboptimality gap directly by the gradient mapping norm. Evaluating the invariant sequence $x_k \coloneqq \max\{0, \Phi_k\}$ reveals a dual-regime dynamical system:
    \begin{align*}
        x_{k+1} \leq x_k - \min \left\{ \frac{\mu\eta}{32}x_k^2, \frac{\lambda}{4}x_k \right\} + \frac{3\eta C_{\nu}}{2\lambda}.
    \end{align*}
    This recurrence yields linear contraction for large values of the Lyapunov potential and a WGD-induced quadratic decrease for smaller values. The additive term $3\eta C_\nu/(2\lambda)$ accounts for stochastic gradient-estimation error. When the additive term is sufficiently small, Lemma~\ref{lemma:linear_recurrence} bounds the iterations required to reach and remain within the target accuracy:
    \begin{align*}
        K \ge \left\lceil\max \left\{ \frac{32}{\mu\eta\epsilon}, \frac{4}{\lambda} \right\} \log \left(\frac{2\Phi_0}{\epsilon}\right)\right\rceil.
    \end{align*} 
    
   \textbf{Part III: Selection of $\alpha, p$ and $\eta$.} We now choose the hyperparameters to satisfy the accuracy condition $C_\nu = \mathcal{O}(\mu\lambda\epsilon^2)$ while controlling the sample cost. Setting $N = \Theta(\epsilon^{-2})$ and $p = \mathcal{O}(B\epsilon^2)$ keeps the expected cost of full-batch refreshes comparable to that of mini-batch updates. Choosing $\alpha = \mathcal{O}(B\epsilon^2)$ controls the momentum-dependent contribution to $C_\nu$, while requiring $p+\alpha = \Omega(B\epsilon^2)$ ensures sufficient estimator-error decay through $\lambda = \Omega(B\epsilon^2)$. With a suitable step size $\eta = \Theta(B\epsilon)$, these choices yield $K = \tilde{\mathcal{O}}((B\epsilon^2)^{-1})$ iterations and a total sample complexity $n = \tilde{\mathcal{O}}(\epsilon^{-2})$.
\end{proofsketch}
%\mat{non riusciamo a mostrare anche la dipendenza da $C_1$ e $L_S$ ?}

Theorem~\ref{thm:lvr_pg_convergence} establishes that \alg attains the minimax-optimal sample complexity in the learning framework described in Section~\ref{sec:limitations}, up to logarithmic factors and horizon terms, using exclusively first-order updates. Notably, this rate is achieved using a fixed step size $\eta = \Theta(B\epsilon)$, and the convergence guarantee applies directly to the last-iterate parameter $\bt_K$, bypassing the need for trajectory averaging or random iterate selection commonly required in previous works~\citep{liu2020improved, ding2022global}. Furthermore, the analysis identifies a continuous two-dimensional region of hyperparameter configurations $(p,\alpha)$ for which the nearly optimal sample complexity is guaranteed. As illustrated in Figure~\ref{fig:design_space}, the conditions $p=\mathcal{O}(B\epsilon^2)$, $\alpha=\mathcal{O}(B\epsilon^2)$, and
$p+\alpha=\Omega(B\epsilon^2)$, together with the batch-size and learning-rate choices established in the proof, provide sufficient conditions for this guarantee.

\begin{remark}[Unification of Existing Algorithms]
    Existing variance-reduced estimators emerge as boundary limits of this optimal space: setting $\alpha = 0$ or $p = 0$ recovers \defpagepg and \defstormpg, respectively, both strictly preserving the $\widetilde{\mathcal{O}}(\epsilon^{-2})$ sample complexity. Setting $p = 1$ forces full-batch evaluations at every step, recovering vanilla \gpomdp. This choice violates the computational boundary $p = \mathcal{O}(B\epsilon^2)$, shifting the algorithm outside the optimal manifold and correctly degrading its complexity to the standard $\widetilde{\mathcal{O}}(\epsilon^{-3})$ rate \citep{montenegro2024learning}.
\end{remark}

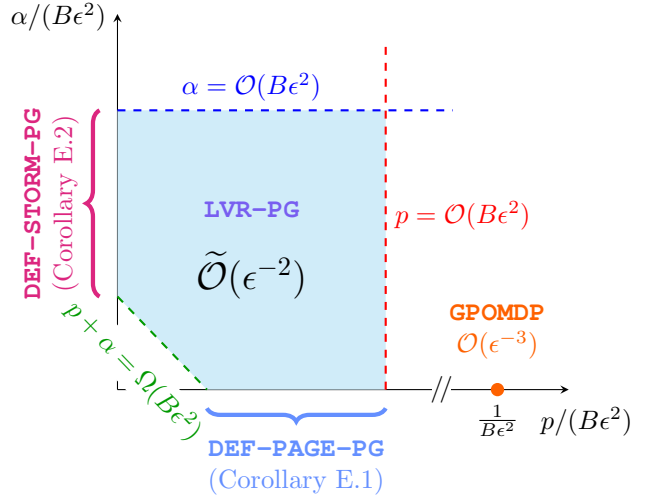
\begin{figure}[t]
    \centering
    \begin{tikzpicture}
        \begin{axis}[
            width=7.2cm,  
            height=6.2cm, 
            xmin=0, xmax=1.9, 
            ymin=0, ymax=1.5,
            axis lines=left,
            xlabel={$p / (B\epsilon^2)$},
            ylabel={$\alpha / (B\epsilon^2)$},
            xlabel style={at={(axis description cs:0.85,0)}, anchor=west, xshift=20pt, yshift=-12pt},
            ylabel style={at={(axis description cs:0.08,1)}, anchor=south, yshift=35pt, rotate=-90},
            xtick={1.7},
            xticklabels={$\frac{1}{B\epsilon^2}$},
            ytick=\empty,
            clip=false,
            axis line style={shorten >=-10pt}
        ]
        
        % 1. Feasible region for LVR-PG 
        \addplot [fill=lightBlue!40, draw=none, forget plot] coordinates {
            (0, 0.4) (0, 1.2) (1.2, 1.2) (1.2, 0) (0.4, 0)
        };
        
        % 1b. Central Text 
        \node[align=center] at (axis cs:0.6, 0.6) {\textbf{\alg} \\\\ \Large $\widetilde{\mathcal{O}}(\epsilon^{-2})$};

        % 2. Minimum Decay Line (dashed, no marks)
        \draw[thick, dashed, green!60!black]
        (axis cs:0, 0.4) -- (axis cs:0.4, 0)
        node[midway, sloped, below=7pt, font=\normalsize,fill=white, inner sep=2pt]
        {$p + \alpha = \Omega(B\epsilon^2)$};
        
        % 3. Momentum Bound (dashed, no marks)
        \addplot [thick, dashed, blue, mark=none] coordinates {
            (0, 1.2) (1.5, 1.2)
        };
        \node[blue, anchor=south] at (axis cs:0.6, 1.2) {$\alpha = \mathcal{O}(B\epsilon^2)$};
        
        % 4. Refresh Bound (dashed, no marks)
        \addplot [thick, dashed, red, mark=none] coordinates {
            (1.2, 0) (1.2, 1.5)
        };
        \node[red, anchor=west] at (axis cs:1.2, 0.75) {$p = \mathcal{O}(B\epsilon^2)$};
        
        % 5. STORM-PG and PAGE-PG Corollaries
        % STORM-PG: 
        \draw[ultra thick, defstorm_col, decorate, decoration={brace, amplitude=5pt}] 
            (axis cs:-0.08, 0.4) -- (axis cs:-0.08, 1.2) 
            node[midway, xshift=-18pt, align=center, defstorm_col, rotate=90] {\textbf{\defstormpg} \\ (Corollary~\ref*{cor:def_storm_pg})};
            
        % PAGE-PG: Below the x-axis
        \draw[ultra thick, defpage_col, decorate, decoration={brace, amplitude=5pt, mirror}] 
            (axis cs:0.4, -0.08) -- (axis cs:1.2, -0.08) 
            node[midway, below=8pt, align=center, defpage_col] {\textbf{\defpagepg} \\ (Corollary~\ref*{cor:def_page_pg})};
        
        % 6. GPOMDP Point (p=1)
        % Axis break before the GPOMDP configuration
        \node[fill=white, inner sep=1pt, text=black]
            at (axis cs:1.45, 0) {$\!/\!/\!$};
        
        % GPOMDP outside the shaded region
        \fill[gpomdp_col] (axis cs:1.7, 0) circle (2.5pt);
        
        \node[above=8pt, align=center, gpomdp_col]
            at (axis cs:1.7, 0)
            {\textbf{\gpomdp} \\ $\mathcal{O}(\epsilon^{-3})$};
        
        \end{axis}
    \end{tikzpicture}
    \caption{The continuous algorithmic design space of \alg in normalized coordinates. The shaded region delineates the manifold of hyperparameter configurations $(p, \alpha)$ that satisfy the stability thresholds derived in Theorem~\ref{thm:lvr_pg_convergence}. In this region, a configuration of \alg is nearly optimal. \gpomdp falls out of this region, since $p = 1$.}
    \label{fig:design_space}
\end{figure}

\subsection{On Fisher-Non-Degeneracy}
We specialize Theorem~\ref{thm:lvr_pg_convergence} to Fisher-Non-Degenerate (FND) policies in the unconstrained setting $\Theta=\mathbb{R}^{d_\Theta}$. This setting allows us to apply the ordinary-gradient domination inequality induced by FND under suitable approximation and concentrability assumptions~\citep{liu2020improved, ding2022global, fatkhullin2023stochastic}, since $\mathcal{G}_\eta(\bt)=\nabla_{\bt}\perf$. Although FND can also be imposed on constrained parameter sets, projection can remove gradient components, so the gradient bound does not imply the mapping-based WGD condition required by Assumption~\ref{ass:WGD}.

%In general, assuming an unconstrained space $\Theta = \mathbb{R}^{d_\Theta}$, when \textit{(i)} the FIM induced by a policy $\pt$ is \emph{non-degenerate} for every $\bt \in \mathbb{R}^{d_\Theta}$, i.e., $\bm{F}(\bt) = \mE_{\pt}\left[\nabla_{\bt}\log \pt(\ba \mid \bs)\nabla_{\bt}\log \pt(\ba \mid \bs)^{\top}\right] \succeq \mu_F \bm{I}$ for some $\mu_{F}>0$ and \textit{(ii)} a \emph{compatible function approximation error}~\citep{agarwal2021theory}, i.e., $\mE_{\pt}\left[\left(A^{\pt}(\bs, \ba) - (1-\gamma)\bm{w}^{\top}\nabla_{\bt}\log \pt(\ba \mid \bs)\right)^2\right] \leq \epsilon_{\text{bias}}$, being $\bm{w} = \bm{F}(\bt)^{-1}\nabla_{\bt}\perf$ and $A^{\pt}$ the advantage function, WGD holds with $\chi = C_1 / \mu_F$ and $\psi = \sqrt{\epsilon_{\text{bias}}}\kappa/(1-\gamma)$, where $\kappa = \left\|\nu^{\pi^*}/\nu^{\pt}\right\|_\infty < \infty$ is a concentrability coefficient and $\nu^{\pt}$ is the state-action distribution induced by the policy $\pt$. \mat{frase lunghissima}
In the unconstrained setting $\Theta=\mathbb{R}^{d_\Theta}$, when the FIM induced by $\pt$ is uniformly non-degenerate, i.e., $\bm{F}(\bt)=\mE_{\pt}[\nabla_{\bt}\log\pt(\ba\mid\bs)\nabla_{\bt}\log\pt(\ba\mid\bs)^\top]\succeq\mu_F\bm{I}$ for some $\mu_F>0$ and every $\bt\in\Theta$, WGD follows provided that the compatible function approximation error~\citep{agarwal2021theory} is bounded. This condition requires $\mE_{\pt}[(A^{\pt}(\bs,\ba)-(1-\gamma)\bm{w}^\top\nabla_{\bt}\log\pt(\ba\mid\bs))^2]\leq\epsilon_{\text{bias}}$, where $\bm{w}=\bm{F}(\bt)^{-1}\nabla_{\bt}\perf$ and $A^{\pt}$ is the advantage function of the policy $\pt$. Under these conditions, WGD holds with $\chi=C_1/\mu_F$ and $\psi=\sqrt{\epsilon_{\text{bias}}}\kappa/(1-\gamma)$. Here, $\kappa=\sup_{\bt \in \mathbb{R}^{d_\Theta}} \|\nu^{\pi^*}/\nu^{\pt}\|_\infty$ is a concentrability coefficient, and $\nu^{\pt}$ denotes the state-action distribution induced by $\pt$. Note that, differently from previous works \citep{ding2022global, fatkhullin2023stochastic}, we define the expected compatible function approximation error with respect to the current state-action distribution $\nu^{\pi_{\bt}}$ rather than the optimal distribution $\nu^{\pi^*}$. Bounding the error under $\nu^{\pi^*}$ is unverifiable because the optimal policy is unknown a priori. Conversely, evaluating this error under $\nu^{\pi_{\bt}}$ requires the parametrization to be expressive along the realized optimization trajectory, a verifiable formulation that inherently incurs the statistical penalty $\kappa$ within the irreducible bias.

\begin{restatable}[Sample complexity under FND]{coroll}{sampleFND}
\label{cor:convergence_FND}
    In the unconstrained setting where $\Theta = \mathbb{R}^{d_\Theta}$, letting $\chi = C_1 / \mu_F$ and $\psi = \sqrt{\epsilon_{\text{bias}}}\kappa/(1-\gamma)$, assuming $\kappa < \infty$, \alg achieves a suboptimality gap $J^* - \mE[J(\bt_K)] \leq \frac{\sqrt{\epsilon_{\text{bias}}}\kappa}{1-\gamma} + \epsilon$ with a total sample complexity of:
    \begin{align*}
        n = \widetilde{\mathcal{O}}\left( \frac{R_{\max}^2 C_1^3}{(1-\gamma)^4 \mu_F^3 \epsilon^2} \right).
    \end{align*}

\end{restatable}

Corollary~\ref{cor:convergence_FND} confirms that when policies satisfy FND, the minimax optimal sample complexity established by Theorem~\ref{thm:lvr_pg_convergence} is robustly preserved without degradation. We refer the reader to Appendix~\ref{apx:FND} for a detailed discussion.

\begin{figure*}[t]
    \centering
    
    \begin{minipage}[b]{0.41\textwidth}
        \centering
        \includegraphics[width=\textwidth]{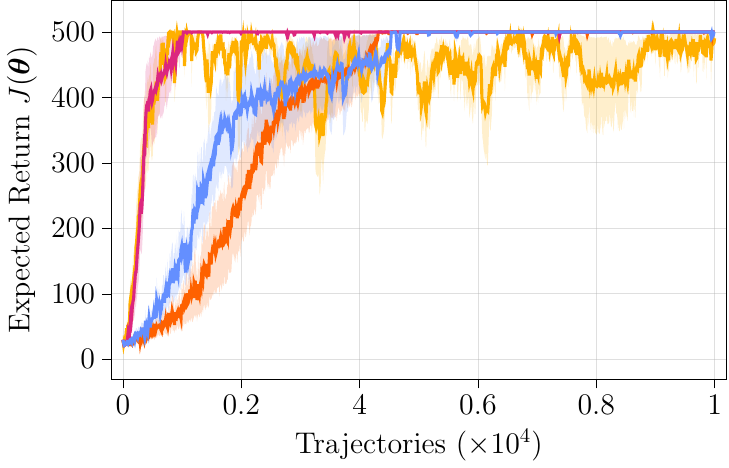}
        \centerline{(a) CartPole-v1 }
        \label{fig:lqr_return}
    \end{minipage}
    \hspace{5em}
    \begin{minipage}[b]{0.41\textwidth}
        \centering
        \includegraphics[width=\textwidth]{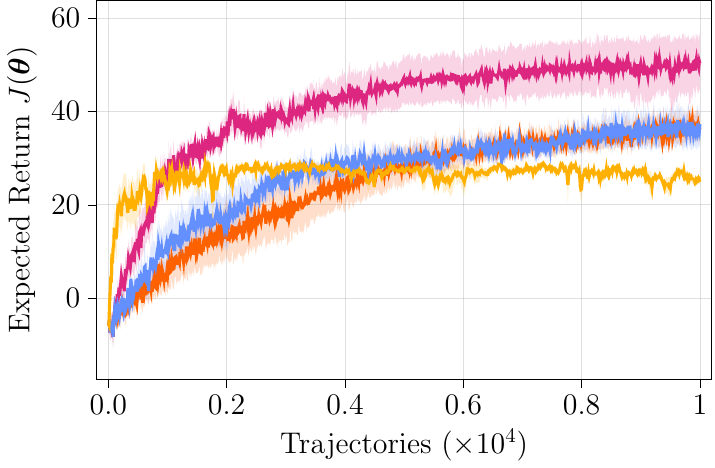}
        \centerline{(b) Swimmer-v4}
        \label{fig:lqr_variance}
    \end{minipage}

    %\vspace{0.5em} % Space between grid and unified legend
    
    % --- Unified Legend ---
    \centerline{
        %\legendline{lvr_col} \alg \hspace{1.5em}
        \legendline{defstorm_col} \defstormpg \hspace{1.5em}
        \legendline{defpage_col} \defpagepg \hspace{1.5em}
        \legendline{npg_col} \texttt{N-PG-IGT} \hspace{1.5em}
        \legendline{gpomdp_col} \gpomdp    
    }
    
    %\vspace{0.5em}
    
    \caption{Comparison of \defstormpg, \defpagepg, \texttt{N-PG-IGT}, and \gpomdp. Shaded regions indicate the estimated standard error of the mean.}
    \label{fig:exp_main}
\end{figure*}
%\leo{Plot: magari mettere nel plot STORM-PG con la nomenclatura di LVR-PG, e specificare che è la versione STORM. Evitare di mettere PAGE, perchè è sempre peggio. NElla classe che abbiamo definito scegliamo il migliore, e lo confrontiamo con cose che sono fuori dalla classe. Dire che l'intervallo è un bootsrap, che servono epr far vedere come si distribuiscono le performance.}

\section{Experimental Results}
\label{sec:experiments}

To evaluate the empirical sample efficiency of our framework, we consider CartPole-v1 and Swimmer-v4, covering classic control and MuJoCo locomotion tasks~\cite{todorov2012mujoco}. The comparisons, hyperparameter ablations, and grid search reported in Appendix~\ref{apx:exp} motivate selecting the \texttt{STORM} gradient estimator as the primary representative of \alg, given its better sample efficiency. This configuration uses momentum-based recursive gradient estimates with defensive IS and no full-batch refreshes ($p=0$, $\alpha>0$). Figure~\ref{fig:exp_main} also includes \defpagepg for completeness, showing its lower sample efficiency in the tested settings. Hybrid configurations are examined in Appendix~\ref{apx:exp}. We compare \defstormpg with \gpomdp, which uses fresh on-policy gradient estimates, and \texttt{N-PG-IGT}, which combines normalized updates with implicit gradient transport without importance sampling. These comparisons assess whether recursive variance reduction with defensive IS improves empirical sample efficiency relative to the two baselines. Across both environments, \defstormpg reaches high returns with fewer trajectories than \gpomdp and \defpagepg. On CartPole-v1 \defstormpg exhibits fewer oscillations than \texttt{N-PG-IGT}, while in Swimmer-v4, the latter converges quickly to a local optimum.

\section{Conclusions}
\label{sec:conclusions}

We established that purely first-order PG methods can achieve minimax-optimal sample complexity up to logarithmic factors and horizon terms, with minimal assumptions and without second-order assumptions. %We introduced the \alg framework, leveraging defensive IS, and proved that it attains a last-iterate sample complexity of $\tilde{\mathcal{O}}(\epsilon^{-2})$ under WGD.
Our findings suggest several directions for future research. First, the defensive IS mechanism requires evaluating trajectories under a mixture distribution, which inherently requires storing both the current and previous policy parameters. Developing variance-reduction mechanisms that avoid this explicit double sampling would significantly reduce memory overhead. Second, extending the \alg theoretical framework to encompass adaptive gradient methods would bridge the gap between our convergence guarantees and the practical optimization dynamics prevalent in deep RL.

\clearpage

%%%%%%%%%%%%%%%%%%%%%%%%%%%%%%%%%%%%%%%%%%%%%%%%%%%%%%%%%%%%
%%% AI STATEMENT
%%%%%%%%%%%%%%%%%%%%%%%%%%%%%%%%%%%%%%%%%%%%%%%%%%%%%%%%%%%%
\subsection*{AI use statement}

%(This section is \textbf{required} and does not count toward the page limit.)

In this work, we used generative AI tools for text polishing, proof assistance, and code generation.
We have not used generative AI tools to generate research ideas or to conduct the literature review, and generating figures or data synthesis is not applicable to this work.
Additionally, we used generative AI tools to correct grammar and improve readability.
We have reviewed all AI-assisted work.
Specifically, the core ideas and initial proofs were developed manually by the authors, and AI was subsequently used to polish the presentation of the mathematical proofs. All AI-assisted proofs were checked line by line, and the authors thoroughly reviewed and tested any AI-assisted code for correctness.
We take responsibility for the final content of this work, including text, claims, or artifacts produced with the aid of generative AI.

%See the AISTATS 2027 AI Policy for Authors for more details. This statement should not be more than 1 page.

%\subsubsection*{Acknowledgements}
%All acknowledgements go at the end of the paper, including thanks to reviewers who gave useful comments, to colleagues who contributed to the ideas, and to funding agencies and corporate sponsors that provided financial support.
%To preserve the anonymity, please include acknowledgements \emph{only} in the camera-ready papers. The acknowledgements do not count against the 9-page page limit in the camera-ready.

%%%%%%%%%%%%%%%%%%%%%%%%%%%%%%%%%%%%%%%%%%%%%%%%%%%%%%%%%%%%
%%% BIBLIOGRAPHY
%%%%%%%%%%%%%%%%%%%%%%%%%%%%%%%%%%%%%%%%%%%%%%%%%%%%%%%%%%%%
\renewcommand{\bibname}{References}
\renewcommand{\bibsection}{\subsubsection*{\bibname}}
\bibliography{bibliography}

%%%%%%%%%%%%%%%%%%%%%%%%%%%%%%%%%%%%%%%%%%%%%%%%%%%%%%%%%%%%
%%% CHECKLIST
%%%%%%%%%%%%%%%%%%%%%%%%%%%%%%%%%%%%%%%%%%%%%%%%%%%%%%%%%%%%

\section*{Checklist}

\begin{enumerate}

  \item For all models and algorithms presented, check if you include:
  \begin{enumerate}
    \item A clear description of the mathematical setting, assumptions, algorithm, and/or model. Yes, the paper describes the framework in Sections~\ref{sec:preliminaries},~\ref{sec:limitations} and~\ref{sec:algorithm}.
    \item An analysis of the properties and complexity (time, space, sample size) of any algorithm. Yes, the paper proves the complexity properties of a class of algorithms.
    \item (Optional) Anonymized source code, with specification of all dependencies, including external libraries. Yes, the anonymized code is provided in the supplemental material.
  \end{enumerate}

  \item For any theoretical claim, check if you include:
  \begin{enumerate}
    \item Statements of the full set of assumptions of all theoretical results. Yes, we define the assumptions in Section~\ref{sec:limitations}.
    \item Complete proofs of all theoretical results. Yes, the proofs of the theoretical results are provided in the appendix
    \item Clear explanations of any assumptions. Yes, the explanation of the assumptions is provided in both the main paper and the appendix.
  \end{enumerate}

  \item For all figures and tables that present empirical results, check if you include:
  \begin{enumerate}
    \item The code, data, and instructions needed to reproduce the main experimental results (either in the supplemental material or as a URL). Yes, the code to run the experiments is in the supplementary material.
    \item All the training details (e.g., data splits, hyperparameters, how they were chosen). Yes, the training details are provided in Appendix~\ref{apx:exp}
    \item A clear definition of the specific measure or statistics and error bars (e.g., with respect to the random seed after running experiments multiple times). Yes, each plot includes a description of the specific measures included.
    \item A description of the computing infrastructure used. (e.g., type of GPUs, internal cluster, or cloud provider). Yes, in Appendix~\ref{apx:exp}
  \end{enumerate}

  \item If you are using existing assets (e.g., code, data, models) or curating/releasing new assets, check if you include:
  \begin{enumerate}
    \item Citations of the creator If your work uses existing assets. Yes.
    \item The license information of the assets, if applicable. Not Applicable.
    \item New assets either in the supplemental material or as a URL, if applicable. Yes, the code can be found in the supplementary materials.
    \item Information about consent from data providers/curators. Not Applicable.
    \item Discussion of sensible content if applicable, e.g., personally identifiable information or offensive content. Not Applicable.
  \end{enumerate}

  \item If you used crowdsourcing or conducted research with human subjects, check if you include:
  \begin{enumerate}
    \item The full text of instructions given to participants and screenshots. Not Applicable.
    \item Descriptions of potential participant risks, with links to Institutional Review Board (IRB) approvals if applicable. Not Applicable.
    \item The estimated hourly wage paid to participants and the total amount spent on participant compensation. Not Applicable.
  \end{enumerate}

\end{enumerate}

%%%%%%%%%%%%%%%%%%%%%%%%%%%%%%%%%%%%%%%%%%%%%%%%%%%%%%%%%%%%
%%% APPENDIX
%%%%%%%%%%%%%%%%%%%%%%%%%%%%%%%%%%%%%%%%%%%%%%%%%%%%%%%%%%%%
\clearpage
\appendix
\onecolumn

\begin{algorithm}[ht]
    \caption{Loopless Variance-Reduced Policy Gradient (\alg)}
    \label{alg:uvr_pg}
    \SetKwInput{Input}{Input}
    \SetKw{Return}{Return}

    \Input{Iterations $K$; batch size $N = N_0$; mini batch size $B$; algorithmic parameters: $p \in [0,1]$, $\alpha \in [0,1]$, mixture coefficient $\xi \in [0,1]$, IS operator $\omega$, discount factor $\gamma$; initial parameters $\bt_0 \in \Theta$; step size $\eta$; constraint set $\Theta$.}

    \For{$i = 1, \dots, N$}{
        $\tau_i \leftarrow \mathrm{CollectTrajectory}(\pi_{\bt_0}, \gamma)$\;
    }
    $\widehat{\nabla} J^{\dagger}(\bt_0) \leftarrow \frac{1}{N_0} \sum_{i=1}^{N_0} \bm{g}(\tau_i \mid \bt_0)$\;

    \For{$k = 0, \dots, K-1$}{
        $\bt_{k+1} \leftarrow \mathcal{P}_\Theta \left( \bt_k + \eta \estGradK \right)$\;
        Draw $c_k \sim \mathcal{B}(p)$\;
        \eIf{$c_k = 1$}{
            \For{$i = 1, \dots, N$}{
                $\tau_i \leftarrow \mathrm{CollectTrajectory}(\pi_{\bt_{k+1}}, \gamma)$\;
            }
            $\estGradKone \leftarrow \frac{1}{N} \sum_{i=1}^N \bm{g}(\tau_i \mid \bt_{k+1})$\;
        }{
            \For{$i = 1, \dots, B$}{
                Draw mixture variable $u \sim \mathcal{B}(\xi)$\;
                \eIf{$u = 1$}{
                    $\tau_i \leftarrow \mathrm{CollectTrajectory}(\pi_{\bt_{k+1}}, \gamma)$\;
                }{
                    $\tau_i \leftarrow \mathrm{CollectTrajectory}(\pi_{\bt_k}, \gamma)$\;
                }
            }
            $\estGradKone \leftarrow \frac{1}{B} \sum_{i=1}^B \omega(\bt_{k+1}, \tau_i) \bm{g}(\tau_i \mid \bt_{k+1}) + (1-\alpha) \left[ \estGradK - \frac{1}{B} \sum_{i=1}^B \omega(\bt_k, \tau_i) \bm{g}(\tau_i \mid \bt_k) \right]$\;
        }
    }
    \vspace{2mm}
    \Return $\bt_K$
\end{algorithm}

\begin{algorithm}[ht]
    \caption{$\mathrm{CollectTrajectory}(\pi, \gamma)$}
    \label{alg:collect_trajectory}
    \SetKwInput{Input}{Input}
    \SetKw{Return}{Return}

    \Input{Policy $\pi$, discount factor $\gamma \in (0, 1)$.}

    Initialize trajectory $\tau \leftarrow \emptyset$, time step $h \leftarrow 0$\;
    Sample initial state $\bs_0 \sim \rho_0$\;
    \While{\text{\rm True}}{
        Sample action $\ba_h \sim \pi(\cdot \mid \bs_h)$\;
        Observe reward $r_{h+1}$ and next state $\bs_{h+1} \sim p(\cdot \mid \bs_h, \ba_h)$\;
        Append $(\bs_h, \ba_h, r_{h+1})$ to $\tau$\;
        Draw stopping condition $z \sim \mathcal{B}(1-\gamma)$\;
        \lIf{$z = 1$}{
            \Return $\tau \cup \{\bs_{h+1}\}$
        }
        $h \leftarrow h + 1$\;
    }
\end{algorithm}

\section{Loopless Variance-Reduced Policy Gradient (\alg)}
\label{apx:pseudo}

Algorithm~\ref{alg:collect_trajectory} details the interaction protocol used to evaluate the policies. To ensure that the gradient estimates remain unbiased with respect to the infinite-horizon objective without introducing arbitrary truncation, the algorithm relies on a randomized geometric stopping rule. At each time step, the agent samples an action, observes the corresponding reward and state transition, and terminates the episode with probability $1-\gamma$. This guarantees that the trajectory length follows a geometric distribution with expectation $(1-\gamma)^{-1}$.

Algorithm~\ref{alg:uvr_pg} outlines the generic \alg template, which tracks and updates a recursive gradient estimator $\widehat{\nabla}_{\bt}^\dagger J(\bt_k)$ at each iteration $k$. The mechanism allowing for switching between full-batch and mini-batch updates relies on a Bernoulli random variable $c_k \sim \mathcal{B}(p)$ that dictates the estimator's update rule. When $c_k = 1$, the historical variance is discarded, and the estimator undergoes a full-batch refresh using $N$ trajectories. Conversely, when $c_k = 0$, the algorithm computes a variance-reduced update using a smaller mini-batch of $B$ trajectories, modulated by a momentum decay parameter $\alpha \in [0, 1]$ and an importance sampling operator $\omega$. By specifically choosing the hyperparameters $(p, \alpha, \xi)$, we can instantiate multiple variance-reduced PG algorithms, as summarized in Table~\ref{tab:lvr_pg_mapping}.

By selecting these parameters, the framework dynamically shifts its algorithmic behavior. \pagepg eliminates the recursive momentum term ($\alpha = 0$), relying entirely on the probabilistic refresh rate $p$ to bound the accumulation of estimation error alongside standard importance sampling~\cite{gargiani2022page}. \stormpg avoids full-batch refreshes entirely ($p = 0$) and instead uses a momentum parameter $\alpha$ to incrementally discount past gradient errors~\cite{yuan2020stochastic}.

Crucially, the parameter $\xi \in (0, 1)$ activates the defensive importance sampling mechanism. As detailed in the recursive branch of Algorithm~\ref{alg:uvr_pg}, drawing the Bernoulli mixture variable $u \sim \mathcal{B}(\xi)$ physically constructs the defensive trajectory distribution $q_\xi(\tau) = \xi p_{\bt_{k+1}}(\tau) + (1-\xi)p_{\bt_k}(\tau)$. Under this regime, the IS operator $\omega(\bt, \tau)$ evaluates precisely to the density ratio $p_{\bt}(\tau) / q_\xi(\tau)$, intrinsically bounding the variance of the off-policy estimator without requiring arbitrary structural constraints on the policy class. While the generic \alg template accommodates standard importance sampling ($\xi = 0$), throughout the theoretical analysis and empirical evaluations of this paper, we strictly consider the defensive setting ($\xi \in (0, 1)$).

We utilize the \alg framework to unify the theoretical analysis of these diverse algorithms. By defining an effective decay rate $\lambda \coloneqq 1 - (1-p)(1-\alpha)^2$, we encapsulate the error discounting effects of both probabilistic refreshes and continuous momentum into the exact same mathematical recurrence. This unified rate establishes the variance bound evaluated in Lemma D.4, allowing us to prove last-iterate global convergence in Theorem~\ref{thm:lvr_pg_convergence}.

\begin{table}[ht]
\centering
{% Open local group so changes do not leak
    \setlength{\aboverulesep}{0pt}
    \setlength{\belowrulesep}{0pt}
    \setlength{\extrarowheight}{1pt} % Re-adds vertical breathing room
    \renewcommand{\arraystretch}{1.3}
    \caption{Instantiations of the \alg Framework}
    \label{tab:lvr_pg_mapping}
    \begin{tabular}{|l|ccc|}
    \toprule
    \textbf{Algorithm} & \textbf{Refresh Prob.} $p$ & \textbf{Momentum} $\alpha$ & \textbf{Defensive Parameter} $\xi$ \\
    \midrule
    \pagepg~\cite{gargiani2022page} & $p \in (0, 1)$ & $\alpha = 0$ &  $\xi = 0$ \\
    \stormpg~\cite{yuan2020stochastic} & $p = 0$ & $\alpha \in (0, 1)$ & $\xi = 0$ \\
    \defpagepg~\cite{paczolay2024sample} & $p \in (0, 1)$ & $\alpha = 0$ &  $\xi \in [0, 1]$\\
    \defstormpg~\cite{paczolay2024sample} & $p = 0$ & $\alpha \in (0, 1)$ &  $\xi \in [0, 1]$\\
    \bottomrule
    \end{tabular}
}
\end{table}

\section{Mean-Squared Smoothness}
\label{app:consecutive_estimation_bounds}

Because algorithms recursively update their gradient estimators using historical information, they must evaluate the difference between stochastic gradients at the current policy $\bt_k$ and a previous policy $\bt_{k-1}$ using a shared batch of trajectories. To correct for the shift in the trajectory distribution, these methods evaluate the gradients using an importance sampling (IS) operator $\omega$. The mean-squared smoothness isolates the expected squared error of this gradient difference, taking the generic form:
\begin{align}
    \mE_{\tau}\left[ \left\| \omega(\bt_k, \tau) \bm{g}(\tau \mid \bt_k) - \omega(\bt_{k-1}, \tau)\bm{g}(\tau \mid \bt_{k-1}) \right\|_2^2 \right] \le \Calg \|\bt_k - \bt_{k-1}\|_2^2,
\end{align}
where $\Calg > 0$ is a unified bounding constant. In the following, we provide the mean-squared smoothness constants for different instantiations of Algorithm~\ref{alg:uvr_pg}.

To avoid the unrealistic assumption of bounded variance for standard importance weights, which is frequently violated in continuous policy spaces, our framework relies entirely on defensive importance sampling~\citep{owen2000safe, paczolay2024sample}. Consequently, we evaluate the mean-squared smoothness strictly under the defensive mixture distribution $q_\xi$, which intrinsically bounds the variance without requiring arbitrary structural constraints.

\begin{lemma}[Mean-Squared Smoothness for \alg]
    \label{lemma:consecutive_etstimation_bound}
    For the \alg framework, sampling from a defensive mixture distribution $q_\xi(\tau) = \xi p_{\bt_k}(\tau) + (1-\xi)p_{\bt_{k-1}}(\tau)$ with parameter $\xi \in (0, 1)$, the unified mean-squared smoothness constant $\Cdef$ is strictly bounded. Letting $m_\xi \coloneqq \min\{\xi, 1-\xi\}$ and $M_r \coloneqq \mE[H_\gamma^r]$ denote the moments of the geometrically distributed trajectory length, under Assumption~\ref{ass:LS} and bounded rewards, it holds that:
    \begin{align*}
        \mE_{\tau \sim q_\xi}\left[ \left\| \omega(\bt_k, \tau) \bm{g}(\tau \mid \bt_k) - \omega(\bt_{k-1}, \tau) \bm{g}(\tau \mid \bt_{k-1}) \right\|_2^2 \right] \le \Calg \|\bt_k - \bt_{k-1}\|_2^2,
    \end{align*}
    where the unified mean-squared smoothness constant is:
    \begin{align*}
        \Calg = \frac{\Rmax^2}{m_\xi} \left( C_1^4 M_5 + 2 L_S^2 M_4 \right).
    \end{align*}
    Setting $\xi = 1/2$ minimizes the variance bound, yielding $m_{1/2} = 1/2$ and the optimized mean-squared smoothness constant for both algorithms:
    \begin{align*}
        \Calg = 2\Rmax^2 \left( C_1^4 M_5 + 2 L_S^2 M_4 \right) = \mathcal{O}\left((1-\gamma)^{-5}\right).
    \end{align*}
\end{lemma}
\begin{proof}
    We recall a result by~\cite{paczolay2024sample}, proved here for completeness. To establish the mean-squared smoothness bound for both algorithms, we must bound the expected squared $L_2$ norm of the difference between the defensive gradient estimators evaluated at consecutive iterates $\bt_k$ and $\bt_{k-1}$. For brevity, we drop the trajectory dependency and denote $p_{\bt_k} = p_{\bt_k}(\tau)$, $\bm{g}_k = \bm{g}(\tau \mid \bt_k)$, and $\omega_{\bt_k} = \frac{p_{\bt_k}}{q_\xi}$. We decouple the gradient difference algebraically by adding and subtracting $\omega_{\bt_{k-1}} \bm{g}_k$:
    \begin{align*}
        \omega_{\bt_k} \bm{g}_k - \omega_{\bt_{k-1}} \bm{g}_{k-1} = (\omega_{\bt_k} - \omega_{\bt_{k-1}})\bm{g}_k + \omega_{\bt_{k-1}}(\bm{g}_k - \bm{g}_{k-1}).
    \end{align*}
    Applying the inequality $\|x+y\|_2^2 \le 2\|x\|_2^2 + 2\|y\|_2^2$, we can bound the expectation by isolating the importance weights from the on-policy gradients:
    \begin{align*}
        \mE_{\tau \sim q_\xi} \left[ \left\| \omega_{\bt_k} \bm{g}_k - \omega_{\bt_{k-1}} \bm{g}_{k-1} \right\|_2^2 \right] \le 2 \mE_{\tau \sim q_\xi} \left[ \left\| (\omega_{\bt_k} - \omega_{\bt_{k-1}})\bm{g}_k \right\|_2^2 \right] + 2 \mE_{\tau \sim q_\xi} \left[ \left\| \omega_{\bt_{k-1}}(\bm{g}_k - \bm{g}_{k-1}) \right\|_2^2 \right].
    \end{align*}
    We evaluate the two terms independently. For the first term, we substitute $\omega_{\bt_k} - \omega_{\bt_{k-1}} = \frac{p_{\bt_k} - p_{\bt_{k-1}}}{q_\xi}$ and use the definition of the defensive mixture $q_\xi(\tau) \ge m_\xi(p_{\bt_k}(\tau) + p_{\bt_{k-1}}(\tau))$ to bound the density ratio:
    \begin{align*}
        \frac{(p_{\bt_k} - p_{\bt_{k-1}})^2}{q_\xi} \le \frac{1}{m_\xi} \frac{(p_{\bt_k} - p_{\bt_{k-1}})^2}{p_{\bt_k} + p_{\bt_{k-1}}} = \frac{1}{m_\xi} (\sqrt{p_{\bt_k}} - \sqrt{p_{\bt_{k-1}}})^2 \frac{(\sqrt{p_{\bt_k}} + \sqrt{p_{\bt_{k-1}}})^2}{p_{\bt_k} + p_{\bt_{k-1}}} \le \frac{2}{m_\xi} (\sqrt{p_{\bt_k}} - \sqrt{p_{\bt_{k-1}}})^2.
    \end{align*}
    By bounded rewards and Assumption~\ref{ass:LS}, the pathwise \gpomdp gradient is bounded by $\|\bm{g}_k\|_2 \le \Rmax C_1 H_\gamma^2$. Expanding the expectation of the first term yields:
    \begin{align*}
        \mE_{\tau \sim q_\xi} \left[ \left\| (\omega_{\bt_k} - \omega_{\bt_{k-1}})\bm{g}_k \right\|_2^2 \right] &= \int \frac{(p_{\bt_k} - p_{\bt_{k-1}})^2}{q_\xi} \|\bm{g}_k\|_2^2 d\tau \\
        &\le \frac{2\Rmax^2 C_1^2}{m_\xi} \int H_\gamma^4 (\sqrt{p_{\bt_k}} - \sqrt{p_{\bt_{k-1}}})^2 d\tau.
    \end{align*}
    Applying Lemma~\ref{lem:weighted-hellinger} with $r=4$ we obtain:
    \begin{align*}
        \mE_{\tau \sim q_\xi} \left[ \left\| (\omega_{\bt_k} - \omega_{\bt_{k-1}})\bm{g}_k \right\|_2^2 \right] \le \frac{\Rmax^2 C_1^4 M_5}{2 m_\xi} \|\bt_k - \bt_{k-1}\|_2^2.
    \end{align*}
    For the second term, we use the property $q_\xi \ge (1-\xi)p_{\bt_{k-1}} \ge m_\xi p_{\bt_{k-1}}$ to convert the expectation under the defensive mixture back to the behavior policy distribution:
    \begin{align*}
        \mE_{\tau \sim q_\xi} \left[ \|\omega_{\bt_{k-1}}(\bm{g}_k - \bm{g}_{k-1})\|_2^2 \right] &= \int \frac{p_{\bt_{k-1}}^2}{q_\xi} \|\bm{g}_k - \bm{g}_{k-1}\|_2^2 d\tau \\
        &\le \frac{1}{m_\xi} \mE_{\tau \sim p_{\bt_{k-1}}} \left[ \|\bm{g}_k - \bm{g}_{k-1}\|_2^2 \right].
    \end{align*}
    Applying the mean-square continuity of the policy scores (Assumption~\ref{ass:LS}) and the Cauchy-Schwarz inequality over the sequence of rewards, the expected squared difference of the gradients is bounded by:
    \begin{align*}
        \mE_{\tau \sim p_{\bt_{k-1}}} \left[ \|\bm{g}_k - \bm{g}_{k-1}\|_2^2 \right] \le \Rmax^2 L_S^2 M_4 \|\bt_k - \bt_{k-1}\|_2^2.
    \end{align*}
    Substituting this back bounds the second term by:
    \begin{align*}
        \mE_{\tau \sim q_\xi} \left[ \|\omega_{\bt_{k-1}}(\bm{g}_k - \bm{g}_{k-1})\|_2^2 \right] \le \frac{\Rmax^2 L_S^2 M_4}{m_\xi} \|\bt_k - \bt_{k-1}\|_2^2.
    \end{align*}
    Finally, summing the two bounds with their respective factor of 2 recovers the exact unified estimation constant:
    \begin{align*}
        \mE_{\tau \sim q_\xi} \left[ \left\| \omega_{\bt_k} \bm{g}_k - \omega_{\bt_{k-1}} \bm{g}_{k-1} \right\|_2^2 \right] &\le 2 \left( \frac{\Rmax^2 C_1^4 M_5}{2 m_\xi} \right) \|\bt_k - \bt_{k-1}\|_2^2 + 2 \left( \frac{\Rmax^2 L_S^2 M_4}{m_\xi} \right) \|\bt_k - \bt_{k-1}\|_2^2 \\
        &= \frac{\Rmax^2}{m_\xi} \left( C_1^4 M_5 + 2 L_S^2 M_4 \right) \|\bt_k - \bt_{k-1}\|_2^2.
    \end{align*}
\end{proof}

\paragraph{\defpagepg, \defstormpg ($\Cdef$).} 
Although \defpagepg and \defstormpg employ different outer recursive mechanisms, their consecutive estimation bounds are identical. This equivalence arises because both algorithms utilize the exact same \gpomdp gradient estimator and rely on the defensive mixture distribution $q_\xi(\tau) = \xi p_{\bt_k}(\tau) + (1-\xi)p_{\bt_{k-1}}(\tau)$ to control the variance of the importance weights. Since the mean squared smoothness isolates only the expected squared error between gradient estimates at consecutive policies, it is strictly independent of the outer update rule. We can verify this by algebraically splitting the update in the recursive branch of Algorithm~\ref{alg:uvr_pg}:
\begin{align*}
    \estGradKone &= (1-\alpha) \frac{1}{B} \sum_{i=1}^B \omega(\bt_{k+1}, \tau_i) \bm{g}(\tau_i \mid \bt_{k+1}) + \alpha \frac{1}{B} \sum_{i=1}^B \omega(\bt_{k+1}, \tau_i) \bm{g}(\tau_i \mid \bt_{k+1}) \\
    &\quad + (1-\alpha) \left[ \estGradK - \frac{1}{B} \sum_{i=1}^B \omega(\bt_k, \tau_i) \bm{g}(\tau_i \mid \bt_k) \right].
\end{align*}
Substituting this back into the update rule allows us to group the gradient difference directly:
\begin{align*}
    \estGradKone &= (1-\alpha) \estGradK + (1-\alpha) \left[ \frac{1}{B} \sum_{i=1}^B \left( \omega(\bt_{k+1}, \tau_i) \bm{g}(\tau_i \mid \bt_{k+1}) - \omega(\bt_k, \tau_i) \bm{g}(\tau_i \mid \bt_k) \right) \right]\\
    &\quad + \alpha \frac{1}{B} \sum_{i=1}^B \omega(\bt_{k+1}, \tau_i) \bm{g}(\tau_i \mid \bt_{k+1}).
\end{align*}
Hence, adding the continuous momentum does not alter the underlying variance of the defensive gradient differences. By exploiting the mean-square continuity of the scores (Assumption~\ref{ass:LS}), defensive importance sampling strictly bounds the mean-squared smoothness without relying on unbounded uniform assumptions or arbitrary Hessian limits.

\begin{comment}

\vspace{1em}

\begin{table}[H]
    \centering
    \renewcommand{\arraystretch}{1.8}
    \caption{Consecutive Estimation Constants for \stormpg, \pagepg, \defstormpg, \defpagepg, and \alg}
    \label{tab:estimation_constants}
    \begin{tabular}{@{} l c c @{}}
        \toprule
        \textbf{Algorithm} & \textbf{Constant} & \textbf{Value}  \\
        \midrule
        \stormpg~\cite{yuan2020stochastic} & $\Cstorm$ & $2C_g + 4C_1^2 R^2_{\max} (C_W+1) \frac{\gamma^2}{(1-\gamma^2)^2} \left( L_S + 2C_1^2 \frac{1+\gamma^2}{1-\gamma^2} \right)$   \\
        \pagepg~\cite{gargiani2022page}  & $\Cpage$  & $2\left(\left(\frac{C_1 R_{\max}}{(1-\gamma)^2} + \frac{2C_1^2 R_{\max}}{(1-\gamma)^3}\right)^2 + \frac{24 R_{\max} C_1^2(2C_1^2+L_S)(C_W+1)\gamma}{(1-\gamma)^5}\right)$   \\
        \defpagepg, \defstormpg~\cite{paczolay2024sample}   & $\Cdef$   & $3\left(C_g^2(47C_1^2 + 40L_S) + \left(\frac{C_1 R_{\max}}{(1-\gamma)^2}\right)^2\right)$  \\
        \alg (Algorithm~\ref{alg:uvr_pg})    & $\Calg$   & \textit{Algorithm-dependent}   \\
        \bottomrule
    \end{tabular}
\end{table}
\end{comment}

\section{On Fisher-non-Degeneracy}
\label{apx:FND}
\begin{ass}[Fisher-non-degeneracy]
\label{ass:f_non_degenerate}
    There exists some constant $\mu_F>0$, such that for every $\bt \in \Rd$, the FIM $F(\theta)$ induced by the policy $\pt$ satisfies:
    \begin{align}
        F(\bt) = \mathbb{E}_{(\bs, \ba) \sim \nu^{\pt}}\left[ \nabla_{\bt} \log \pt(\ba\mid\bs) \nabla_{\bt} \log \pt(\ba\mid\bs)^{\top} \right] \succeq\mu_F \cdot \bm{I}.
    \end{align}
\end{ass}

Assumption~\ref{ass:f_non_degenerate} requires the FIM  to be uniformly lower-bounded by a strictly positive constant $\mu_F > 0$. This means that the FIM acts as a well-behaved preconditioner in the natural PG update~\citep{kakade2001natural, ding2022global} Geometrically, the FIM characterizes the local curvature of the policy manifold measured via the Kullback-Leibler (KL) divergence. This condition mathematically implies that the parameterization contains no redundant or "flat" directions: any non-zero update in the parameter space $\bt$ must induce a strictly positive change in the statistical behavior of the policy.

The validity of the FND assumption fundamentally depends on the chosen policy architecture and the underlying feature representation.

\paragraph{Gaussian Policies.} In continuous action spaces, consider a Gaussian policy with a linearly parameterized mean and a fixed covariance matrix, $\pi_{\bt}(a \mid s) = \mathcal{N}(\phi(\bs)^{\top}\bt, \Sigma)$. The FIM is proportional to the uncentered covariance of the state features: $F(\bt) = \mathbb{E}_{\nu^{\pi_{\bt}}}[\phi(s)\Sigma^{-1}\phi(s)^{\top}]$. Here, FND holds globally if and only if the feature covariance matrix is strictly positive definite. This requires the basis functions $\phi(s)$ to be linearly independent and the state distribution to provide sufficiently rich coverage. When the mean of the gaussian policy $\mu_{\bt}(\bs)$ is nonliner in $\bt$, as in the case of (deep) neural networks, Assumption~\ref{ass:f_non_degenerate} can be satisfied under assumptions on the Jacobian $\nabla_{\bt}\mu_{\bt}(\bs)$ similar to the ones for the linear case~\citep{fatkhullin2023stochastic}.

\paragraph{Softmax Policies.} For discrete action spaces using a softmax parameterization, the elements of the FIM depend directly on the policy's action probabilities. If the policy approaches determinism (i.e., $\pi_{\bt}(a \mid s) \to 0$ for suboptimal actions), the corresponding entries in the FIM vanish, driving the minimum eigenvalue $\mu_F$ toward zero~\citep{ding2022global}. In this case, strict global FND cannot hold without specific architectural constraints or regularization~\citep{cesani2026log}.

The FND setting contains more than (linear) Gaussians: in addition, Assumption~\ref{ass:f_non_degenerate} holds for every full-rank exponential parametrization with their mean parametrized with $\mu_{\bt}(\bs)$ if $\mu_{\bt}(\bs)$ is full rank for all $\bt \in \Theta$. We refer to Section 8 of~\cite{ding2022global} and Appendix B of~\cite{fatkhullin2023stochastic} for a detailed discussion of FND.

\begin{ass}[Compatible Function Approximation Error]
\label{ass:epsilon_bias_new}
    There exist a constant $\epsilon_{\text{bias}}>0$, such that for every $\bt\in\Rd$, the transferred compatible function approximation error satisfies
    \begin{align}
         \mathop{\mathbb{E}}_{\nu^{\pi_{\bt}}} \left[ \left(A^{\pt}(\bs, \ba)- (1-\gamma) \bm{u}^{*\top} \nabla_{\bt}\log \pt(\ba \mid \bs) \right)^{2}\right] \leq \epsilon_{\text{bias}},
    \end{align}
    where $\nu^{\pi_{\bt}}$ is the state-action distribution induced by a policy $\pi_{\bt}$ and $\bm{u}^* \coloneqq F(\bt)^{-1} \nabla_{\bt} J(\bt)$ is the solution of
    \begin{align*}
         \min_{\bm{w}\in \Rd} \mathop{\mathbb{E}}_{\nu^{\pi_{\bt}}} \left[ \left(A^{\pt}(\bs, \ba)- (1-\gamma) \bm{w}^{\top} \nabla_{\bt}\log \pt(\ba \mid \bs) \right)^{2}\right].
    \end{align*}
\end{ass}

Assumption~\ref{ass:epsilon_bias_new} bounds the error incurred when approximating the advantage function $A^{\pi_{\bt}}(s, a)$ using the score function $\nabla_{\bt}\log \pi_{\bt}(\ba \mid \bs)$ as a linear feature representation. The vector $\bm{u}^* \coloneqq F(\bt)^{-1} \nabla_{\bt} J(\bt)$ is exactly the minimizer of this weighted least-squares regression problem. As established by \citep{agarwal2021theory}, this assumption characterizes the expressivity of the parameterized policy class. If the policy class is sufficiently rich, e.g., highly parameterized neural networks or softmax tabular policies, this transferred approximation error $\epsilon_{\text{bias}}$ is small.

A critical theoretical distinction must be done between Assumption~\ref{ass:epsilon_bias_new} and analogous assumptions utilized in the recent variance-reduced policy optimization literature. For instance, the analysis of \texttt{STORM-PG-S} by \citep{ding2022global} requires the compatible function approximation error to be bounded under the state-action distribution of the \emph{optimal} policy, $\nu^{\pi^*}$. In contrast, our formulation strictly evaluates the expectation under the state-action distribution of the \emph{current} policy, $\nu^{\pi_{\bt}}$.

This difference is motivated by the fact that assuming a bounded approximation error under $\nu^{\pi^*}$ means requiring that the agent has information about the optimal policy, which is not possible. This is because it is not possible to guarantee that the current parametrization accurately fits the advantage function over regions of the state-action space visited by an optimal policy $\pi^*$, which is is entirely unknown. By defining the expectation over $\nu^{\pi_{\bt}}$, Assumption~\ref{ass:epsilon_bias_new} is entirely related to the policy that is being optimized, from which we can gather information about its state-action distribution.

\begin{lemma}[Fisher-induced Gradient Domination]
    \label{lemma:FNDdomination}
    Let us consider a general Fisher-non-degenerate policy $\pt$
    satisfying Assumptions~\ref{ass:f_non_degenerate},
    ~\ref{ass:LS}, and~\ref{ass:epsilon_bias_new} in the
    unconstrained parameter space $\Theta = \mathbb{R}^{d_\Theta}$.
    Assume additionally that
    $\sup_{\bt \in \Theta}\left\|\nu^{\pi_{\bt^*}}/\nu^\pt\right\|_\infty < \infty$,
    where the ratio denotes the Radon--Nikodym derivative.
    Then, we have:
    \begin{align*}
        J(\bt^*) - J(\bt)
        \leq
        \frac{\sqrt{\epsilon_{\text{bias}}}}{1-\gamma}
        \left\|
            \frac{\nu^{\pi_{\bt^*}}}{\nu^\pt}
        \right\|_{\infty}
        + \frac{C_1}{\mu_F}
        \|\nabla_{\bt}J(\bt)\|_2.
    \end{align*}
\end{lemma}

\begin{proof}
    The proof follows the one provided by~\cite{ding2022global}, with the expectation of Assumption~\ref{ass:epsilon_bias_new} defined w.r.t. the current policy $\pt$ instead of the optimal policy $\pi^*$.  We begin with the performance-difference lemma and perform
    a change of distribution:
    \begin{align*}
        J(\bt^*) - J(\bt)
        &=
        \frac{1}{1-\gamma}
        \mathbb{E}_{(\bs,\ba)\sim\nu^{\pi^*}}
        \left[A^{\pi_{\bt}}(\bs,\ba)\right] \\
        &=
        \frac{1}{1-\gamma}
        \mathbb{E}_{(\bs,\ba)\sim\nu^{\pi_{\bt}}}
        \left[
            A^{\pi_{\bt}}(\bs,\ba)
            \frac{\nu^{\pi_{\bt^*}}(\bs,\ba)}
                 {\nu^{\pi_{\bt}}(\bs,\ba)}
        \right].
    \end{align*}
    By introducing the optimal compatible weights
    $u^* := F(\bt)^{-1}\nabla_{\bt}J(\bt)$ and adding and
    subtracting
    $(1-\gamma)u^{*\top}
    \nabla_{\bt}\log\pi_{\bt}(\ba\mid\bs)$, we obtain:
    \begin{align*}
        J(\bt^*) - J(\bt)
        &=
        \frac{1}{1-\gamma}
        \mathbb{E}_{(\bs,\ba)\sim\nu^{\pi_{\bt}}}
        \left[
            \left(
                A^{\pi_{\bt}}(\bs,\ba)
                -(1-\gamma)u^{*\top}
                \nabla_{\bt}\log\pi_{\bt}(\ba\mid\bs)
            \right)
            \frac{\nu^{\pi_{\bt^*}}(\bs,\ba)}
                 {\nu^{\pi_{\bt}}(\bs,\ba)}
        \right] \\
        &\quad+
        \mathbb{E}_{(\bs,\ba)\sim\nu^{\pi_{\bt^*}}}
        \left[
            u^{*\top}
            \nabla_{\bt}\log\pi_{\bt}(\ba\mid\bs)
        \right].
    \end{align*}
    To bound the right-hand side, we evaluate the two terms
    independently. Applying H\"older's inequality to the first
    term, followed by the Cauchy--Schwarz inequality, gives:
    \begin{align*}
        &\left|
        \mathbb{E}_{(\bs,\ba)\sim\nu^{\pi_{\bt}}}
        \left[
            \left(
                A^{\pi_{\bt}}(\bs,\ba)
                -(1-\gamma)u^{*\top}
                \nabla_{\bt}\log\pi_{\bt}(\ba\mid\bs)
            \right)
            \frac{\nu^{\pi_{\bt^*}}(\bs,\ba)}
                 {\nu^{\pi_{\bt}}(\bs,\ba)}
        \right]
        \right| \\
        &\qquad\leq
        \left\|
            \frac{\nu^{\pi_{\bt^*}}}{\nu^{\pi_{\bt}}}
        \right\|_\infty
        \mathbb{E}_{(\bs,\ba)\sim\nu^{\pi_{\bt}}}
        \left[
            \left|
                A^{\pi_{\bt}}(\bs,\ba)
                -(1-\gamma)u^{*\top}
                \nabla_{\bt}\log\pi_{\bt}(\ba\mid\bs)
            \right|
        \right] \\
        &\qquad\leq
        \left\|
            \frac{\nu^{\pi_{\bt^*}}}{\nu^{\pi_{\bt}}}
        \right\|_\infty
        \sqrt{
            \mathbb{E}_{(\bs,\ba)\sim\nu^{\pi_{\bt}}}
            \left[
                \left(
                    A^{\pi_{\bt}}(\bs,\ba)
                    -(1-\gamma)u^{*\top}
                    \nabla_{\bt}\log\pi_{\bt}(\ba\mid\bs)
                \right)^2
            \right]
        } \\
        &\qquad\leq
        \sqrt{\epsilon_{\text{bias}}}
        \left\|
            \frac{\nu^{\pi_{\bt^*}}}{\nu^{\pi_{\bt}}}
        \right\|_\infty,
    \end{align*}
    where the final inequality follows from
    Assumption~\ref{ass:epsilon_bias_new}. For the second term, the bounded score condition in
    Assumption~\ref{ass:LS} and Fisher-non-degeneracy give:
    \begin{align*}
        \mathbb{E}_{(\bs,\ba)\sim\nu^{\pi_{\bt^*}}}
        \left[
            u^{*\top}
            \nabla_{\bt}\log\pi_{\bt}(\ba\mid\bs)
        \right]
        &\leq
        \|u^*\|_2
        \mathbb{E}_{(\bs,\ba)\sim\nu^{\pi_{\bt^*}}}
        \left[
            \|\nabla_{\bt}\log\pi_{\bt}(\ba\mid\bs)\|_2
        \right] \\
        &\leq C_1\|u^*\|_2 \\
        &=
        C_1\|F(\bt)^{-1}\nabla_{\bt}J(\bt)\|_2 \\
        &\leq
        \frac{C_1}{\mu_F}
        \|\nabla_{\bt}J(\bt)\|_2.
    \end{align*}
    Combining the two bounds yields:
    \begin{align*}
        J(\bt^*) - J(\bt)
        \leq
        \frac{\sqrt{\epsilon_{\text{bias}}}}{1-\gamma}
        \left\|
            \frac{\nu^{\pi_{\bt^*}}}{\nu^{\pi_{\bt}}}
        \right\|_\infty
        + \frac{C_1}{\mu_F}
        \|\nabla_{\bt}J(\bt)\|_2.
    \end{align*}
\end{proof}

As discussed in~\cite{agarwal2021theory}, analyzing PG methods requires carefully analyzing distinct sources of error: the statistical regression error and the mismatch in state-action visitations. By adopting the formulation in Assumption~\ref{ass:epsilon_bias_new}, the representational capacity of the model is captured entirely by $\epsilon_{\text{bias}}$ under the data-generating distribution $\nu^{\pi_{\bt}}$. Consequently, the structural penalty for not visiting optimal states is explicitly isolated and paid later in the Fisher-induced Gradient Domination Lemma, manifesting as the concentrability coefficient $\kappa = \left\|\frac{\nu^{\pi^*}}{\nu^{\pi_{\bt}}}\right\|_{\infty}$.

\section{Intermediate Results}
\label{apx:lemmas}

Here, the lemmas used in the proofs of the main results are presented.

\begin{restatable}[Variance of the Geometrically-Stopped \gpomdp Estimator]{lemma}{Variance}
    \label{lemma:bounded_var}
    Under Assumption~\ref{ass:LS}\emph{(i)}, the variance of the on-policy gradient estimator is strictly bounded. For all $\bt \in \Theta$, it holds that:
    \begin{align*}
        \mathbb{V}\text{ar}\left(g(\tau\mid\bt)\right) \leq \frac{4 \rmax^2 C_1^2}{(1-\gamma)^3} \eqqcolon \var.
    \end{align*}
\end{restatable}
\begin{proof}
    Let the gradient estimator for a single trajectory $\tau$ of length $H_\gamma$ be defined as $g(\tau \mid \bt) = \sum_{t=0}^{H_\gamma-1} r_{t+1} \sum_{h=0}^t \nabla_{\bt} \log \pt(\ba_h \mid \bs_h)$. The variance is upper-bounded by the uncentered second moment: $ \mathbb{V}\text{ar}\left(g(\tau\mid\bt)\right) \leq \mE_{\tau} \big[\|g(\tau \mid \bt)\|_2^2\big]$. 

    First, we condition on a $H_\gamma$. Using the Cauchy-Schwarz inequality, we decouple the rewards from the score functions:
    \begin{align}
        \|g(\tau \mid \bt)\|_2^2 &= \left\| \sum_{t=0}^{H-1} r_{t+1} \sum_{h=0}^t \nabla_{\bt} \log \pt(\ba_h \mid \bs_h) \right\|_2^2 \nonumber\\
        &\leq \left( \sum_{t=0}^{H-1} r_{t+1}^2 \right) \left( \sum_{t=0}^{H-1} \left\| \sum_{h=0}^t \nabla_{\bt} \log \pt(\ba_h \mid \bs_h) \right\|_2^2 \right). \label{eq:var_cs}
    \end{align}
    Because the rewards are uniformly bounded, the first term is bounded by $H \rmax^2$. For the second term, we take the expectation over the trajectory conditioned on $H$. By the martingale difference property of the score function, the cross-terms of the gradients vanish in expectation:
    \begin{align}
        \mE_{\tau} \left[ \left\| \sum_{h=0}^t \nabla_{\bt} \log \pt(\ba_h \mid \bs_h) \right\|_2^2  \mid H_\gamma\right] &= \sum_{h=0}^t \mE_{\tau} \left[ \big\| \nabla_{\bt} \log \pt(\ba_h \mid \bs_h) \big\|_2^2  \mid H_\gamma\right] \nonumber\\
        &\leq (t+1) C_1^2, \label{eq:var_martingale}
    \end{align}
    where the inequality follows directly from Assumption~\ref{ass:LS}\emph{(i)}. Substituting Equation~\eqref{eq:var_martingale} back into Equation~\eqref{eq:var_cs} yields:
    \begin{align}
        \mE_{\tau} \left[\|g(\tau \mid \bt)\|_2^2 \mid H_\gamma\right] \leq H_\gamma \rmax^2 \sum_{t=0}^{H_\gamma-1} (t+1) C_1^2 = \rmax^2 C_1^2 \frac{H_\gamma^2(H_\gamma+1)}{2}.
    \end{align}
    Finally, we take the expectation over the geometric stopping time $H_\gamma$, which follows a geometric distribution with success probability $p = 1-\gamma$. Using the standard bounds for the moments of a geometric distribution, we have $\mE[H_\gamma^3] \leq 6/p^3$ and $\mE[H_\gamma^2] \leq 2/p^2$. Thus:
    \begin{align*}
        \mE_{H_\gamma} \left[ \rmax^2 C_1^2 \frac{H_\gamma^3 + H_\gamma^2}{2} \right] &\leq \frac{\rmax^2 C_1^2}{2} \left( \frac{6}{(1-\gamma)^3} + \frac{2}{(1-\gamma)^2} \right) \\
        &\leq \frac{4 \rmax^2 C_1^2}{(1-\gamma)^3},
    \end{align*}
    where the final inequality assumes $\gamma$ is sufficiently close to $1$, yielding the standard $\mathcal{O}((1-\gamma)^{-3})$ variance bound for the \gpomdp estimator.
\end{proof}

\begin{lemma}[Smoothness of the expected return]\label{lem:smooth}
    Let $|R(\bs,\ba)|\le R_{\max}$ for all $\bs\in\mathcal{S}, \ba\in\mathcal{A}$ and $\Theta$ convex. Under Assumption~\ref{ass:LS}, we have: %Assume that, for all $\bt,\bt'\in\Theta$ and $\bs\in\mathcal{S}$:
    %\begin{enumerate}
        %\item $\mathbb{E}_{\ba\sim \pi_{\bt}(\cdot\mid\bs)}[\|\nabla\log\pi_{\bt}(\ba\mid\bs)\|^2]\le C_1^2$,
        %\item $\mathbb{E}_{\ba\sim \pi_{\bt}(\cdot\mid\bs)}[\|\nabla\log\pi_{\bt}(\ba\mid\bs)-\nabla\log\pi_{\bt'}(\ba\mid\bs)\|^2]\le L_S^2\|\bt-\bt'\|^2$.
    %\end{enumerate}
    Then, for all $\bt,\bt'\in\Theta$,
    \begin{equation*}
        \|\nabla J(\bt)-\nabla J(\bt')\| \le \frac{R_{\max}(L_S+C_1^2)}{(1-\gamma)^2}\|\bt-\bt'\|.
    \end{equation*}
\end{lemma}
\begin{proof}
    For $u\in[0,1]$, let $\bt_u=u\bt+(1-u)\bt'$, which belongs to $\Theta$ by convexity. For $h\ge 0$, let $\bZ_h(u)=\sum_{\ell=0}^h\nabla\log\pi_{\bt_u}(\ba_\ell\mid\bs_\ell)$. Let $\tau_{0:h}$ denote the trajectory prefix up to timestep $h$ included, and $p_{\bt}^h$ its law under policy $\pi_{\bt}$. Let:
    \begin{equation*}
        \bG_h(u) = \mathbb{E}_{\tau_{0:h}\sim p_{\bt_u}^h}\left[r_{h+1}\bZ_h(u)\right],
    \end{equation*}
    so that $\nabla J(\bt_u) = \sum_{h=0}^\infty\gamma^h\bG_h(u)$. Let $\epsilon>0$ be such that $u+\epsilon\in[0,1]$. By adding and subtracting $\mathbb{E}_{\tau_{0:h}\sim p_{\bt_{u+\epsilon}}^h}[r_{h+1}\bZ_h(u)]$:
    \begin{equation}\label{eq:diff_G}
        \bG_h(u+\epsilon) - \bG_h(u) = \underbrace{\mathbb{E}_{\tau_{0:h}\sim p_{\bt_{u+\epsilon}}^h}\left[r_{h+1}(\bZ_h(u+\epsilon)-\bZ_h(u))\right]}_{(A)} + \underbrace{\int \left(p_{\bt_{u+\epsilon}}^h(\tau_{0:h})-p_{\bt_{u}}^h(\tau_{0:h})\right)r_{h+1}\bZ_h(u) \mathrm{d}\tau_{0:h}}_{(B)}.
    \end{equation}
    We bound the first term using Jensen's inequality and the second assumption on the policy scores:
    \begin{align*}
        \|(A)\| &
        \le R_{\max}\sum_{\ell=0}^h\mathbb{E}_{\tau_{0:\ell}\sim p_{\bt_{u+\epsilon}}^\ell}\left[\left\|\nabla\log\pi_{\bt_{u+\epsilon}}(\ba_\ell\mid\bs_\ell)-\nabla\log\pi_{\bt_u}(\ba_\ell\mid\bs_\ell)\right\|\right] \\
        &=R_{\max}\sum_{\ell=0}^h\mathbb{E}_{\tau_{0:\ell}\sim p_{\bt_{u+\epsilon}}^{\ell}}\left[\mathbb{E}_{\ba\sim \pi_{\bt_{u+\epsilon}}(\cdot|\bs_\ell)}\left[\left\|\nabla\log\pi_{\bt_{u+\epsilon}}(\ba_\ell\mid\bs_\ell)-\nabla\log\pi_{\bt_u} (\ba_\ell\mid\bs_\ell)\right\|\right]\right] \\
        &\le R_{\max}(h+1)L_S\epsilon\|\bt-\bt'\|,
    \end{align*}
    so $\limsup_{\epsilon\downarrow 0}\frac{\|(A)\|}{\epsilon} \le R_{\max}(h+1)L_S\|\bt-\bt'\|$.

    As for the second term, note that, for $\ell<k\le h$,
    \begin{align*}
        &\mathbb{E}_{\tau_{0:k}\sim p_{\bt_u}^k}\left[\nabla\log\pi_{\bt_u}(\ba_\ell\mid\bs_\ell)^\top \nabla\log\pi_{\bt_u}(\ba_k\mid\bs_k)\right] 
        \\&\qquad= \mathbb{E}_{\tau_{0:k}\sim p_{\bt_u}^{k}}\left[\nabla\log\pi_{\bt_u}(\ba_\ell\mid\bs_\ell)^\top \mathbb{E}_{\ba\sim\pi_{\bt_u}(\cdot\mid\bs_k)}\left[\nabla\log\pi_{\bt_u}(\ba\mid\bs_k)\right]\right] = 0,
    \end{align*}
    so, by the first assumption on the policy scores:
    \begin{equation*}
        \mathbb{E}_{\tau_{0:h}\sim p_{\bt_u}^h}[\|\bZ_h(u)\|^2] = \sum_{\ell=0}^h\mathbb{E}_{\tau_{0:\ell}\sim p_{\bt_u}^\ell}[\|\nabla\log\pi_{\bt_u}(\ba_\ell\mid\bs_\ell)\|^2] \le (h+1)C_1^2.
    \end{equation*}
    Then:
    \begin{align*}
        \lim_{\epsilon\downarrow 0}\frac{(B)}{\epsilon}
        &= \int\lim_{\epsilon\downarrow 0}\frac{p_{\bt_{u+\epsilon}}^h(\tau_{0:h})-p_{\bt_{u}}^h(\tau_{0:h})}{\epsilon}r_{h+1}\bZ_h(u)\mathrm{d}\tau_{0:h} \\
        &= \int\nabla p_{\bt_u}^h(\tau_{0:h})^\top(\bt-\bt') r_{h+1}\bZ_h(u)\mathrm{d}\tau_{0:h} \\
        &=\mathbb{E}_{\tau_{0:h}\sim p_{\bt_u}^h}\left[\bZ_h(u)\bZ_h(u)^\top(\bt-\bt')r_{h+1}\right],
    \end{align*}
    and 
    \begin{align*}
        \left\|\lim_{\epsilon\downarrow 0}\frac{(B)}{\epsilon}\right\| 
        &\le R_{\max}\mathbb{E}_{\tau_{0:h}\sim p_{\bt_u}^h}\left[\left\|\bZ_h(u)\right\|^2\right]\|\bt-\bt'\| \\
        &\le R_{\max}(h+1)C_1^2\|\bt-\bt'\|.
    \end{align*}
    Now, from Equation~\eqref{eq:diff_G}:
    \begin{align*}
        \limsup_{\epsilon\downarrow 0}\frac{\left\|\bG_h(u+\epsilon) - \bG_h(u)\right\|}{\epsilon} 
        &\le \limsup_{\epsilon\downarrow 0}\frac{\|(A)\|}{\epsilon} + \limsup_{\epsilon\downarrow 0}\frac{\|(B)\|}{\epsilon} \\
        &= \limsup_{\epsilon\downarrow 0}\frac{\|(A)\|}{\epsilon} + \left\|\lim_{\epsilon\downarrow 0}\frac{(B)}{\epsilon}\right\| \\
        &\le R_{\max}(h+1)(L_S+C_1^2)\|\bt-\bt'\|.
    \end{align*}
    The same decomposition shows that $\bG_h$ is continuous on $[0,1]$. By Lemma~\ref{lem:dini}:
    \begin{equation*}
        \|\bG_h(1)-\bG_h(0)\| \le R_{\max}(h+1)(L_S+C_1^2)\|\bt-\bt'\|.
    \end{equation*}
    Finally:
    \begin{align*}
        \|\nabla J(\bt')-\nabla J(\bt)\|
        &\le \sum_{h=0}^\infty \gamma^h \|\bG_h(1)-\bG_h(0)\| \\
        &\le R_{\max}(L_S+C_1^2)\|\bt-\bt'\|\sum_{h=0}^\infty \gamma^h(h+1) \\
        &=\frac{R_{\max}(L_S+C_1^2)}{(1-\gamma)^2}\|\bt-\bt'\|.
    \end{align*}
\end{proof}

\begin{lemma}[Weighted Hellinger continuity]
\label{lem:weighted-hellinger}
Suppose that Assumption~\ref{ass:LS}(1) holds. For every integer $r\geq0$ and every $\bt_1,\bt_2\in\Theta$,
\[
\int_{\mathcal T}
H_\gamma(\tau)^r
\left(
    \sqrt{p(\tau\mid\bt_1)}
    -
    \sqrt{p(\tau\mid\bt_2)}
\right)^2
\mathrm{d}\tau
\leq
\frac{C_1^2M_{r+1}}{4}
\left\|\bt_1-\bt_2\right\|^2.
\]
\end{lemma}
\begin{proof}
We recall a result by~\cite{paczolay2024sample}, proved here for completeness
Let
\[
\bt_t
=
(1-t)\bt_2+t\bt_1,
\qquad
\boldsymbol{\Delta}
=
\bt_1-\bt_2,
\qquad
t\in[0,1].
\]
Convexity of $\Theta$ guarantees that $\bt_t\in\Theta$. Because the stopping rule and the environment are independent of the policy parameters,
\[
\nabla\log p(\tau\mid\bt_t)
=
\sum_{h=0}^{H_\gamma(\tau)-1}
\nabla\log\pi_{\bt_t}(\ba_a\mid \bs_h).
\]
Consequently,
\[
\frac{\mathrm{d}}{\mathrm{d} t}
\sqrt{p(\tau\mid\bt_t)}
=
\frac{1}{2}
\sqrt{p(\tau\mid\bt_t)}
\left\langle
    \boldsymbol{\Delta},
    \nabla\log p(\tau\mid\bt_t)
\right\rangle.
\]

By the fundamental theorem of calculus and Cauchy--Schwarz,
\begin{align*}
&\left(
    \sqrt{p(\tau\mid\bt_1)}
    -
    \sqrt{p(\tau\mid\bt_2)}
\right)^2
\\
&\qquad\leq
\frac{\|\boldsymbol{\Delta}\|^2}{4}
\int_0^1
p(\tau\mid\bt_t)
\left\|
    \nabla\log p(\tau\mid\bt_t)
\right\|^2
\mathrm{d} t.
\end{align*}
Multiplying by $H_\gamma(\tau)^r$ and integrating over $\tau$ yields
\begin{align*}
&\int_{\mathcal T}
H_\gamma(\tau)^r
\left(
    \sqrt{p(\tau\mid\bt_1)}
    -
    \sqrt{p(\tau\mid\bt_2)}
\right)^2
\mathrm{d}\tau
\\
&\qquad\leq
\frac{\|\boldsymbol{\Delta}\|^2}{4}
\int_0^1
\mathbb E_{\tau\sim p(\cdot\mid\bt_t)}
\left[
    H_\gamma^r
    \left\|
        \sum_{h=0}^{H_\gamma-1}
        \nabla\log\pi_{\bt_t}(\ba_a\mid \bs_h)
    \right\|^2
\right]
\mathrm{d} t.
\end{align*}
The score increments are martingale differences. Conditionally on $H_\gamma$, Assumption~\ref{ass:LS}(i) therefore gives
\[
\mathbb E_{\tau\sim p(\cdot\mid\bt_t)}
\left[
    \left\|
        \sum_{h=0}^{H_\gamma-1}
        \nabla\log\pi_{\bt_t}(\ba_a\mid \bs_h)
    \right\|^2
    \,\middle|\,
    H_\gamma
\right]
\leq
C_1^2H_\gamma.
\]
Since the distribution of $H_\gamma$ is independent of $\bt_t$,
\[
\mathbb E
\left[
    H_\gamma^r
    \left\|
        \nabla\log p(\tau\mid\bt_t)
    \right\|^2
\right]
\leq
C_1^2M_{r+1}.
\]
Substitution proves the claim.
\end{proof}

\begin{lemma}[Estimation Error for \alg]
    \label{lemma:lvr_pg_variance}
    Under Assumption~\ref{ass:LS}, let $\Delta_k := \mE[\|\estGradK - \gradK\|_2^2]$ be the estimation error of the \alg estimator. Letting $\lambda \coloneqq 1 - (1-p)(1-\alpha)^2$ and letting $\Calg > 0$ denote the bound on the variance of the (defensive) importance weights, the following recurrence holds for all $k \in \dsb{K}$:
    \begin{align*}
        \Delta_{k+1} \le (1-\lambda) \Delta_k + \frac{2(1-\lambda)\Calg}{B} \mE[\|\bt_{k+1} - \bt_{k}\|_2^2] + C_{\nu},
    \end{align*}
    where $C_{\nu} \coloneqq \left( \frac{p}{N} + \frac{4\alpha^2(1-p)}{B} \right) \var$.
\end{lemma}  
\begin{proof}
    We analyze the estimation error conditioned on the Bernoulli random variable $c_k \sim \mathcal{B}(p)$. By the law of total expectation, the expected estimation error $\Delta_{k+1}$ decomposes as follows:
    \begin{align*}
        \Delta_{k+1}
        &=
        p \cdot \mE\left[
            \|\estGradKone-\gradKone\|_2^2
            \mid c_k=1
        \right]
        \\
        &\quad+
        (1-p)\cdot\mE\left[
            \|\estGradKone-\gradKone\|_2^2
            \mid c_k=0
        \right].
    \end{align*}
    When $c_k=1$, the estimator executes a full-batch refresh. Because this relies on a new batch of $N$ on-policy trajectories, the conditional error is bounded by the variance of the standard \gpomdp estimator obtained in Lemma~\ref{lemma:bounded_var}:
    \begin{align*}
        \mE\left[
            \|\estGradKone-\gradKone\|_2^2
            \mid c_k=1
        \right]
        \le \frac{\var}{N}.
    \end{align*}

    When $c_k=0$, the estimator applies the momentum update. Following the proof of Lemma~8 of~\cite{yuan2020stochastic}, we condition on the history $\mathcal{F}_k$ available after computing $\bt_{k+1}$ and before drawing $c_k$ and the fresh trajectories. Under this conditioning, $\bt_k$, $\bt_{k+1}$, and $\estGradK$ are fixed, while the fresh trajectories $\{\tau_i\}_{i=1}^{B}$ are independent samples from $\xi p_{\bt_{k+1}}+(1-\xi)p_{\bt_k}$. Substituting the momentum update and adding and subtracting the corresponding gradients gives:
    \begin{align*}
        \estGradKone-\gradKone
        &=
        (1-\alpha)(\estGradK-\gradK)
        \\
        &\quad+
        (1-\alpha)\left[
            \frac{1}{B}\sum_{i=1}^{B}
            \left(
                \omega_\xi(\bt_{k+1},\tau_i)
                \bm{g}(\tau_i\mid\bt_{k+1})
                -
                \omega_\xi(\bt_k,\tau_i)
                \bm{g}(\tau_i\mid\bt_k)
            \right)
            -\gradKone+\gradK
        \right]
        \\
        &\quad+
        \alpha\left[
            \frac{1}{B}\sum_{i=1}^{B}
            \omega_\xi(\bt_{k+1},\tau_i)
            \bm{g}(\tau_i\mid\bt_{k+1})
            -\gradKone
        \right].
    \end{align*}
    The three terms represent the previous estimation error, the error in estimating the consecutive-gradient difference, and the current-gradient estimation error, respectively. The latter two have zero conditional expectation by unbiasedness of the weighted estimates. Consequently, their cross terms with the previous estimation error vanish in conditional expectation. Applying $\|u+v\|_2^2\le 2\|u\|_2^2+2\|v\|_2^2$ to these two fresh error terms, and using independence across trajectories, yields
    \begin{align*}
        &\mE\left[
            \|\estGradKone-\gradKone\|_2^2
            \mid \mathcal{F}_k,c_k=0
        \right]
        \\
        &\qquad\le
        (1-\alpha)^2\|\estGradK-\gradK\|_2^2
        \\
        &\qquad\quad+
        \frac{2(1-\alpha)^2}{B}
        \mE_{\tau\sim \xi p_{\bt_{k+1}}+(1-\xi)p_{\bt_k}}
        \left[
            \left\|\omega_\xi(\bt_{k+1},\tau)\bm{g}(\tau\mid\bt_{k+1})
            -\omega_\xi(\bt_k,\tau)\bm{g}(\tau\mid\bt_k)
            -\gradKone+\gradK\right\|_2^2
        \right]
        \\
        &\qquad\quad+
        \frac{2\alpha^2}{B}
        \mE_{\tau\sim \xi p_{\bt_{k+1}}+(1-\xi)p_{\bt_k}}
        \left[
            \left\|\omega_\xi(\bt_{k+1},\tau)
            \bm{g}(\tau\mid\bt_{k+1})-\gradKone\right\|_2^2
        \right].
    \end{align*}
    The factors $1/B$ follow because the centered contributions from distinct trajectories are conditionally independent. For the consecutive-gradient error, removing the centering and applying Lemma~\ref{lemma:consecutive_etstimation_bound} gives
    \begin{align*}
        &\mE_{\tau\sim \xi p_{\bt_{k+1}}+(1-\xi)p_{\bt_k}}
        \left[
            \left\|\omega_\xi(\bt_{k+1},\tau)\bm{g}(\tau\mid\bt_{k+1})
            -\omega_\xi(\bt_k,\tau)\bm{g}(\tau\mid\bt_k)
            -\gradKone+\gradK\right\|_2^2
        \right]
        \\
        &\qquad\le
        \mE_{\tau\sim \xi p_{\bt_{k+1}}+(1-\xi)p_{\bt_k}}
        \left[
            \left\|\omega_\xi(\bt_{k+1},\tau)\bm{g}(\tau\mid\bt_{k+1})
            -\omega_\xi(\bt_k,\tau)\bm{g}(\tau\mid\bt_k)\right\|_2^2
        \right]
        \\
        &\qquad\le
        \Calg\|\bt_{k+1}-\bt_k\|_2^2.
    \end{align*}
    Here, the first inequality uses the fact that the expectation of the weighted gradient difference is $\gradKone-\gradK$. The second inequality is the mean-squared error bound provided by Lemma~\ref{lemma:consecutive_etstimation_bound}. For the current-gradient error, we similarly obtain:
    \begin{align*}
        &\mE_{\tau\sim
            \xi p_{\bt_{k+1}}+(1-\xi)p_{\bt_k}}
        \left[
            \left\|
                \omega_\xi(\bt_{k+1},\tau)
                \bm{g}(\tau\mid\bt_{k+1})
                -\gradKone
            \right\|_2^2
        \right]
        \\
        &\qquad\le
        \mE_{\tau\sim
            \xi p_{\bt_{k+1}}+(1-\xi)p_{\bt_k}}
        \left[
            \left\|
                \omega_\xi(\bt_{k+1},\tau)
                \bm{g}(\tau\mid\bt_{k+1})
            \right\|_2^2
        \right]
        \\
        &\qquad=
        \mE_{\tau\sim p_{\bt_{k+1}}}
        \left[
            \omega_\xi(\bt_{k+1},\tau)
            \|\bm{g}(\tau\mid\bt_{k+1})\|_2^2
        \right]
        \\
        &\qquad\le
        \frac{1}{\xi}
        \mE_{\tau\sim p_{\bt_{k+1}}}
        \left[
            \|\bm{g}(\tau\mid\bt_{k+1})\|_2^2
        \right]
        \le \frac{\var}{\xi}.
    \end{align*}
    The equality follows by canceling one importance weight against the mixture distribution. The final inequalities use $\omega_\xi(\bt_{k+1},\tau)\le 1/\xi$ and the uncentered second-moment bound established in the proof of Lemma~\ref{lemma:bounded_var}. Substituting these bounds into the conditional estimation error inequality gives:
    \begin{align*}
        &\mE\left[
            \|\estGradKone-\gradKone\|_2^2
            \mid \mathcal{F}_k,c_k=0
        \right]
        \\
        &\qquad\le
        (1-\alpha)^2\|\estGradK-\gradK\|_2^2
        +
        \frac{2(1-\alpha)^2\Calg}{B}
        \|\bt_{k+1}-\bt_k\|_2^2
        +
        \frac{2\alpha^2\var}{\xi B}.
    \end{align*}
    Taking expectation over the preceding history, and using the independence of $c_k$ from that history, yields,
    \begin{align*}
        &\mE\left[
            \|\estGradKone-\gradKone\|_2^2
            \mid c_k=0
        \right]
        \\
        &\qquad\le
        (1-\alpha)^2\Delta_k
        +
        \frac{2(1-\alpha)^2\Calg}{B}
        \mE\left[\|\bt_{k+1}-\bt_k\|_2^2\right]
        +
        \frac{2\alpha^2\var}{\xi B}.
    \end{align*}
    Summing the two branches, we obtain,
    \begin{align*}
        \Delta_{k+1}
        &\le
        p\left(\frac{\var}{N}\right)
        +(1-p)\left[
            (1-\alpha)^2\Delta_k
            +
            \frac{2(1-\alpha)^2\Calg}{B}
            \mE\left[\|\bt_{k+1}-\bt_k\|_2^2\right]
            +
            \frac{2\alpha^2\var}{\xi B}
        \right]
        \\
        &=
        (1-p)(1-\alpha)^2\Delta_k
        +
        \frac{2(1-p)(1-\alpha)^2\Calg}{B}
        \mE\left[\|\bt_{k+1}-\bt_k\|_2^2\right]
        +
        \frac{p\var}{N}
        +
        \frac{2\alpha^2(1-p)\var}{\xi B}.
    \end{align*}
    Recognizing that $(1-p)(1-\alpha)^2=1-\lambda$, and plugging our choice of $\xi = 1/2$ from Lemma~\ref{lemma:consecutive_etstimation_bound}, this yields the recurrence:
    \begin{align*}
        \Delta_{k+1}
        \le
        (1-\lambda)\Delta_k
        +
        \frac{2(1-\lambda)\Calg}{B}
        \mE\left[\|\bt_{k+1}-\bt_k\|_2^2\right]
        +
        \left(
            \frac{p}{N}
            +
            \frac{4\alpha^2(1-p)}{B}
        \right)\var.
    \end{align*}
\end{proof}

\section{Convergence}
\label{apx:convergence}

In this section, the last-iterate global-convergence guarantee for the \alg algorithm is provided. In particular, Theorem~\ref{thm:lvr_pg_convergence} provides the sample-complexity for the \alg framework, and while Corollaries~\ref{cor:def_page_pg} and~\ref{cor:def_storm_pg} provide the sample complexities \defpagepg and \defstormpg as instantiations of \alg. Finally, Corollary~\ref{cor:convergence_FND} provides the convergence result for \alg when instantiated for Fisher-non-degenerate policies. The lemmas used to prove the results of this section are provided in Appendix~\ref{apx:lemmas} and~\ref{apx:technical}.

\subsection{Proof of Theorem~\ref{thm:lvr_pg_convergence}}

Here, the proof of Theorem~\ref{thm:lvr_pg_convergence} is provided. Leveraging the proof technique in Theorem 4.5 of~\cite {xu2019sample} for projected gradient ascent, we construct a recurrence that depends on both the sub-optimality gap and the variance of the estimator. We show that this recurrence achieves minimax-optimal sample complexity up to logarithmic factors and horizon terms. 

\algConvergence*
\begin{proof}
    Following the proof of Theorem 4.5 by \cite{xu2019sample}, we start by redefining the objective function to incorporate the constraint set via the non-smooth indicator function $\mathbb{I}_{\Theta}(\bt)$, introducing $\mathcal{J}(\bt) \coloneqq J(\bt) - \mathbb{I}_{\Theta}(\bt)$. We formally define the generalized gradient mapping as $\mathcal{G}_{\eta}(\bt) \coloneqq \frac{1}{\eta}(\mathcal{P}_{\Theta}(\bt + \eta \nabla J(\bt)) - \bt)$. Note that the mapping $\mathcal{G}_{\eta}$ serves as a generalized projected gradient at $\bt$. By definition, in the unconstrained setting where $\Theta=\mathbb{R}^{d_{\Theta}}$, this recovers the standard gradient, yielding $\mathcal{G}_{\eta}(\bt)=\nabla_{\bt}J(\bt)$. According to the definition of $\mathbb{I}_{\Theta}(\bt)$:
    \begin{align*}
        \mathbb{I}_\Theta(\bm{u}) = 
        \begin{cases}
            0,&\text{if}~\bm{u} \in \Theta\\
            +\infty&\text{otherwise}
        \end{cases},
    \end{align*}
    we have $\text{arg}\max_{\bt \in \mathbb{R}^{d_\Theta}} \mathcal{J}(\bt) = \text{arg}\max_{\bt \in \Theta} J(\bt) \coloneqq \bt^*$. Letting $\tilde{\mathcal{J}}^* \coloneqq \mathcal{J}^* - \psi$ and defining $\mu \coloneqq 1/\chi^2$, we reformulate the WGD condition (Assumption~\ref{ass:WGD}) over the bounded domain using this gradient mapping:
    \begin{align*}
        \sqrt{\mu}\max\{0, \tilde{\mathcal{J}}^* - \mathcal{J}(\bt_k)\} \le \|\mathcal{G}_{\eta}(\bt_k)\|_2
    \end{align*}

    Considering an iteration $k \in \dsb{K}$, the parameter update applies projected stochastic gradient ascent: $\bt_{k+1} = \mathcal{P}_\Theta(\bt_k + \eta \estGradK)$. By expanding the Euclidean projection and discarding terms independent of the optimization variable, this is equivalent to:
    \begin{align*}
        \bt_{k+1} = \arg\min_{\bm{u} \in \mathbb{R}^d} \left\{ \mathbb{I}_\Theta(\bm{u}) + \frac{1}{2\eta} \|\bm{u} - \bt_k\|_2^2 - \langle \estGradK, \bm{u} \rangle \right\}.
    \end{align*}

    We define the stochastic gradient mapping as $\tilde{\mathcal{G}}_k \coloneqq \frac{1}{\eta}(\bt_{k+1} - \bt_k)$. Let $\bm{g} \in \partial \mathbb{I}_\Theta(\bt_{k+1})$ be a subgradient of the indicator function. According to the optimality condition of the projection step, we have $\bm{g} + \frac{1}{\eta}(\bt_{k+1} - \bt_k) - \estGradK = 0$. By the convexity of $\mathbb{I}_\Theta$, it follows that:
    \begin{align*}
        \mathbb{I}_\Theta(\bt_{k+1}) &\le \mathbb{I}_\Theta(\bt_k) + \langle \bm{g}, \bt_{k+1} - \bt_k \rangle \\
        &= \mathbb{I}_\Theta(\bt_k) - \left\langle \frac{1}{\eta}(\bt_{k+1} - \bt_k) - \estGradK, \bt_{k+1} - \bt_k \right\rangle.
    \end{align*}

    From the $L$-smoothness of $J(\bt)$ obtained in Lemma~\ref{lem:smooth}, we can write the ascent lemma as follows:
    \begin{align*}
        J(\bt_{k+1}) \ge J(\bt_k) + \langle \gradK, \bt_{k+1} - \bt_k \rangle - \frac{L}{2}\|\bt_{k+1} - \bt_k\|_2^2.
    \end{align*}
    
    We construct the one-step descent for the constrained objective $\constrPerf(\bt)$ by subtracting the indicator function upper bound from the smoothness lower bound. Expanding the inner product isolates the quadratic penalty, yielding:
    \begin{align*}
        \constrPerf(\bt_{k+1}) &\ge \constrPerf(\bt_k) + \langle \gradK - \estGradK, \bt_{k+1} - \bt_k \rangle + \frac{1}{\eta}\|\bt_{k+1} - \bt_k\|_2^2 - \frac{L}{2}\|\bt_{k+1} - \bt_k\|_2^2 \\
        &= \constrPerf(\bt_k) + \eta \langle \gradK - \estGradK, \tilde{\mathcal{G}}_k \rangle + \eta \|\tilde{\mathcal{G}}_k\|_2^2 - \frac{L}{2}\|\bt_{k+1} - \bt_k\|_2^2.
    \end{align*}

    Applying Young's inequality $\langle x, y \rangle \ge - \frac{1}{2}\|x\|_2^2 - \frac{1}{2}\|y\|_2^2$ to the inner product, we bound the relation:
    \begin{align*}
        \constrPerf(\bt_{k+1}) &\ge \constrPerf(\bt_k) - \frac{\eta}{2}\|\gradK - \estGradK\|_2^2 + \frac{\eta}{2}\|\tilde{\mathcal{G}}_k\|_2^2 - \frac{L}{2}\|\bt_{k+1} - \bt_k\|_2^2 \\
        &= \constrPerf(\bt_k) - \frac{\eta}{2}\|\gradK - \estGradK\|_2^2 + \frac{\eta}{4}\|\tilde{\mathcal{G}}_k\|_2^2 + \left(\frac{1}{4\eta} - \frac{L}{2}\right)\|\bt_{k+1} - \bt_k\|_2^2.
    \end{align*}

    By expanding the definitions of the gradient mappings and applying the property of the Euclidean projection operator $\|\mathcal{P}_\Theta(x) - \mathcal{P}_\Theta(y)\|_2 \le \|x - y\|_2$, we bound the difference between the exact and stochastic gradient mappings:
    \begin{align*}
        \|\mathcal{G}_\eta(\bt_k) - \tilde{\mathcal{G}}_k\|_2 &= \frac{1}{\eta} \left\| \mathcal{P}_\Theta(\bt_k + \eta \gradK) - \mathcal{P}_\Theta(\bt_k + \eta \estGradK) \right\|_2 \\
        &\le \frac{1}{\eta} \left\| (\bt_k + \eta \gradK) - (\bt_k + \eta \estGradK) \right\|_2 \\
        &= \|\gradK - \estGradK\|_2.
    \end{align*}
    Utilizing the inequality $\|\mathcal{G}_\eta(\bt_k)\|_2^2 \le 2\|\tilde{\mathcal{G}}_k\|_2^2 + 2\|\mathcal{G}_\eta(\bt_k) - \tilde{\mathcal{G}}_k\|_2^2$, which follows from adding and subtracting the stochastic gradient mapping and applying Young's inequality, we bound the remaining mapping term:
    \begin{align*}
        \frac{\eta}{4}\|\tilde{\mathcal{G}}_k\|_2^2 \ge \frac{\eta}{8}\|\mathcal{G}_\eta(\bt_k)\|_2^2 - \frac{\eta}{4}\|\gradK - \estGradK\|_2^2.
    \end{align*}

    Substituting this lower bound yields the final single-step proximal descent inequality:
    \begin{align}
        \constrPerf(\bt_{k+1}) \ge \constrPerf(\bt_k) - \frac{3\eta}{4}\|\gradK - \estGradK\|_2^2 + \frac{\eta}{8}\|\mathcal{G}_\eta(\bt_k)\|_2^2 + \left(\frac{1}{4\eta} - \frac{L}{2}\right)\|\bt_{k+1} - \bt_k\|_2^2. \label{proof:potential_lvr_1}
    \end{align}

    Reversing the sign, applying the total expectation, and substituting the suboptimality $r_k \coloneqq \mE [\tilde{\constrPerf}^* - \constrPerf(\bt_k)]$ and the estimation error $\Delta_k \coloneqq \mE[\|\estGradK - \gradK\|_2^2]$, we obtain the projected suboptimality recurrence:
    \begin{align}
        r_{k+1} \leq r_k - \frac{\eta}{8} \mE\left[\|\mathcal{G}_\eta(\bt_k)\|_2^2\right] + \frac{3\eta}{4}\Delta_k - \left(\frac{1}{4\eta} - \frac{L}{2}\right)\mE\left[\|\bt_{k+1} - \bt_k\|_2^2\right].\label{proof:potential_lvr_2}
    \end{align}

    From Lemma~\ref{lemma:lvr_pg_variance}, the unified variance recurrence uses the effective decay rate $\lambda \coloneqq 1-(1-p)(1-\alpha)^2$ and noise floor $C_{\nu} \coloneqq \left(\frac{p}{N} + \frac{4\alpha^2(1-p)}{B}\right)\nu^2$. The variance recurrence is:
    \begin{align}
        \Delta_{k+1} \le (1-\lambda) \Delta_k + \frac{(1-\lambda)\Calg}{B} \mE\left[\|\bt_{k+1} - \bt_k\|_2^2\right] + C_{\nu}.\label{proof:potential_lvr_3}
    \end{align}

    We introduce the potential function $\Phi_k \coloneqq r_k + c\Delta_k$ with a coupling constant $c > 0$. Because the estimation error $\Delta_k$ is intrinsically non-negative, the potential strictly upper-bounds the suboptimality ($\Phi_k \ge r_k$). Consequently, establishing the convergence of $\Phi_k$ is sufficient to guarantee the convergence of $r_k$. Combining equations~\eqref{proof:potential_lvr_2} and~\eqref{proof:potential_lvr_3}, we construct the unified recurrence:
    \begin{align*}
        \Phi_{k+1} &\leq r_k - \frac{\eta}{8} \mE\left[\|\mathcal{G}_\eta(\bt_k)\|_2^2\right] + \frac{3\eta}{4}\Delta_k - \left(\frac{1}{4\eta} - \frac{L}{2}\right)\mE\left[\|\bt_{k+1} - \bt_k\|_2^2\right]\\
        & \quad + c\left((1-\lambda) \Delta_k + \frac{2(1-\lambda)\Calg}{B} \mE\left[\|\bt_{k+1} - \bt_k\|_2^2\right] + C_{\nu}\right)\\
        & = \Phi_k - \frac{\eta}{8} \mE\left[\|\mathcal{G}_\eta(\bt_k)\|_2^2\right] - \left(c\lambda - \frac{3\eta}{4}\right)\Delta_k \\
        & \quad - \left(\frac{1}{4\eta} - \frac{L}{2} - c\frac{2(1-\lambda)\Calg}{B}\right)\mE\left[\|\bt_{k+1} - \bt_k\|_2^2\right] + c C_{\nu}.
    \end{align*}

    To simplify the recurrence, we determine the constant $c$ by requiring the $\Delta_k$ coefficient to retain an exact $-\frac{3\eta}{4}$ residual descent rate:
    \begin{align*}
        c\lambda - \frac{3\eta}{4} = \frac{3\eta}{4} \Rightarrow c = \frac{3\eta}{2\lambda}.
    \end{align*}

    Substituting this choice of $c$, we require the sequence distance coefficient to remain non-negative to eliminate it from the upper bound:
    \begin{align*}
        \frac{1}{4\eta} - \frac{L}{2} - \frac{3\eta(1-\lambda)\Calg}{\lambda B} \ge 0.
    \end{align*}

    To ensure the sequence distance coefficient provides sufficient slack to absorb the variance accumulation, we tighten the standard smoothness restriction. By rigorously restricting the baseline learning rate to $\eta \le \frac{1}{4L}$, we enforce the bound $\frac{L}{2} \le \frac{1}{8\eta}$, which guarantees the strictly positive residual $\frac{1}{4\eta} - \frac{L}{2} \ge \frac{1}{4\eta} - \frac{1}{8\eta} = \frac{1}{8\eta}$. Substituting this residual isolates the maximum stable learning rate required to absorb the variance accumulation:
    \begin{align*}
        \frac{1}{8\eta} \ge \frac{3\eta(1-\lambda)\Calg}{\lambda B} \Rightarrow \eta \leq \sqrt{\frac{\lambda B}{24 (1-\lambda) \Calg}}.
    \end{align*}

    Under these stability conditions, the recurrence simplifies to:
    \begin{align}
        \Phi_{k+1} \leq \Phi_{k} - \frac{\eta}{8}\mE\left[\|\mathcal{G}_\eta(\bt_k)\|_2^2\right] - \frac{3\eta}{4}\Delta_k + \frac{3\eta C_{\nu}}{2\lambda}.\label{proof:potential_lvr_4}
    \end{align}

    Substituting the modified gradient domination condition and applying Jensen's inequality for both the maximum and square operations yields:
    \begin{align}
        \Phi_{k+1} \leq \Phi_k - \frac{\mu\eta}{8}\max\left\{0, r_k\right\}^2 - \frac{3\eta}{4}\Delta_k + \frac{3\eta C_{\nu}}{2\lambda}.\label{proof:potential_lvr_5}
    \end{align}

    To establish a last-iterate guarantee, we must prove that the algorithm does not leave the target accuracy region once it enters it. We explicitly define the non-negative sequence $x_k \coloneqq \max\{0, \Phi_k\}$ and show it unconditionally satisfies a unified descent recurrence. We evaluate the transition for $x_k$ based on the sign of $\Phi_k$. If $\Phi_k \le 0$, we have $x_k = 0$. From Equation~\eqref{proof:potential_lvr_4}, dropping the non-positive descent terms yields $\Phi_{k+1} \le \Phi_k + \frac{3\eta C_{\nu}}{2\lambda} \le \frac{3\eta C_{\nu}}{2\lambda}$. Taking the positive part of both sides guarantees $x_{k+1} \le \frac{3\eta C_{\nu}}{2\lambda}$, which satisfies the target recurrence evaluated at $x_k = 0$. Conversely, if $\Phi_k > 0$, we have $x_k = \Phi_k$. Because $\Phi_k \leq \max\{0, r_k\} + c \Delta_k$, it follows that at least one of the following conditions must hold: either $\max\{0, r_k\} \ge \frac{\Phi_k}{2}$, or $c\Delta_k \ge \frac{\Phi_k}{2}$. If the first case holds ($\max\{0, r_k\} \ge \frac{\Phi_k}{2}$), we bound the suboptimality descent while discarding the negative variance term:
    \begin{align*}
        - \frac{\mu\eta}{8}\max\left\{0, r_k\right\}^2 \le - \frac{\mu\eta}{8}\left(\frac{x_k}{2}\right)^2 = -\frac{\mu\eta}{32}x_k^2.
    \end{align*}
    
    If the second case holds ($c\Delta_k \ge \frac{\Phi_k}{2}$), utilizing our choice of $c = 3\eta/2\lambda$ yields $\Delta_k \ge \frac{\lambda x_k}{3\eta}$. Discarding the suboptimality term bounds the variance descent term:
    \begin{align*}
        -\frac{3\eta}{4}\Delta_k \le -\frac{3\eta}{4}\left(\frac{\lambda x_k}{3\eta}\right) = -\frac{\lambda}{4}x_k.
    \end{align*}
    
    Bounding the descent in Equation~\eqref{proof:potential_lvr_5} by the minimum of these two rates, and observing that taking the maximum operator $x_{k+1} = \max\{0, \Phi_{k+1}\}$ on the left preserves the inequality because the right-hand side is inherently non-negative or bounded safely by $0$, we obtain the final unconditional recurrence:
    \begin{align*}
        x_{k+1} \leq x_k - \min \left\{ \frac{\mu\eta}{32}x_k^2, \frac{\lambda}{4}x_k \right\} + \frac{3\eta C_{\nu}}{2\lambda}.
    \end{align*}
    
    We directly apply Lemma~\ref{lemma:linear_recurrence} to the invariant sequence $x_k$ with $a = \frac{\mu\eta}{32}$, $b = \frac{\lambda}{4}$, and $e = \frac{3\eta C_{\nu}}{2\lambda}$. Because $\lambda \le 1$ structurally, we inherently satisfy the lemma's contraction requirement $\rho_\epsilon \le b = \frac{\lambda}{4} < 1$. Hence, from Lemma~\ref{lemma:linear_recurrence}, the iteration complexity is described by:

    \begin{align}
        K \ge \max \left\{ \frac{32}{\mu\eta\epsilon}, \frac{4}{\lambda} \right\} \log_{+} \frac{2\Phi_0}{\epsilon}.\label{proof:potential_lvr_6}
    \end{align}
    
    To guarantee convergence to the target accuracy $\epsilon$, the noise term must satisfy two conditions. The first condition $e \le a \frac{\epsilon^2}{2}$ requires $\frac{3\eta C_{\nu}}{2\lambda} \le \frac{\mu\eta\epsilon^2}{64}$, which simplifies to bounding the noise floor $C_{\nu} \le \frac{\mu\lambda\epsilon^2}{96}$. The second condition $e \le b \frac{\epsilon}{2}$ requires $\frac{3\eta C_{\nu}}{2\lambda} \le \frac{\lambda\epsilon}{8}$, which dictates the maximum stable learning rate $\eta \le \frac{\lambda^2 \epsilon}{12 C_{\nu}}$. We satisfy the first accuracy condition $C_{\nu} \le \frac{\mu\lambda\epsilon^2}{96}$ by configuring the algorithmic parameters such that each of the two additive terms in $C_{\nu} = \left( \frac{p}{N} + \frac{4\alpha^2(1-p)}{B} \right)\nu^2$ is independently bounded by half of the target limit, $\frac{\mu\lambda\epsilon^2}{192}$. With the unified decay rate $\lambda = p + (1-p)(2\alpha - \alpha^2)$, we evaluate both terms to define the feasible geometric region.

    \paragraph{First Term of $C_{\nu}$ and Batch Size.} Observing that $2\alpha-\alpha^2 \ge 0$, we establish the lower bound $\lambda \ge p$. It is therefore sufficient to enforce:
    \begin{align*}
        \frac{p\nu^2}{N} \le \frac{\mu p \epsilon^2}{192} \implies N \ge \frac{192\nu^2}{\mu \epsilon^2}.
    \end{align*}
    To minimize sample cost while satisfying this requirement, we select a full batch size $N = \Theta(\epsilon^{-2})$ and initialize with $N_0 = N = \Theta(\epsilon^{-2})$.
    
    \paragraph{Second Term of $C_{\nu}$ and Momentum.} Observing that $2\alpha - \alpha^2 = \alpha(2-\alpha) \ge \alpha$ for any $\alpha \in (0, 1]$, and dropping the positive $p$ term from $\lambda$, we obtain the lower bound $\lambda \ge (1-p)\alpha$. Substituting this lower bound into the right-hand side, it is strictly sufficient to require:
    \begin{align*}
        \frac{4\alpha^2(1-p)\nu^2}{B} \le \frac{\mu(1-p)\alpha\epsilon^2}{192}.
    \end{align*}
    Because $\alpha > 0$, we divide both sides by the common factor $\alpha(1-p)$ to directly isolate $\alpha$:
    \begin{align*}
        \frac{4\alpha\nu^2}{B} \le \frac{\mu\epsilon^2}{192} \implies \alpha \le \frac{\mu B \epsilon^2}{768\nu^2}.
    \end{align*}
    This establishes the strict upper bound $\alpha = \mathcal{O}(B\epsilon^2)$.
    
    \paragraph{Expected Sample Cost and Refresh Probability.}
    We choose $p,\alpha\ge 0$ such that
    \begin{align*}
        p+\alpha=\frac{\mu B\epsilon^2}{768\nu^2}.
    \end{align*}
    This ensures the momentum constraint established above. For sufficiently small $\epsilon$, we have $p+\alpha\le 1/2$, and hence
    \begin{align*}
        p+\alpha
        \le \lambda
        =p+(1-p)(2\alpha-\alpha^2)
        \le p+2\alpha
        \le 2(p+\alpha).
    \end{align*}
    Therefore, $\lambda=\Theta(\mu B\epsilon^2/\nu^2)$, with universal implicit constants. Moreover, since $p\le\mu B\epsilon^2/(768\nu^2)$,
    \begin{align*}
        pN+(1-p)B
        &\le
        \frac{\mu B\epsilon^2}{768\nu^2}
        \left(\frac{192\nu^2}{\mu\epsilon^2}+1\right)+B \\
        &\le \frac{5B}{4}+\frac{1}{2}
        =\mathcal{O}(B).
    \end{align*}
    
    \paragraph{Learning Rate.} To achieve the optimal total sample complexity $n = \tilde{\mathcal{O}}(\epsilon^{-2})$, and match the lower-bound established by Theorem~\ref{thr:lower_bound} up to logarithmic factors, the iteration complexity $K$ must be bounded by $\mathcal{O}((B\epsilon^2)^{-1})$. This bound limits the total number of optimization steps to guarantee that the cumulative sample cost of the $\mathcal{O}(B)$ per-iteration updates does not breach the optimal $\tilde{\mathcal{O}}(\epsilon^{-2})$ global sample complexity. Given the initial batch $N_0 = \Theta(\epsilon^{-2})$ and an expected per-iteration sample cost bounded by $\mathcal{O}(B)$, the number of iterations must satisfy $K \cdot \mathcal{O}(B) \le \tilde{\mathcal{O}}(\epsilon^{-2})$, which mandates $K \le \mathcal{O}((B\epsilon^2)^{-1})$. The sufficient iteration budget in Equation~(15) contains the estimator-decay term $4/\lambda$. To keep this term within the desired iteration budget $\mathcal{O}((B\epsilon^2)^{-1})$, we impose $\lambda=\Omega(B\epsilon^2)$. This condition controls the iteration complexity of our analysis; it does not establish that configurations with smaller $\lambda$ necessarily require more iterations to reach the target accuracy. Expanding the definition of the decay rate yields: $\lambda = 1 - (1-p)(1-\alpha)^2 = p + 2\alpha - \alpha^2 - 2p\alpha + p\alpha^2$. Since the upper bounds $p = \mathcal{O}(B\epsilon^2)$ and $\alpha = \mathcal{O}(B\epsilon^2)$ are established, the higher-order terms $\alpha^2$, $p\alpha$, and $p\alpha^2$ scale as $\mathcal{O}(B^2\epsilon^4)$ or smaller. Consequently, for sufficiently small target accuracy $\epsilon \to 0$, the decay rate is dominated by its linear components, reducing asymptotically to $\lambda = p + 2\alpha - \mathcal{O}(B^2\epsilon^4)$. To satisfy the $\lambda = \Omega(B\epsilon^2)$ decay floor, these dominant terms must jointly fulfill the geometric condition $p + \alpha = \Omega(B\epsilon^2)$. For any configuration $(p, \alpha)$ satisfying this combined system of bounds, the decay rate and noise floor scale precisely as $\lambda = \Theta(B\epsilon^2)$ and $C_{\nu} = \Theta(B\epsilon^4)$, respectively. We evaluate the remaining step-size restrictions using these asymptotic rates. The last-iterate stability condition $\eta \le \frac{\lambda^2 \epsilon}{12 C_{\nu}}$ resolves to $\eta \le \Theta(B\epsilon)$. Concurrently, the variance recurrence threshold evaluates to $\sqrt{\frac{\lambda B}{12(1-\lambda)\Calg}} = \Theta(B\epsilon)$. By explicitly selecting the step size $\eta$ to match this scaling alongside the $L$-smoothness limit, all stability requirements are unconditionally satisfied:
    \begin{align*}
        \eta = \min \left\{ \frac{1}{4L}, \frac{\lambda^2 \epsilon}{12 C_{\nu}}, \sqrt{\frac{\lambda B}{24 (1-\lambda) \Calg}} \right\} = \Theta(B\epsilon).
    \end{align*}
    
   \paragraph{Sample Complexity.}
    With these conditions established, Lemma~\ref{lemma:linear_recurrence} guarantees that it suffices to choose the number of iterations as:
    \begin{align*}
        K=\left\lceil
            \max\left\{
                \frac{32}{\mu\eta\epsilon},
                \frac{4}{\lambda}
            \right\}
            \log_{+}\frac{2\Phi_0}{\epsilon}
        \right\rceil.
    \end{align*}
    Substituting our choice of $\eta$ into the iteration bound gives:
    \begin{align*}
        K=\left\lceil
            \max\left\{
                \frac{128L}{\mu\epsilon},
                \frac{384C_\nu}{\mu\lambda^2\epsilon^2},
                \frac{32\sqrt{24(1-\lambda)\Calg}}
                     {\mu\epsilon\sqrt{\lambda B}},
                \frac{4}{\lambda}
            \right\}
            \log_{+}\frac{2\Phi_0}{\epsilon}
        \right\rceil.
    \end{align*}
    The noise constraint $C_\nu\le\mu\lambda\epsilon^2/96$
    implies
    \begin{align*}
        \frac{384C_\nu}{\mu\lambda^2\epsilon^2}
        \le
        \frac{384}{\mu\lambda^2\epsilon^2}
        \frac{\mu\lambda\epsilon^2}{96}
        =\frac{4}{\lambda}.
    \end{align*}
    Thus, this term is absorbed by the estimator-decay term. Moreover, the established bound
    $\lambda\ge\mu B\epsilon^2/(768\nu^2)$ yields
    \begin{align*}
        \frac{4}{\lambda}
        \le \frac{3072\nu^2}{\mu B\epsilon^2}.
    \end{align*}
    Using the same lower bound on $\lambda$ and $1-\lambda\le 1$, the variance-stability term satisfies
    \begin{align*}
        \frac{32\sqrt{24(1-\lambda)\Calg}}
             {\mu\epsilon\sqrt{\lambda B}}
        &\le
        \frac{32\sqrt{24\Calg}}{\mu\epsilon}
        \sqrt{\frac{768\nu^2}{\mu B^2\epsilon^2}} \\
        &=\frac{3072\sqrt{2}\,\nu\sqrt{\Calg}}
                {\mu^{3/2}B\epsilon^2}.
    \end{align*}
    Combining these bounds, absorbing numerical constants and logarithmic factors, and accounting for the ceiling, we obtain:
    \begin{align*}
        K=\widetilde{\mathcal{O}}\left(
            1
            +\frac{L}{\mu\epsilon}
            +\frac{\nu^2}{\mu B\epsilon^2}
            +\frac{\nu\sqrt{\Calg}}
                   {\mu^{3/2}B\epsilon^2}
        \right).
    \end{align*}
    
    The expected total sample complexity $n$ is determined by the initial batch size $N_0$ plus the sum of the expected batch sizes drawn over the $K$ iterations:
    \begin{align*}
        n=N_0+K(pN+(1-p)B).
    \end{align*}
    By substituting $N_0=\left\lceil192\nu^2/(\mu\epsilon^2)\right\rceil$,
    the iteration bound above, and the bounded per-iteration cost $pN+(1-p)B=\mathcal{O}(B)$, we obtain:
    \begin{align*}
        n=\widetilde{\mathcal{O}}\left(
            1+B
            +\frac{BL}{\mu\epsilon}
            +\frac{\nu^2}{\mu\epsilon^2}
            +\frac{\nu\sqrt{\Calg}}
                   {\mu^{3/2}\epsilon^2}
        \right).
    \end{align*}
    Thus, for fixed $B$ and fixed positive problem-dependent constants, the complexity is $n=\widetilde{\mathcal{O}}(\epsilon^{-2})$ as $\epsilon\to 0$. To establish the asymptotic dependence on the reward bounds and the effective horizon, we substitute the problem-dependent constants. From Lemma~\ref{lemma:bounded_var}, the \gpomdp second-moment bound scales as
    $\nu^2=\mathcal{O}(R_{\max}^2(1-\gamma)^{-3})$. From Lemma~\ref{lemma:consecutive_etstimation_bound}, the defensive mean-squared smoothness constant scales as $\Calg=\mathcal{O}(R_{\max}^2(1-\gamma)^{-5})$. Substituting these bounds into the terms proportional to $\epsilon^{-2}$, and using $\mu=\chi^{-2}$, yields:
    \begin{align*}
        n=\widetilde{\mathcal{O}}\left(
            \frac{R_{\max}^2}{\epsilon^2}
            \left[
                \frac{\chi^2}{(1-\gamma)^3}
                +\frac{\chi^3}{(1-\gamma)^4}
            \right]
        \right),
    \end{align*}
    where the policy-regularity and defensive-mixture constants are treated as fixed. This establishes the minimax-optimal dependence on $\epsilon$, up to logarithmic factors, while retaining the explicit dependence on the WGD constant and the effective horizon.
\end{proof}

\subsection{Corollaries of Theorem~\ref{thm:lvr_pg_convergence}}

Here, we consider the sample complexity of two specific instantiations of \alg: \defpagepg and \defstormpg. First, we show how the parameter $\lambda$, defined in Lemma~\ref{lemma:consecutive_etstimation_bound}, changes setting the $\alpha$ and $p$ parameters. Then, we study the batch size $N$, mini-batch size $B$ and learning rate $\eta$ to recover the desired sample complexity. 

\begin{coroll}[Sample Complexity of \defpagepg]
    \label{cor:def_page_pg}
    Under Assumptions~\ref{ass:LS}--\ref{ass:WGD}, for a fixed mini-batch size $B\ge 1$ and sufficiently small target accuracy $\epsilon>0$, through an appropriate choice of the algorithmic parameters (full batch size
    $N$, initial batch size $N_0$, refresh probability $p$, and learning rate $\eta$), \defpagepg with $\alpha=0$ guarantees $J^*-\mE[\perfK]\le\psi+\epsilon$ with a sample complexity of at most:
    \begin{align*}
        n
        &=N_0+K\left(1+B-\frac{B}{N}\right) \\
        &=\widetilde{\mathcal{O}}\left(
            \frac{R_{\max}^2}{\epsilon^2}
            \left[
                \frac{\chi^2}{(1-\gamma)^3}
                +\frac{\chi^3}{(1-\gamma)^4}
            \right]
        \right).
    \end{align*}
    Here, $\widetilde{\mathcal{O}}(\cdot)$ hides logarithmic factors, and $B$, the policy-regularity constants, and the defensive-mixture parameter are treated as fixed.
\end{coroll}

\begin{proof}
    Setting the momentum parameter to $\alpha=0$ specializes the generic variance recurrence used in the proof of Theorem~\ref{thm:lvr_pg_convergence}
    to the \defpagepg estimator. Under this configuration, the effective decay rate $\lambda$ and $C_\nu$ simplify to:
    \begin{align*}
        \lambda=p,
        \qquad
        C_\nu=\frac{p\nu^2}{N}.
    \end{align*}
    For $p>0$, the noise constraint $C_\nu\le\mu\lambda\epsilon^2/96$ reduces to:
    \begin{align*}
        \frac{p\nu^2}{N}
        \le\frac{\mu p\epsilon^2}{96}
        \implies
        N\ge\frac{96\nu^2}{\mu\epsilon^2}.
    \end{align*}
    Accordingly, we select $N_0=N=\left\lceil96\nu^2/(\mu\epsilon^2)\right\rceil$
    and $p=1/N$. For sufficiently small $\epsilon$, we have $96\nu^2/(\mu\epsilon^2)\ge 1$, so that
    \begin{align*}
        \frac{96\nu^2}{\mu\epsilon^2}
        \le N
        \le\frac{192\nu^2}{\mu\epsilon^2}.
    \end{align*}

    Next, we evaluate the stability constraints on the learning rate. For the variance stability requirement $\eta\le\sqrt{\lambda B/[24(1-\lambda)\Cdef]}$, using $\lambda=1/N$ and $1-\lambda\le 1$, it suffices to enforce:
    \begin{align*}
        \eta
        \le\sqrt{\frac{\mu B}{4608\nu^2\Cdef}}\epsilon
        \le\sqrt{\frac{B}{24N\Cdef}}.
    \end{align*}
    For the accuracy requirement $\eta\le\lambda^2\epsilon/(12C_\nu)$, substitution gives:
    \begin{align*}
        \frac{\lambda^2\epsilon}{12C_\nu}
        =\frac{pN\epsilon}{12\nu^2}
        =\frac{\epsilon}{12\nu^2}.
    \end{align*}
    Defining
    \begin{align*}
        c_\eta
        \coloneqq
        \min\left\{
            \frac{1}{12\nu^2},
            \sqrt{\frac{\mu B}{4608\nu^2\Cdef}}
        \right\},
    \end{align*}
    we therefore select $\eta=\min\{1/(4L),c_\eta\epsilon\}$. The sufficient iteration budget established in the proof of Theorem~\ref{thm:lvr_pg_convergence} is $K\ge\max\{32/(\mu\eta\epsilon),4/\lambda\} \log_{+}(2\Phi_0/\epsilon)$. Using $4/\lambda=4N\le768\nu^2/(\mu\epsilon^2)$ and substituting the full learning-rate choice, we select:
    \begin{align*}
        K=\left\lceil
            \max\left\{
                \frac{128L}{\mu\epsilon},
                \frac{768\nu^2}{\mu\epsilon^2},
                \frac{1536\sqrt{2}\,\nu\sqrt{\Cdef}}
                     {\mu^{3/2}\sqrt{B}\epsilon^2}
            \right\}
            \log_{+}\frac{2\Phi_0}{\epsilon}
        \right\rceil.
    \end{align*}

    Since an iteration draws $N$ trajectories with probability $p=1/N$ and $B$ trajectories otherwise, the expected per-iteration cost is $1+(1-1/N)B$. Consequently:
    \begin{align*}
        n=N_0+K\left(1+B-\frac{B}{N}\right).
    \end{align*}
    Substituting the selected values of $N_0$ and $K$,and using $1+B-B/N\le 1+B$, yields:
    \begin{align*}
        n=\widetilde{\mathcal{O}}\left(
            1+B
            +\frac{(1+B)L}{\mu\epsilon}
            +\frac{(1+B)\nu^2}{\mu\epsilon^2}
            +\frac{(1+B)\nu\sqrt{\Cdef}}
                   {\mu^{3/2}\sqrt{B}\epsilon^2}
        \right).
    \end{align*}
    For fixed $B$ and fixed positive problem-dependent constants, the terms proportional to $\epsilon^{-2}$ dominate as $\epsilon\to 0$. Substituting $\nu^2=\mathcal{O}(R_{\max}^2(1-\gamma)^{-3})$, $\Cdef=\mathcal{O}(R_{\max}^2(1-\gamma)^{-5})$, and $\mu=\chi^{-2}$ gives:
    \begin{align*}
        n=\widetilde{\mathcal{O}}\left(
            \frac{R_{\max}^2}{\epsilon^2}
            \left[
                \frac{\chi^2}{(1-\gamma)^3}
                +\frac{\chi^3}{(1-\gamma)^4}
            \right]
        \right),
    \end{align*}
    where $B$, the policy-regularity constants, and the defensive-mixture parameter are treated as fixed.
\end{proof}

\begin{coroll}[Sample Complexity of \defstormpg]
    \label{cor:def_storm_pg}
    Under Assumptions~\ref{ass:LS}--\ref{ass:WGD}, for a fixed mini-batch size $B\ge 1$ and sufficiently small target accuracy $\epsilon>0$, through an appropriate choice of the algorithmic parameters (momentum $\alpha$, initial batch size $N_0$, and learning rate $\eta$), \defstormpg with $p=0$ and defensive-mixture parameter $\xi=1/2$ guarantees $J^*-\mE[\perfK]\le\psi+\epsilon$ with a sample complexity of at most:
    \begin{align*}
        n
        &=N_0+KB \\
        &=\widetilde{\mathcal{O}}\left(
            \frac{R_{\max}^2}{\epsilon^2}
            \left[
                \frac{\chi^2}{(1-\gamma)^3}
                +\frac{\chi^3}{(1-\gamma)^4}
            \right]
        \right).
    \end{align*}
    Here, $\widetilde{\mathcal{O}}(\cdot)$ hides logarithmic factors, and $B$ and the policy-regularity constants are treated as fixed.
\end{coroll}

\begin{proof}
    Setting the refresh probability to $p=0$ specializes the generic variance recurrence used in the proof of Theorem~\ref{thm:lvr_pg_convergence} to the \defstormpg estimator. Under the defensive-mixture choice $\xi=1/2$, the decay rate $\lambda$ and the corrected noise floor $C_\nu$ simplify to:
    \begin{align*}
        \lambda=2\alpha-\alpha^2,
        \qquad
        C_\nu=\frac{4\alpha^2\nu^2}{B}.
    \end{align*}
    Since $\lambda\ge\alpha$ for $\alpha\in[0,1]$, a sufficient condition for $C_\nu\le\mu\lambda\epsilon^2/96$ is:
    \begin{align*}
        \frac{4\alpha^2\nu^2}{B}
        \le\frac{\mu\alpha\epsilon^2}{96}
        \implies
        \alpha\le\frac{\mu B\epsilon^2}{384\nu^2}.
    \end{align*}
    We therefore select $\alpha=\mu B\epsilon^2/(384\nu^2)$ and
    $N_0=\left\lceil96\nu^2/(\mu\epsilon^2)\right\rceil$. For sufficiently small $\epsilon$, this choice satisfies $\alpha\in(0,1)$. For the variance stability requirement $\eta\le\sqrt{\lambda B/[24(1-\lambda)\Cdef]}$, using $\lambda\ge\alpha$ and $1-\lambda\le 1$, we enforce:
    \begin{align*}
        \eta
        \le\sqrt{\frac{\alpha B}{24\Cdef}}
        =\sqrt{\frac{\mu B^2}{9216\nu^2\Cdef}}\epsilon.
    \end{align*}
    For the last-iterate accuracy requirement, $\lambda\ge\alpha$ gives:
    \begin{align*}
        \frac{\lambda^2\epsilon}{12C_\nu}
        \ge
        \frac{\alpha^2\epsilon}
             {12(4\alpha^2\nu^2/B)}
        =\frac{B\epsilon}{48\nu^2}.
    \end{align*}
    Consequently, defining
    \begin{align*}
        c_\eta
        \coloneqq
        \min\left\{
            \frac{B}{48\nu^2},
            \sqrt{\frac{\mu B^2}{9216\nu^2\Cdef}}
        \right\},
    \end{align*}
    we select $\eta=\min\{1/(4L),c_\eta\epsilon\}$. Using the sufficient iteration budget established in the proof of Theorem~\ref{thm:lvr_pg_convergence}, together with $4/\lambda\le4/\alpha =1536\nu^2/(\mu B\epsilon^2)$, we select:
    \begin{align*}
        K=\left\lceil
            \max\left\{
                \frac{128L}{\mu\epsilon},
                \frac{1536\nu^2}{\mu B\epsilon^2},
                \frac{3072\nu\sqrt{\Cdef}}
                     {\mu^{3/2}B\epsilon^2}
            \right\}
            \log_{+}\frac{2\Phi_0}{\epsilon}
        \right\rceil.
    \end{align*}
    Here, the terms proportional to $\epsilon^{-1}$ and $\epsilon^{-2}$ follow by substituting the full learning-rate choice into $32/(\mu\eta\epsilon)$. Because $p=0$, the algorithm draws a fixed mini-batch of size $B$ at every iteration. Consequently, the total sample complexity is:
    \begin{align*}
        n=N_0+KB.
    \end{align*}
    Substituting the selected values of $N_0$ and $K$ gives:
    \begin{align*}
        n=\widetilde{\mathcal{O}}\left(
            1+B
            +\frac{BL}{\mu\epsilon}
            +\frac{\nu^2}{\mu\epsilon^2}
            +\frac{\nu\sqrt{\Cdef}}
                   {\mu^{3/2}\epsilon^2}
        \right).
    \end{align*}
    For fixed $B$ and fixed positive problem-dependent constants, the terms proportional to $\epsilon^{-2}$ dominate as $\epsilon\to 0$.
    Substituting $\nu^2=\mathcal{O}(R_{\max}^2(1-\gamma)^{-3})$, $\Cdef=\mathcal{O}(R_{\max}^2(1-\gamma)^{-5})$, and $\mu=\chi^{-2}$ yields:
    \begin{align*}
        n=\widetilde{\mathcal{O}}\left(
            \frac{R_{\max}^2}{\epsilon^2}
            \left[
                \frac{\chi^2}{(1-\gamma)^3}
                +\frac{\chi^3}{(1-\gamma)^4}
            \right]
        \right),
    \end{align*}
    where the policy-regularity constants are treated as fixed.
\end{proof}

\section{Sample Complexity under Fisher-non-Degeneracy}
\label{apx:proof_fnd_corollary}

While Theorem~\ref{thm:lvr_pg_convergence} establishes the sample complexity bounds under the abstract weak gradient domination condition, we apply Assumptions~\ref{ass:epsilon_bias_new} and~\ref{ass:f_non_degenerate} to apply Lemma~\ref{lemma:FNDdomination}. Note that, in the FND setting, we consider the unconstrained parameter space $\Theta = \mR^{d}$.

\sampleFND*
\begin{proof}
    The result follows directly from Theorem~\ref{thm:lvr_pg_convergence} by applying the Fisher-induced Gradient Domination Lemma. Specifically, the lemma establishes that under Assumptions~\ref{ass:epsilon_bias_new} and~\ref{ass:f_non_degenerate}, the Weak Gradient Domination condition holds with the constants $\chi = C_1 / \mu_F$ and $\psi = \sqrt{\epsilon_{\text{bias}}}\kappa/(1-\gamma)$. Substituting these specific values into the optimality gap and the general sample complexity bound of Theorem~\ref{thm:lvr_pg_convergence} immediately yields the stated result.
\end{proof}

\section{Lower-Bound}
\label{sec:lower}

\lowerBound*

Before starting with the proof of Theorem~\ref{thr:lower_bound}, we describe it and its difference from previous works.  The proof relies on the information-theoretic framework for minimax lower bounds and uses Le Cam's two-point method. While the hard instance construction shares topological similarities with the absorbing MDPs introduced by~\cite{azar2013minimax}, we accommodate the specific geometry of policy optimization. While~\cite{azar2013minimax} establishes lower bounds for unconstrained tabular reinforcement learning, in our proof, we consider a constrained continuous policy parametrization, $\theta \in [\alpha, 1-\alpha]$. This domain restriction is the critical mechanism that enforces Assumption~\ref{ass:LS}. Furthermore, we demonstrate that this constrained geometry satisfies Assumption~\ref{ass:WGD} with $\chi = 4/3$ and $\psi = 0$, avoiding vanishing gradients at the optimal boundaries.

To establish a fundamental lower bound on the sample complexity of \emph{any} learning algorithm, we must construct an adversarial scenario in which making the correct decision is statistically difficult, yet an incorrect guess incurs a penalty. We achieve this by designing two nearly identical MDPs, $\mathcal{M}_1$ and $\mathcal{M}_2$. These instances differ strictly by a small perturbation $\delta$ in the transition probabilities, making them difficult to distinguish, yet they require mutually exclusive optimal policies. We first quantify the sub-optimality gap to show that if an algorithm $\mathcal{A}$ fails to identify the true underlying MDP and selects a parameter favouring the wrong action, it suffers a performance penalty of at least $\Omega(\delta(1-\gamma)^{-1})$. By tuning this "wrong-action" penalty to perfectly match our target accuracy threshold $\epsilon$, we isolate the required perturbation scale to $\delta = \mathcal{O}(\epsilon(1-\gamma)^2)$.

Because a valid $\mathcal{A}$ must guarantee $\epsilon$-optimality across \emph{all} instances, it is forced to determine whether it is interacting with $\mathcal{M}_1$ or $\mathcal{M}_2$. If it collects too few samples, the trajectory distributions of the two MDPs remain virtually identical, and the algorithm will inevitably guess wrong with high probability, thereby suffering the tuned $\epsilon$ penalty. We formalize this through Le Cam's method, which implies that, to reliably distinguish the two instances, the total Kullback-Leibler divergence between their $n$-trajectory distributions must exceed a given threshold. 
    
Finally, we establish how much data is required to cross this threshold by analytically upper-bounding the KL divergence. We show that the divergence accumulates to at most $\mathcal{O}(n\delta^2(1-\gamma)^{-2})$. Substituting our chosen $\delta$ into this upper bound and solving for the sample size $n$ reveals that the minimum amount of data required to distinguish the instances, and thus avoid the sub-optimality penalty, is exactly $\Omega((1-\gamma)^{-2}\rmax^2\epsilon^{-2})$.

\begin{proof}
    %\leavevmode
    The proof is subdivided into several parts.
    
    \paragraph{Part I: Construction of the Hard Instance.} We consider a class of MDPs consisting of two instances, $\mathcal{M} \in \{\mathcal{M}_1, \mathcal{M}_2\}$, with state space $\mathcal{S} = \{s, \bar{s}\}$ and action space $\mathcal{A} = \{a_0, a_1\}$. The state $\bar{s}$ is an absorbing state where any action $a \in \mathcal{A}$ yields a deterministic self-transition with zero reward, meaning $P(\bar{s} | \bar{s}, a) = 1$ and $R(\bar{s}, a, \bar{s}) = 0$. In the initial state $s$, taking an action $a$ triggers a self-transition to $s$ with probability $p_a$, yielding a reward of $\rmax$, and transitions to the absorbing state $\bar{s}$ with probability $1-p_a$, yielding $\rmax$. The two instances differ strictly in the transition probabilities from the initial state $s$. For the instance $\mathcal{M}_1$, we set $p_{a_0} = p + \delta$ and $p_{a_1} = p$, making $a_0$ the optimal action. Conversely, for the instance $\mathcal{M}_2$, we set $p_{a_0} = p$ and $p_{a_1} = p + \delta$, making $a_1$ the optimal action. To ensure the transition probabilities remain valid while capturing the lower-bound complexity dependence on the effective horizon, we define the base transition probability as $p = 1 - \frac{1-\gamma}{\gamma}$. Consequently, to satisfy the requirement that $p + \delta \leq 1$, the perturbation parameter $\delta$ is strictly bounded by $\delta \leq \frac{1-\gamma}{2\gamma}$.

    \begin{figure}[htbp]
        \centering
        \begin{minipage}{0.48\textwidth}
            \centering
            \begin{tikzpicture}[->, >=stealth, node distance=4cm, auto, thick]
                \node[circle, draw, minimum size=1cm] (S) {$s$};
                \node[circle, draw, minimum size=1cm] (SA) [right of=S] {$\bar{s}$};
                
                \path
                (S) edge [loop above, dashed] node {$a_0 : p + \delta, \rmax$} (S)
                (S) edge [loop below] node {$a_1 : p, \rmax$} (S)
                (S) edge [bend left=20, dashed] node {$a_0 : 1 - (p + \delta), \rmax$} (SA)
                (S) edge [bend right=20] node[below] {$a_1 : 1 - p, \rmax$} (SA)
                (SA) edge [loop right] node {$\forall a : 1, 0$} (SA);
            \end{tikzpicture}
            \caption{Transition dynamics for instance $\mathcal{M}_1$.}
            \label{fig:m1_dynamics}
        \end{minipage}\hfill
        \begin{minipage}{0.48\textwidth}
            \centering
            \begin{tikzpicture}[->, >=stealth, node distance=4cm, auto, thick]
                \node[circle, draw, minimum size=1cm] (S) {$s$};
                \node[circle, draw, minimum size=1cm] (SA) [right of=S] {$\bar{s}$};
                
                \path
                (S) edge [loop above] node {$a_0 : p, \rmax$} (S)
                (S) edge [loop below, dashed] node {$a_1 : p + \delta, \rmax$} (S)
                (S) edge [bend left=20] node {$a_0 : 1 - p, \rmax$} (SA)
                (S) edge [bend right=20, dashed] node[below] {$a_1 : 1 - (p + \delta), \rmax$} (SA)
                (SA) edge [loop right] node {$\forall a : 1, 0$} (SA);
            \end{tikzpicture}
            \caption{Transition dynamics for instance $\mathcal{M}_2$.}
            \label{fig:m2_dynamics}
        \end{minipage}
    \end{figure}
    
    We consider a direct policy parametrization in which the scalar parameter $\theta$ determines the probability of selecting the action $a_0$ in the initial state $s$. For the binary action space $\mathcal{A} = \{a_0, a_1\}$, the policy is defined as:
    \begin{align*}
        \pi_{\theta}(a \mid s) &= 
        \begin{cases} 
        \theta & \text{if } a = a_0, \\
        1 - \theta & \text{otherwise}.
        \end{cases}
    \end{align*}
    
    \paragraph{Part II: Verification of the Assumptions.} Here, we check that the policy class we have choosen satisfies Assumptions~\ref{ass:LS} and~\ref{ass:WGD}. To ensure the policy space satisfies the regularity constraints of Assumption~\ref{ass:LS}, we restrict the parameter domain to $\Theta = [\alpha, 1-\alpha]$ for a strictly positive constant $\alpha \in (0, 1/2)$. The score function is given by:
    \begin{align*}
        \nabla_{\theta} \log \pi_{\theta}(a \mid s) &=
        \begin{cases} 
            \frac{1}{\theta} & \text{if } a = a_0, \\
            -\frac{1}{1-\theta} & \text{otherwhise}. 
        \end{cases}
    \end{align*}
    
    Condition~\textit{(i)} requires the score norm to be bounded. Since $\theta \in [\alpha, 1-\alpha]$, the probability of any action is bounded strictly away from zero, uniformly bounding the score. We establish the constant $C_1$ as the supremum of the score function over the parameter space and action space:
    \begin{align*}
        \sup_{a \in \mathcal{A}} \big\| \nabla_{\theta} \log \pi_{\theta}(a \mid s) \big\|_2 &= \max_{\theta \in [\alpha, 1-\alpha]} \max \left\{ \frac{1}{\theta}, \frac{1}{1-\theta} \right\} = \frac{1}{\alpha}.
    \end{align*}
    Thus, condition \textit{(i)} is satisfied with $C_1 \leq 1/\alpha$. To verify condition~\textit{(ii)}, we must establish that the score function is Lipschitz continuous in expectation. We first calculate the Hessian of the log-policy:
    \begin{align*}
        \nabla^2_{\theta} \log \pi_{\theta}(a \mid s) &=
        \begin{cases}
        -\frac{1}{\theta^2} & \text{if } a = a_0, \\
        -\frac{1}{(1-\theta)^2} & \text{otherwise}.
        \end{cases}
    \end{align*}
    The restricted domain $\Theta$ guarantees that the Hessian norm is uniformly bounded by $1/\alpha^2$. By the Mean Value Theorem, the squared difference in the score function for any two parameters $\theta, \theta' \in \Theta$ is bounded by the squared supremum of the Hessian. For any action $a \in \mathcal{A}$:
    \begin{align*}
        \big\| \nabla_{\theta} \log \pi_{\theta}(a \mid s) - \nabla_{\theta} \log \pi_{\theta'}(a \mid s) \big\|_2^2 &\leq \left( \sup_{\tilde{\theta} \in \Theta} \big\| \nabla^2_{\theta} \log \pi_{\tilde{\theta}}(a \mid s) \big\|_2 \right)^2 \|\theta - \theta'\|_2^2 \\
        &\leq \left( \frac{1}{\alpha^2} \right)^2 \|\theta - \theta'\|_2^2 \\
        &= \frac{1}{\alpha^4} \|\theta - \theta'\|_2^2.
    \end{align*}
    
    Because this upper bound holds deterministically for all actions $a \in \mathcal{A}$, taking the expectation over $a \sim \pi_{\theta}(\cdot \mid s)$ preserves the inequality. Thus, condition~\textit{(ii)} is satisfied with the constant $L_S = 1/\alpha^2$:
    \begin{align*}
        \mathop{\mathbb{E}}_{a \sim \pi_{\theta}} \left[ \big\| \nabla_{\theta} \log \pi_{\theta}(a \mid s) - \nabla_{\theta} \log \pi_{\theta'}(a \mid s) \big\|_2^2 \right] &\leq \frac{1}{\alpha^4} \|\theta - \theta'\|_2^2.
    \end{align*}

    We now verify the validity of Assumption~\ref{ass:WGD} with $\psi = 0$ in our setting. We start computing the expected return $J^{(i)}(\theta)$, its optimal value $J^{*, (i)}$ within the restricted domain $\Theta = [\alpha, 1-\alpha]$, and the gradient norm $\| \nabla_{\theta} J^{(i)}(\theta) \|_2$ for both instances $i \in \{1, 2\}$. By assigning a reward of $\rmax$ to any action taken in state $s$ prior to reaching the zero-reward absorbing state $\bar{s}$, the expected immediate reward at time $t$, given $s_t = s$, is exactly $\rmax$. The expected discounted return resolves to a standard geometric series:
    \begin{align*}
        J(\theta) &= \sum_{t=0}^{\infty} \gamma^t P(s_t = s \mid s_0 = s) \rmax = \rmax \sum_{t=0}^{\infty} (\gamma p_{\theta})^t = \frac{\rmax}{1 - \gamma p_{\theta}}.
    \end{align*}
    For the specific instances $\mathcal{M}_1$ and $\mathcal{M}_2$, the effective self-transition probabilities under the policy $\pi_{\theta}$ are:
    \begin{align*}
        p_{\theta}^{(1)} &= \theta(p + \delta) + (1-\theta)p = p + \theta\delta, \\
        p_{\theta}^{(2)} &= \theta p + (1-\theta)(p + \delta) = p + (1-\theta)\delta.
    \end{align*}
    Because the expected return monotonically increases with the self-transition probability, the optimal returns $J^{*, (i)}$ over the constrained domain $\Theta = [\alpha, 1-\alpha]$ are achieved by maximizing $p_{\theta}^{(i)}$. For $\mathcal{M}_1$, this occurs at $\theta = 1-\alpha$, and for $\mathcal{M}_2$, this occurs at $\theta = \alpha$. In both cases, the maximum self-transition probability evaluates to $p + (1-\alpha)\delta$. Differentiating the objective function with respect to $\theta$ cleanly isolates the gradient of the transition probability:
    \begin{align*}
        \nabla_{\theta} J^{(i)}(\theta) &= \frac{\gamma \rmax}{(1 - \gamma p_{\theta}^{(i)})^2} \nabla_{\theta} p_{\theta}^{(i)}.
    \end{align*}

\begin{table}[t]
    \centering
    \renewcommand{\arraystretch}{1.5}
    \caption{Values for $J^{(i)}(\theta)$, $J^{*, (i)}$, and
    $\big\|\nabla_{\theta}J^{(i)}(\theta)\big\|_2$
    for the $\mathcal{M}$ instances.}
    \label{tab:lower_bound}
        \begin{tabular}{@{}lccc@{}}
            \toprule
            \textbf{Instance}
            & $J^{(i)}(\theta)$
            & $J^{*, (i)}$
            & $\big\|\nabla_{\theta}J^{(i)}(\theta)\big\|_2$ \\
            \midrule
            $\mathcal{M}_1$
            & $\frac{\rmax}{1-\gamma(p+\theta\delta)}$
            & $\frac{\rmax}{1-\gamma(p+(1-\alpha)\delta)}$
            & $\frac{\gamma\delta\rmax}
                    {(1-\gamma(p+\theta\delta))^2}$ \\
            $\mathcal{M}_2$
            & $\frac{\rmax}{1-\gamma(p+(1-\theta)\delta)}$
            & $\frac{\rmax}{1-\gamma(p+(1-\alpha)\delta)}$
            & $\frac{\gamma\delta\rmax}
                    {(1-\gamma(p+(1-\theta)\delta))^2}$ \\
            \bottomrule
        \end{tabular}%
\end{table}
    Applying this to both instances yields the values defined in Table~\ref{tab:lower_bound}. We establish the constant $\chi$ by first evaluating the ratio with respect to the unconstrained gradient, $\frac{J^{*,(i)}-J^{(i)}(\theta)}{\|\nabla_{\theta}J^{(i)}(\theta)\|_2} \le \chi$. As we will formalize in Remark~\ref{rem:boundary_wgd}, bounding this unconstrained ratio is sufficient to satisfy Assumption~\ref{ass:WGD} for the generalized mapping $\mathcal{G}_\eta$. Starting with $\mathcal{M}_1$, the sub-optimality gap reduces to:
    \begin{align*}
        J^{*, (1)} - J^{(1)}(\theta) &= \frac{\rmax}{1 - \gamma (p + (1-\alpha)\delta)} - \frac{\rmax}{1 - \gamma (p + \theta\delta)} \\
        &= \frac{\gamma \rmax ((p + (1-\alpha)\delta) - (p + \theta\delta))}{(1 - \gamma (p + (1-\alpha)\delta))(1 - \gamma (p + \theta\delta))} \\
        &= \frac{\gamma \rmax (1 - \alpha - \theta)\delta}{(1 - \gamma (p + (1-\alpha)\delta))(1 - \gamma (p + \theta\delta))}.
    \end{align*}

    Dividing this gap by the gradient norm causes both $\gamma$ and $\rmax$ to cancel out:
    \begin{align*}
        \frac{J^{*, (1)} - J^{(1)}(\theta)}{\big\| \nabla_{\theta} J^{(1)}(\theta) \big\|_2} &= \frac{\gamma \rmax (1 - \alpha - \theta)\delta}{(1 - \gamma (p + (1-\alpha)\delta))(1 - \gamma (p + \theta\delta))} \cdot \frac{(1 - \gamma (p + \theta\delta))^2}{\gamma \delta \rmax} \\
        &= (1 - \alpha - \theta) \frac{1 - \gamma (p + \theta\delta)}{1 - \gamma (p + (1-\alpha)\delta)}.
    \end{align*}

    Applying the definitions $p = \frac{2\gamma - 1}{\gamma}$ (requiring $\gamma > 1/2$) and $\delta \leq \frac{1-\gamma}{2\gamma}$, the baseline effective horizon evaluates to $1 - \gamma p = 2(1-\gamma)$. We bound the denominator from below:
    \begin{align*}
        1 - \gamma (p + (1-\alpha)\delta) &\geq 1 - \gamma p - \gamma \delta \geq 2(1-\gamma) - \frac{1-\gamma}{2} = \frac{3}{2}(1-\gamma).
    \end{align*}
    Leveraging $\theta\delta > 0$, the numerator is bounded from above:
    \begin{align*}
        1 - \gamma (p + \theta\delta) &\leq 1 - \gamma p = 2(1-\gamma).
    \end{align*}
    Because the domain restriction ensures $(1 - \alpha - \theta) \leq 1$, we obtain the uniform bound for $\mathcal{M}_1$:
    \begin{align*}
        \frac{J^{*, (1)} - J^{(1)}(\theta)}{\big\| \nabla_{\theta} J^{(1)}(\theta) \big\|_2} &\leq 1 \cdot \frac{2(1-\gamma)}{\frac{3}{2}(1-\gamma)} = \frac{4}{3}.
    \end{align*}

    We repeat the exact procedure for $\mathcal{M}_2$. The optimal parameter for $\mathcal{M}_2$ is $\theta^* = \alpha$. The sub-optimality gap is:
    \begin{align*}
        J^{*, (2)} - J^{(2)}(\theta) &= \frac{\rmax}{1 - \gamma (p + (1-\alpha)\delta)} - \frac{\rmax}{1 - \gamma (p + (1-\theta)\delta)} \\
        &= \frac{\gamma \rmax ((p + (1-\alpha)\delta) - (p + (1-\theta)\delta))}{(1 - \gamma (p + (1-\alpha)\delta))(1 - \gamma (p + (1-\theta)\delta))} \\
        &= \frac{\gamma \rmax (\theta - \alpha)\delta}{(1 - \gamma (p + (1-\alpha)\delta))(1 - \gamma (p + (1-\theta)\delta))}.
    \end{align*}
    Dividing by the gradient norm for $\mathcal{M}_2$ yields:
    \begin{align*}
        \frac{J^{*, (2)} - J^{(2)}(\theta)}{\big\| \nabla_{\theta} J^{(2)}(\theta) \big\|_2} &= (\theta - \alpha) \frac{1 - \gamma (p + (1-\theta)\delta)}{1 - \gamma (p + (1-\alpha)\delta)}.
    \end{align*}
    The denominator bound remains identical to $\mathcal{M}_1$. For the numerator, since $(1-\theta)\delta > 0$:
    \begin{align*}
        1 - \gamma (p + (1-\theta)\delta) &\leq 1 - \gamma p = 2(1-\gamma).
    \end{align*}
    The parameter restriction $\theta \in [\alpha, 1-\alpha]$ enforces $(\theta - \alpha) \leq 1 - 2\alpha < 1$. Substituting these bounds yields the limit for $\mathcal{M}_2$:
    \begin{align*}
        \frac{J^{*, (2)} - J^{(2)}(\theta)}{\big\| \nabla_{\theta} J^{(2)}(\theta) \big\|_2} &\leq (1 - 2\alpha) \cdot \frac{2(1-\gamma)}{\frac{3}{2}(1-\gamma)} < \frac{4}{3}.
    \end{align*}
    The preceding computations establish
    \begin{align*}
        J^{*,(i)} - J^{(i)}(\theta)
        \le \frac{4}{3}
        \|\nabla_\theta J^{(i)}(\theta)\|_2,
        \qquad i \in \{1,2\}.
    \end{align*}
    We now verify the corresponding inequality for the
    generalized gradient mapping.
    
    \begin{remark}[Boundary Conditions and the Gradient Mapping]
        \label{rem:boundary_wgd}
        Because the policy space is restricted to $\Theta=[\alpha,1-\alpha]$, we must verify that Assumption~\ref{ass:WGD} also holds when projection clips the
        update. The optimal parameters are $\theta^{*,(1)}=1-\alpha$ and $\theta^{*,(2)}=\alpha$. For both instances, the gradient points toward the corresponding optimum, so projection gives
        \begin{align*}
            \|G_\eta^{(i)}(\theta)\|_2
            =
            \min\left\{
                \|\nabla_\theta J^{(i)}(\theta)\|_2,
                \frac{|\theta^{*,(i)}-\theta|}{\eta}
            \right\}.
        \end{align*}
        If the update is not clipped, the generalized mapping equals the gradient, and the previously established inequality $J^{*,(i)}-J^{(i)}(\theta) \le \frac{4}{3}\|\nabla_\theta J^{(i)}(\theta)\|_2$ applies directly. Otherwise, the projected update reaches $\theta^{*,(i)}$, yielding $|\theta^{*,(i)}-\theta|
        =\eta\|G_\eta^{(i)}(\theta)\|_2$. By the mean value theorem,
        \begin{align*}
            J^{*,(i)}-J^{(i)}(\theta)
            &\le
            \sup_{\theta'\in\Theta}
            \|\nabla_\theta J^{(i)}(\theta')\|_2
            |\theta^{*,(i)}-\theta| \\
            &=
            \frac{\eta\gamma\delta R_{\max}}
                 {[1-\gamma(p+(1-\alpha)\delta)]^2}
            \|G_\eta^{(i)}(\theta)\|_2.
        \end{align*}
        The regularity constants established above, $C_1=1/\alpha$ and $L_S=1/\alpha^2$, give $L=2R_{\max}/[\alpha^2(1-\gamma)^2$]. Using $\eta\le 1/L$, $1-\gamma(p+(1-\alpha)\delta)\ge \frac{3}{2}(1-\gamma)$, and $\gamma\delta\le(1-\gamma)/2$, the coefficient in the last display is at most $\alpha^2(1-\gamma)/9\le 4/3$.
        Thus, $J^{*,(i)}-J^{(i)}(\theta) \le \frac{4}{3}\|G_\eta^{(i)}(\theta)\|_2$ also holds when projection clips the update. At the optimum, both sides vanish.
        Consequently, Assumption~\ref{ass:WGD} holds throughout $\Theta$, for every $0<\eta\le 1/L$, with $\psi=0$ and $\chi=4/3$.
    \end{remark}

    \paragraph{Part III: Value Gap and Choice of Perturbation.} 
    To establish the statistical hardness, we must ensure that failing to identify the correct MDP results in a substantial performance penalty. Specifically, we require that if an algorithm $\mathcal{A}$ outputs a policy that favors the wrong action, the sub-optimality gap is at least $4\epsilon$. 

    In instance $\mathcal{M}_1$, the optimal parameter is $\theta^* = 1-\alpha > 1/2$. The algorithm outputs a highly sub-optimal policy if it incorrectly favors action $a_1$, which occurs when $\theta \leq 1/2$. We evaluate the sub-optimality gap for $\mathcal{M}_1$ given $\theta \leq 1/2$ to find a strict lower bound. From the previous derivation, the gap is:
    \begin{align*}
        J^{*, (1)} - J^{(1)}(\theta) &= \frac{\gamma \rmax (1 - \alpha - \theta)\delta}{(1 - \gamma (p + (1-\alpha)\delta))(1 - \gamma (p + \theta\delta))}.
    \end{align*}

    To lower-bound this gap, we lower-bound the numerator and upper-bound the denominator. Because $\theta \leq 1/2$, the term $(1 - \alpha - \theta)$ is bounded below by $(1/2 - \alpha)$. Furthermore, since we established the condition $\gamma \geq 1/2$ to ensure valid transition probabilities, we substitute $\gamma = 1/2$ to strictly lower-bound the numerator and remove its dependence on $\gamma$:
    \begin{align*}
        \gamma \rmax (1 - \alpha - \theta)\delta &\geq \frac{1}{2} \rmax \left(\frac{1}{2} - \alpha\right)\delta = \frac{\rmax (1 - 2\alpha)\delta}{4}.
    \end{align*}

    For the denominator, since $\theta\delta \geq 0$ and $(1-\alpha)\delta \geq 0$, the effective horizon terms are strictly upper-bounded by the baseline effective horizon $1 - \gamma p = 2(1-\gamma)$:
    \begin{align*}
        (1 - \gamma (p + (1-\alpha)\delta))(1 - \gamma (p + \theta\delta)) &\leq (2(1-\gamma))(2(1-\gamma)) = 4(1-\gamma)^2.
    \end{align*}

    Combining these bounds, the sub-optimality gap for any parameter $\theta \leq 1/2$ is bounded from below:
    \begin{align*}
        J^{*, (1)} - J^{(1)}(\theta) &\geq \frac{\frac{\rmax (1 - 2\alpha)\delta}{4}}{4(1-\gamma)^2} = \frac{\rmax (1 - 2\alpha)\delta}{16(1-\gamma)^2}.
    \end{align*}

    To enforce that this misidentification penalty is at least $4\epsilon$, we equate this lower bound to $4\epsilon$ and solve for the perturbation parameter $\delta$:
    \begin{align*}
        \frac{\rmax (1 - 2\alpha)\delta}{16(1-\gamma)^2} &= 4\epsilon \implies \delta = \frac{64\epsilon(1-\gamma)^2}{\rmax(1-2\alpha)}.
    \end{align*}

    This specific choice of $\delta$ enforces the required value gap while controlling the statistical indistinguishability of the instances. To satisfy the transition probability constraint $\delta \leq \frac{1-\gamma}{2\gamma}$, the target accuracy $\epsilon$ must be sufficiently small:
    \begin{align*}
        \frac{64\epsilon(1-\gamma)^2}{\rmax(1-2\alpha)} \leq \frac{1-\gamma}{2\gamma} \implies \epsilon \leq \frac{\rmax(1-2\alpha)}{128\gamma(1-\gamma)}.
    \end{align*}

    \textbf{Part IV: Statistical Test and Information-Theoretic Lower Bound.}
        We formulate the learning process as a statistical testing problem. Suppose an algorithm $\mathcal{A}$ interacts with an unknown instance $\mathcal{M}_i$, $i\in\{1,2\}$, collects a possibly random number of trajectories $N_{\mathcal{A}}$, and subsequently outputs a parameter estimate $\widehat{\theta}\in\Theta$. The algorithm is not informed whether the underlying environment is $\mathcal{M}_1$ or $\mathcal{M}_2$, and its policies may depend on all previously observed trajectories. For $\mathcal{A}$ to be considered a successful $\epsilon$-optimal algorithm on these instances, its expected suboptimality gap must satisfy
        \begin{align*}
            \mathbb{E}_{\mathcal{A},\mathcal{M}_i}
            \left[
                J^{*,(i)}-J^{(i)}(\widehat{\theta})
            \right]
            \leq\epsilon,
            \qquad i\in\{1,2\},
        \end{align*}
        where $J^{*,(i)}=\max_{\theta\in\Theta}J^{(i)}(\theta)$. We define the worst-case expected trajectory budget as
        \begin{align*}
            n:=
            \max_{i\in\{1,2\}}
            \mathbb{E}_{\mathcal{A},\mathcal{M}_i}
            [N_{\mathcal{A}}].
        \end{align*}
        We may assume that $n<\infty$, since otherwise the claimed lower bound holds trivially. Using the suboptimality bound derived in Part III, any output $\widehat{\theta}<1/2$ under $\mathcal{M}_1$ incurs a gap of at least $4\epsilon$. Therefore,
        \begin{align*}
            \epsilon
            &\geq
            \mathbb{E}_{\mathcal{A},\mathcal{M}_1}
            \left[
                J^{*,(1)}-J^{(1)}(\widehat{\theta})
            \right]\\
            &\geq
            4\epsilon\,
            \mathbb{P}_{\mathcal{A},\mathcal{M}_1}
            \left(\widehat{\theta}<\frac12\right),
        \end{align*}
        which implies:
        \begin{align*}
            \mathbb{P}_{\mathcal{A},\mathcal{M}_1}
            \left(\widehat{\theta}<\frac12\right)
            \leq\frac14.
        \end{align*}
        By symmetric reasoning under $\mathcal{M}_2$, any output $\widehat{\theta}\geq1/2$ incurs a gap of at least $4\epsilon$, yielding,
        \begin{align*}
            \mathbb{P}_{\mathcal{A},\mathcal{M}_2}
            \left(\widehat{\theta}\geq\frac12\right)
            \leq\frac14.
        \end{align*}
        Summing these error probabilities gives:
        \begin{align*}
            \mathbb{P}_{\mathcal{A},\mathcal{M}_1}
            \left(\widehat{\theta}<\frac12\right)
            +
            \mathbb{P}_{\mathcal{A},\mathcal{M}_2}
            \left(\widehat{\theta}\geq\frac12\right)
            \leq\frac12.
        \end{align*}
        
        Let $\mathbb{P}_{\mathcal{A},\mathcal{M}_1}$ and $\mathbb{P}_{\mathcal{A},\mathcal{M}_2}$ denote the probability measures over the complete stopped interaction transcript under the two instances. This transcript includes the algorithm's internal randomization, the observed trajectories, and its final output $\widehat{\theta}$. Interpreting $\widehat{\theta}\geq1/2$ as selecting $\mathcal{M}_1$, Le Cam's inequality gives:
        \begin{align*}
            &\mathbb{P}_{\mathcal{A},\mathcal{M}_1}
            \left(\widehat{\theta}<\frac12\right)
            +
            \mathbb{P}_{\mathcal{A},\mathcal{M}_2}
            \left(\widehat{\theta}\geq\frac12\right)\\
            &\qquad\geq
            1-d_{\mathrm{TV}}\!\left(
                \mathbb{P}_{\mathcal{A},\mathcal{M}_1},
                \mathbb{P}_{\mathcal{A},\mathcal{M}_2}
            \right).
        \end{align*}
        Consequently, the total variation distance must be at least $1/2$. Applying Pinsker's inequality, we obtain,
        \begin{align*}
            \frac12
            &\leq
            d_{\mathrm{TV}}\!\left(
                \mathbb{P}_{\mathcal{A},\mathcal{M}_1},
                \mathbb{P}_{\mathcal{A},\mathcal{M}_2}
            \right)\\
            &\leq
            \sqrt{
                \frac12
                D_{\mathrm{KL}}\!\left(
                    \mathbb{P}_{\mathcal{A},\mathcal{M}_1}
                    \,\middle\|\,
                    \mathbb{P}_{\mathcal{A},\mathcal{M}_2}
                \right)
            },
        \end{align*}
        and hence:
        \begin{align*}
            D_{\mathrm{KL}}\!\left(
                \mathbb{P}_{\mathcal{A},\mathcal{M}_1}
                \,\middle\|\,
                \mathbb{P}_{\mathcal{A},\mathcal{M}_2}
            \right)
            \geq\frac12.
        \end{align*}
        To translate this information requirement into a trajectory complexity lower bound, we next upper-bound the divergence in terms of the expected number of collected trajectories.

    \textbf{Part V: Upper Bound on the KL Divergence.}
        We upper-bound the KL divergence of the complete interaction transcript through the conditional chain rule for relative entropy. Conditional on the same realized history, the algorithm uses the same action-selection and stopping rules
        under both instances. Moreover, the initial state, reward function, and independent geometric termination mechanism are identical under $\mathcal{M}_1$ and $\mathcal{M}_2$. Consequently, only the environment transition kernels
        contribute to the conditional KL divergence. We first bound the divergence contributed by a single transition. Since $\bar{s}$ is absorbing and possesses
        identical transition dynamics in both instances,
        \begin{align*}
            D_{\mathrm{KL}}\!\left(
                P^{(1)}(\cdot\mid\bar{s},a)
                \,\middle\|\,
                P^{(2)}(\cdot\mid\bar{s},a)
            \right)=0.
        \end{align*}
        At the active state $s$, the Bernoulli parameters under $(\mathcal{M}_1,\mathcal{M}_2)$ are $(p+\delta,p)$ for $a_0$ and $(p,p+\delta)$ for $a_1$. Using the standard
        inequality $d_{\mathrm{KL}}(x\|y)\leq(x-y)^2/[y(1-y)]$, we obtain,
        \begin{align*}
            d_{\mathrm{KL}}(p+\delta\|p)
            &\leq\frac{\delta^2}{p(1-p)},\\
            d_{\mathrm{KL}}(p\|p+\delta)
            &\leq\frac{\delta^2}
            {(p+\delta)(1-p-\delta)}.
        \end{align*}
        Recall that $p=(2\gamma-1)/\gamma$, $\gamma\geq2/3$, and
        $\delta\leq(1-\gamma)/(2\gamma)$. These conditions imply:
        \begin{align*}
            p\geq\frac12,
            \qquad
            1-p=\frac{1-\gamma}{\gamma},
            \qquad
            1-p-\delta\geq\frac{1-\gamma}{2\gamma}.
        \end{align*}
        Thus, both variance terms are bounded below by $(1-\gamma)/(4\gamma)$, yielding,
        \begin{align*}
            \max_{a\in\{a_0,a_1\}}
            D_{\mathrm{KL}}\!\left(
                P^{(1)}(\cdot\mid s,a)
                \,\middle\|\,
                P^{(2)}(\cdot\mid s,a)
            \right)
            \leq\frac{4\gamma\delta^2}{1-\gamma}
            \leq\frac{4\delta^2}{1-\gamma}.
        \end{align*}
        This bound holds uniformly over the selected action. Next, we bound the number of informative transitions within each trajectory. Let $H_j$ denote the number of observed transitions in trajectory $j$, and define:
        \begin{align*}
            T_j:=
            \sum_{t=0}^{H_j-1}
            \mathbf{1}_{\{s_{j,t}=s\}},
        \end{align*}
        where $s_{j,t}$ and $a_{j,t}$ denote the state and action at time $t$ of trajectory $j$. Thus, $T_j$ counts the observed transitions from the active state. Each trajectory starts at $s$ and follows the independent geometric termination protocol. We set $T_j=0$ when trajectory $j$ is not collected. Let $\mathcal{F}_{j-1}$ denote the history immediately before trajectory $j$, including the algorithm's decision whether to collect it. In particular, $\{N_{\mathcal{A}}\geq j\}$ is $\mathcal{F}_{j-1}$-measurable. On this event, the conditional probability of remaining in $s$ after any active transition is at most $p+\delta$,
        regardless of the selected action. The geometric termination mechanism can only reduce the number of observed active transitions. Consequently,
        \begin{align*}
            \mathbb{P}_{\mathcal{A},\mathcal{M}_1}
            \left(T_j>t\mid\mathcal{F}_{j-1}\right)
            &\leq(p+\delta)^t,
            \qquad t\geq0,\\
            \mathbb{E}_{\mathcal{A},\mathcal{M}_1}
            \left[T_j\mid\mathcal{F}_{j-1}\right]
            &\leq\sum_{t=0}^{\infty}(p+\delta)^t\\
            &=\frac{1}{1-p-\delta}
            \leq\frac{2\gamma}{1-\gamma}
            \leq\frac{2}{1-\gamma}.
        \end{align*}
        The conditional expectation bound also holds when the trajectory is not collected, since then $T_j=0$. For a history at which trajectory $j$ is collected, let
        $Q_{i,j}(\cdot\mid\mathcal{F}_{j-1})$ denote its conditional law under $\mathcal{M}_i$, and define:
        \begin{align*}
            K_j:=
            D_{\mathrm{KL}}\!\left(
                Q_{1,j}(\cdot\mid\mathcal{F}_{j-1})
                \,\middle\|\,
                Q_{2,j}(\cdot\mid\mathcal{F}_{j-1})
            \right).
        \end{align*}
        We set $K_j=0$ at histories where trajectory $j$ is not collected. Applying the conditional chain rule within a collected trajectory gives:
        \begin{align*}
            K_j
            &=
            \mathbb{E}_{\mathcal{A},\mathcal{M}_1}
            \left[
                \sum_{t=0}^{H_j-1}
                D_{\mathrm{KL}}\!\left(
                    P^{(1)}(\cdot\mid s_{j,t},a_{j,t})
                    \,\middle\|\,
                    P^{(2)}(\cdot\mid s_{j,t},a_{j,t})
                \right)
                \,\middle|\,
                \mathcal{F}_{j-1}
            \right]\\
            &\leq
            \frac{4\delta^2}{1-\gamma}
            \mathbb{E}_{\mathcal{A},\mathcal{M}_1}
            \left[T_j\mid\mathcal{F}_{j-1}\right]\\
            &\leq\frac{8\delta^2}{(1-\gamma)^2}.
        \end{align*}
        Therefore, the same upper bound applies to every collected trajectory, uniformly over the preceding history and the algorithm's policy choices. Finally, applying the conditional chain rule across trajectories yields:
        \begin{align*}
            &D_{\mathrm{KL}}\!\left(
                \mathbb{P}_{\mathcal{A},\mathcal{M}_1}
                \,\middle\|\,
                \mathbb{P}_{\mathcal{A},\mathcal{M}_2}
            \right)\\
            &\quad=
            \sum_{j=1}^{\infty}
            \mathbb{E}_{\mathcal{A},\mathcal{M}_1}
            \left[
                \mathbf{1}_{\{N_{\mathcal{A}}\geq j\}}K_j
            \right]\\
            &\quad\leq
            \frac{8\delta^2}{(1-\gamma)^2}
            \sum_{j=1}^{\infty}
            \mathbb{P}_{\mathcal{A},\mathcal{M}_1}
            (N_{\mathcal{A}}\geq j)\\
            &\quad=
            \frac{8\delta^2}{(1-\gamma)^2}
            \mathbb{E}_{\mathcal{A},\mathcal{M}_1}
            [N_{\mathcal{A}}]\\
            &\quad\leq
            \frac{8n\delta^2}{(1-\gamma)^2}.
        \end{align*}
        The chain-rule identities for random trajectory lengths and the complete interaction history follow by padding observations with zeros after termination, applying the chain rule to finite prefixes, and passing to the limit. The algorithm's internal randomization has the same distribution under both instances, while its action-selection, continuation, and output rules coincide conditional on the same history. Hence, only the environment transition dynamics contribute to the KL divergence.

    \textbf{Part VI: Sample Complexity Lower Bound.}
        Combining the results from Parts IV and V, any algorithm achieving an expected suboptimality gap bounded by $\epsilon$ on both instances must satisfy,
        \begin{align*}
            \frac12
            \leq
            D_{\mathrm{KL}}\!\left(
                \mathbb{P}_{\mathcal{A},\mathcal{M}_1}
                \,\middle\|\,
                \mathbb{P}_{\mathcal{A},\mathcal{M}_2}
            \right)
            \leq
            \frac{8n\delta^2}{(1-\gamma)^2}.
        \end{align*}
        This yields the trajectory budget requirement:
        \begin{align*}
            n\geq\frac{(1-\gamma)^2}{16\delta^2}.
        \end{align*}
        Substituting the perturbation parameter selected in Part III,
        \begin{align*}
            \delta
            =\frac{64\epsilon(1-\gamma)^2}
            {R_{\max}(1-2\alpha)},
        \end{align*}
        we obtain:
        \begin{align*}
            n
            &\geq
            \frac{(1-\gamma)^2}{16}
            \frac{R_{\max}^2(1-2\alpha)^2}
            {4096\epsilon^2(1-\gamma)^4}\\
            &=
            \frac{R_{\max}^2(1-2\alpha)^2}
            {65536\epsilon^2(1-\gamma)^2}.
        \end{align*}
        In particular, fixing $\alpha=1/4$ gives:
        \begin{align*}
            \max_{i\in\{1,2\}}
            \mathbb{E}_{\mathcal{A},\mathcal{M}_i}
            [N_{\mathcal{A}}]
            \geq
            \frac{R_{\max}^2}
            {262144\epsilon^2(1-\gamma)^2}.
        \end{align*}
        Thus, for $\gamma\in[2/3,1)$ and
        \begin{align*}
            0<\epsilon
            \leq\frac{R_{\max}}{256\gamma(1-\gamma)},
        \end{align*}
        any algorithm guaranteeing $\epsilon$-optimality on both possible unknown instances requires a worst-case expected number of trajectories satisfying
        \begin{align*}
            n=\Omega\!\left(
                \frac{R_{\max}^2}
                {(1-\gamma)^2\epsilon^2}
            \right).
        \end{align*}
        This establishes the claimed trajectory complexity lower bound for adaptive algorithms, including algorithms whose number of collected trajectories is random.
\end{proof}

\section{Technical Lemmas}
\label{apx:technical}

\begin{lemma}\label{lem:dini}
    Let $\bG:[0,1]\to\mathbb{R}^d$ be continuous on $[0,1]$. If, for any $u\in[0,1)$,
    \begin{equation*}
        \limsup_{\epsilon\downarrow 0}\frac{\|\bG(u+\epsilon)-\bG(u)\|}{\epsilon} \le C,
    \end{equation*}
    then $\|\bG(1)-\bG(0)\| \le C$.
\end{lemma}
\begin{proof}
    If $\bG(1) = \bG(0)$ there is nothing to prove, so assume $\|\bG(1)-\bG(0)\|>0$.
    By contradiction, suppose $\|\bG(1)-\bG(0)\| > C$ and choose $c\in\mathbb{R}$ such that $\|\bG(1)-\bG(0)\| > c > C$. Let 
    \begin{equation*}
        \bv\coloneqq\frac{\bG(1)-\bG(0)}{\|\bG(1)-\bG(0)\|},    
    \end{equation*}
    and define $f:[0,1]\to\mathbb{R}$ as $f(u) \coloneqq \bv^\top\bG(u)-cu$.
    Then, $f(1)-f(0)=\|\bG(1)-\bG(0)\|-c>0$. Since $f$ is continuous on $[0,1]$, it attains its minimum at some $u^\star\in[0,1]$. Since $f(1)>f(0)$, $u^\star<1$. For sufficiently small $\epsilon$, $f(u^\star+\epsilon)-f(u^\star)\ge 0$, so:
    \begin{equation*}
        \limsup_{\epsilon\downarrow 0}\frac{f(u^\star+\epsilon)-f(u^\star)}{\epsilon} \ge 0.
    \end{equation*}
    On the other hand, since $\|\bv\|=1$,
    \begin{align*}
        \limsup_{\epsilon\downarrow 0}\frac{f(u^\star+\epsilon)-f(u^\star)}{\epsilon}
        &=\limsup_{\epsilon\downarrow 0}\frac{\bv^\top(\bG(u^\star+\epsilon)-\bG(u^\star))}{\epsilon} -c \\
        &\le\limsup_{\epsilon\downarrow 0}\frac{\|\bG(u^\star+\epsilon)-\bG(u^\star)\|}{\epsilon} - c \\
        &\le C - c < 0.
    \end{align*}
\end{proof}

\begin{lemma}[Recurrence]
    \label{lemma:linear_recurrence}
    Let $\{x_k\}_{k=0}^{\infty}$ be a non-negative sequence satisfying the following recurrence relation for all $k \ge 0$:
    \begin{align*}
        x_{k+1} \le x_k - \min \left\{ a x_k^2, b x_k \right\} + e,
    \end{align*}
    where $a > 0$, $b > 0$, and $e > 0$. For a given target accuracy $\epsilon > 0$, define the effective contraction rate $\rho_\epsilon \coloneqq \min \{ a\epsilon, b \}$. Assume the algorithmic parameters are configured such that $\rho_\epsilon \in (0, 1)$. If the noise term $e$ satisfies $e \le a \frac{\epsilon^2}{2}$ and $e \le b \frac{\epsilon}{2}$, then, it suffices to choose a number of iterations $K$ satisfying:
    \begin{align*}
        K \ge
        \max\left\{\frac{1}{a\epsilon},\frac{1}{b}\right\}
        \log_{+}\left(\frac{2x_0}{\epsilon}\right),
    \end{align*}
    where $\log_{+}(a) \coloneqq \log \max \{1, a\}$, to ensure that $x_k \le \epsilon$ for every $k \ge K$.
\end{lemma}
\begin{proof}
    We first establish that the sequence remains within the target
    neighborhood once it reaches it. By the stated noise conditions,
    we have:
    \begin{align*}
        e \le
        \min\left\{\frac{a\epsilon^2}{2},
                   \frac{b\epsilon}{2}\right\}
        = \frac{\rho_\epsilon\epsilon}{2}.
    \end{align*}
    Suppose that $x_k \le \epsilon$. If $x_k \le \epsilon/2$,
    dropping the non-negative descent term yields:
    \begin{align*}
        x_{k+1}
        \le x_k + e
        \le \frac{\epsilon}{2}
            + \frac{\rho_\epsilon\epsilon}{2}
        \le \epsilon,
    \end{align*}
    where the final inequality follows from $\rho_\epsilon \in (0,1)$.
    Conversely, if $\epsilon/2 < x_k \le \epsilon$, we lower-bound
    the descent term as follows:
    \begin{align*}
        \min\{ax_k^2,bx_k\}
        &= x_k\min\{ax_k,b\} \\
        &\ge x_k\min\left\{\frac{a\epsilon}{2},b\right\} \\
        &\ge \frac{\rho_\epsilon}{2}x_k.
    \end{align*}
    Substituting this bound into the recurrence gives:
    \begin{align*}
        x_{k+1}
        &\le \left(1-\frac{\rho_\epsilon}{2}\right)x_k + e \\
        &\le \left(1-\frac{\rho_\epsilon}{2}\right)\epsilon
              + \frac{\rho_\epsilon\epsilon}{2}
        = \epsilon.
    \end{align*}
    Consequently, the target neighborhood $[0,\epsilon]$ is invariant:
    if $x_k \le \epsilon$ at any iteration, then $x_j \le \epsilon$
    for every $j \ge k$. We now evaluate the sequence using a hitting-time argument.
    Let $T_\epsilon := \inf\{k \ge 0 \mid x_k \le \epsilon\}$
    denote the first iteration at which the sequence reaches the
    target accuracy, with $\inf\varnothing := +\infty$.
    If $x_0 \le \epsilon$, the invariance established above directly
    guarantees the claim for every iteration.
    Assume therefore that $x_0 > \epsilon$.
    By the definition of the hitting time, $x_k > \epsilon$
    for every $k < T_\epsilon$. This allows us to lower-bound
    the descent rate prior to reaching the target neighborhood.
    For any $k < T_\epsilon$, factoring out $x_k$ yields:
    \begin{align*}
        \min \left\{ a x_k^2, b x_k \right\} = x_k \min \left\{ a x_k, b \right\} \ge x_k \min \left\{ a \epsilon, b \right\}.
    \end{align*}
    Substituting the definition of the effective contraction rate $\rho_\epsilon \coloneqq \min \{ a\epsilon, b \}$ into the original inequality yields a standard linear recurrence:
    \begin{align*}
        x_{k+1} \le (1 - \rho_\epsilon)x_k + e.
    \end{align*}
    
    To set the hyperparameters so that the target $x_K \le \epsilon$ is reached, we partition the error budget equally, requiring the asymptotic bias to be bounded by $\frac{\epsilon}{2}$ and the transient term to be bounded by $\frac{\epsilon}{2}$. Bounding the asymptotic bias requires $\frac{e}{\rho_\epsilon} \le \frac{\epsilon}{2}$, which rearranges to $e \le \rho_\epsilon \frac{\epsilon}{2}$. Because $\rho_\epsilon$ evaluates to the minimum of two regimes, this condition must hold for both possible values. If $\rho_\epsilon = a\epsilon$, the constraint dictates $e \le a\frac{\epsilon^2}{2}$. If $\rho_\epsilon = b$, the constraint dictates $e \le b\frac{\epsilon}{2}$. Satisfying both of these stated hyperparameter conditions ensures the asymptotic bias remains strictly bounded by $\frac{\epsilon}{2}$. Consequently, as long as $x_k \ge \epsilon > 0$, we can rewrite the linear recurrence as:
    \begin{align*}
        x_{k+1} - \frac{\epsilon}{2} \le (1 - \rho_\epsilon)\left(x_k - \frac{\epsilon}{2}\right).
    \end{align*}
    For any integer $K < T_\epsilon$, unrolling this recursive
    inequality over $K$ steps yields:
    \begin{align*}
        x_K
        &\le \frac{\epsilon}{2}
              + (1-\rho_\epsilon)^K
                \left(x_0-\frac{\epsilon}{2}\right) \\
        &\le \frac{\epsilon}{2}
              + \exp(-\rho_\epsilon K)x_0,
    \end{align*}
    where the final bound applies the standard logarithmic inequality
    $1-x \le \exp(-x)$, valid for $\rho_\epsilon \in (0,1)$.
    Bounding the remaining transient term by $\epsilon/2$ requires:
    \begin{align*}
        K
        &\ge \frac{1}{\rho_\epsilon}
              \log_{+}\left(\frac{2x_0}{\epsilon}\right) \\
        &= \max\left\{\frac{1}{a\epsilon},\frac{1}{b}\right\}
           \log_{+}\left(\frac{2x_0}{\epsilon}\right).
    \end{align*}
    Fix an integer $K$ satisfying this bound. If $T_\epsilon > K$,
    the preceding estimate gives $x_K \le \epsilon$, contradicting
    the definition of $T_\epsilon$. Hence, $T_\epsilon \le K$.
    By the invariance of the target neighborhood, it follows that
    $x_k \le \epsilon$ for every $k \ge K$, completing the proof.
\end{proof}

\section{Experimental Details}
\label{apx:exp}

We evaluate the empirical sample efficiency of the \alg framework on CartPole-v1, Swimmer-v4, HalfCheetah-v4, and Hopper-v4. The experiments have two objectives: \textit{(i)} to compare the framework's configurations with \gpomdp~\citep{baxter2001infinite} and \texttt{N-PG-IGT}~\citep{fatkhullin2023stochastic}, and \textit{(ii)} to examine sensitivity to the recursive-update parameter $\alpha$ and refresh probability $p$. We first describe the experimental protocol, then present the environment comparisons, the gradient-variability diagnostic, and the hyperparameter studies.

\subsection*{Policies and Implementation}

CartPole-v1 uses a softmax policy with two hidden layers of width $4$. The continuous-control tasks use Gaussian policies with two hidden layers of width $32$. All hidden layers use $\tanh$ activations, with Xavier-uniform weight initialization and zero biases. Gaussian policies learn a scalar log standard deviation shared across action coordinates, initialized to the values in Table~\ref{tab:environment_hyperparams}.

The experiments use fixed trajectory horizons. Return comparisons and hyperparameter studies use Adam, whereas the gradient-variability diagnostic uses a constant scalar step size. Training and evaluation use separate random-number streams. Evaluation returns are undiscounted and obtained from stochastic-policy rollouts; evaluation trajectories are excluded from the training sample count. For the environment comparisons, each evaluation averages $10$ episodes, with evaluations every $10$ collected training trajectories and at the final updated policy.

\begin{table*}[t]
    \centering
    \renewcommand{\arraystretch}{1.3}
    \caption{Mean normalized area under the learning curve (AUC) across the four environments; the $\pm$ indicate the estimated standard error of the mean. Bold entries indicate the largest mean AUC in each environment.}
    \begin{tabular}{|l|ccccc|}
        \hline
        \rowcolor{lightBlue!40}
        \multicolumn{1}{|c|}{\textbf{Environment}} & \gpomdp & \defstormpg & \defpagepg & \alg & \texttt{N-PG-IGT} \\
        \hline
        Cart-Pole & $396.05 \pm 34.56$ & $481.90 \pm 7.85$ & $417.34 \pm 37.96$ & $\mathbf{482.39 \pm 5.49}$ & $380.03 \pm 39.46$ \\
        Swimmer & $24.19 \pm 6.15$ & $\mathbf{40.86 \pm 7.94}$ & $26.39 \pm 5.46$ & $33.90 \pm 13.28$ & $26.22 \pm 1.01$ \\
        Half-Cheetah & $13.49 \pm 30.16$ & $\mathbf{105.76 \pm 51.57}$ & $24.08 \pm 32.67$ & $55.82 \pm 41.76$ & $87.67 \pm 23.31$ \\
        Hopper & $201.90 \pm 21.68$ & $\mathbf{369.94 \pm 89.62}$ & $233.65 \pm 28.36$ & $309.13 \pm 44.92$ & $368.81 \pm 23.97$ \\
        \hline
    \end{tabular}
    
    \label{tab:auc_results}
\end{table*}

\begin{table*}[t]
    \centering
    \renewcommand{\arraystretch}{1.3}
    \caption{Environment-specific hyperparameter choices for the experiments.}
    \begin{tabular}{|l|ccccc|}
        \hline
        \rowcolor{lightBlue!40}
        \textbf{Environment} & Horizon $T$ & Batch $N$ & Mini-batch $B$ & Learning rate $\eta$ & Initial std $\sigma$ \\
        \hline
        CartPole-v1 & $500$ & $50$ & $5$ & $10^{-2}$ & -- \\
        Swimmer-v4 & $500$ & $100$ & $10$ & $10^{-3}$ & $1$ \\
        HalfCheetah-v4 & $100$ & $256$ & $16$ & $10^{-4}$ & $\sqrt{0.3}$ \\
        Hopper-v4 & $100$ & $256$ & $16$ & $10^{-4}$ & $0.3$ \\
        \hline
    \end{tabular}
    \label{tab:environment_hyperparams}
\end{table*}

\subsection*{Algorithms}

We consider \defstormpg ($\alpha>0$, $p=0$), \defpagepg ($\alpha=0$, $p>0$), and a hybrid configuration ($\alpha>0$, $p>0$). The first two are the defensive-IS versions of \texttt{STORM-PG}~\citep{yuan2020stochastic} and \texttt{PAGE-PG}~\citep{gargiani2022page}, respectively. In the comparison tables and figures, \alg denotes the best hybrid configuration, hence the configuration of $\alpha$ and $p$ that allow \alg to obtain the maximum performance (please refer to Section~\ref{apx:grid_search} to an example of ablation study). The parameter $p$ determines the probability of refreshing the gradient estimator with $N$ new trajectories. The parameter $\alpha$ controls the recursive momentum update; increasing $\alpha$ therefore reduces the weight of this correction. All three configurations use defensive IS with $\xi=0.5$, as suggested by the theory.

We compare these configurations with \gpomdp~\citep{baxter2001infinite} and \texttt{N-PG-IGT}~\cite{fatkhullin2023stochastic}. \gpomdp computes a fresh on-policy gradient estimate at each update. \texttt{N-PG-IGT} combines normalized updates with implicit gradient transport, using neither importance sampling nor second-order information. The algorithms use an average baseline for gradient estimation. Table~\ref{tab:unified_hyperparams} summarizes the algorithmic settings used in the environment comparisons.

\begin{table*}[t]
    \centering
    \renewcommand{\arraystretch}{1.3}
    \caption{Algorithmic settings for the return comparisons. ($^*$The hybrid uses $p=0.5$ on CartPole-v1.) For \texttt{N-PG-IGT}, $\alpha$ denotes its gradient-transport momentum parameter.}
    \begin{tabular}{|l|ccccc|}
        \hline
        \rowcolor{lightBlue!40}
        \multicolumn{1}{|c|}{\textbf{Hyperparameter}} & \gpomdp & \defstormpg & \defpagepg & \alg & \texttt{N-PG-IGT} \\
        \hline
        Refresh probability $p$ & -- & $0$ & $0.5$ & $0.1^*$ & -- \\
        Recursive parameter $\alpha$ & -- & $0.9$ & $0$ & $0.9$ & $0.9$ \\
        Defensive parameter $\xi$ & -- & $0.5$ & $0.5$ & $0.5$ & -- \\
        \hline
    \end{tabular}
    \label{tab:unified_hyperparams}
\end{table*}

\subsection*{Performance Measures}

Learning curves report mean evaluation returns over five independent runs.  Shaded regions indicate the estimated standard error of the mean. We summarize performance over the training budget using normalized Area under the Curve (AUC). Higher AUC indicates that high returns are reached earlier and maintained over a larger part of the budget. Table~\ref{tab:auc_results} reports the mean AUC and its confidence interval, showing that the momentum-only configuration of \alg performs consistently better than the competitors in almost all the environments.

\subsection{CartPole-v1}

CartPole-v1 has a four-dimensional state space ($\mathcal{S}\subset\mathbb{R}^4$) and two discrete actions ($\mathcal{A}=\{0,1\}$). Figure~\ref{fig:cart_pole_eval} compares the five algorithms using the settings in Tables~\ref{tab:environment_hyperparams} and~\ref{tab:unified_hyperparams}.

\begin{figure}[h]
    \centering
    \includegraphics[width=0.48\linewidth]{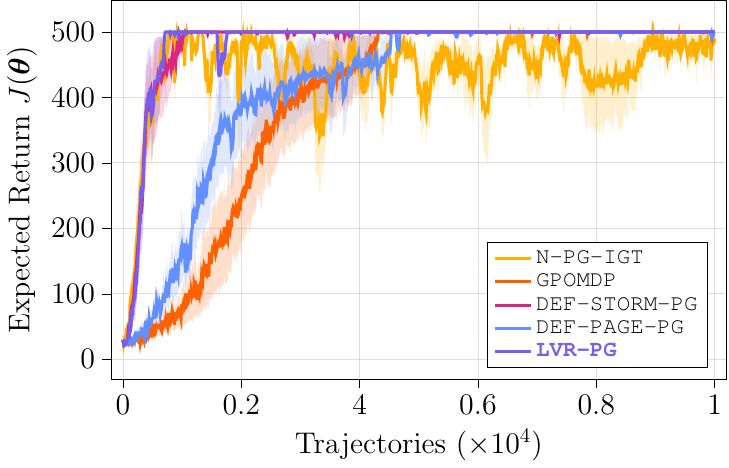}
    \caption{Learning curves on CartPole-v1. Curves show mean evaluation returns over five runs;  shaded regions indicate the estimated standard error of the mean.}
    \label{fig:cart_pole_eval}
\end{figure}

\defstormpg and \alg reach returns close to the maximum episodic return of $500$ after approximately $10^3$ trajectories. Both improve more rapidly than \gpomdp and \defpagepg. This shows that, empirically, momentum behaves better than probabilistic refreshes. \texttt{N-PG-IGT} attains high returns but exhibits larger training variance.

\subsection{Swimmer-v4}

Swimmer-v4 is a locomotion task with an eight-dimensional state space ($\mathcal{S}\subset\mathbb{R}^8$) and a two-dimensional action space ($\mathcal{A}\subset\mathbb{R}^2$). The policy controls an articulated swimmer moving through a fluid environment.

\begin{figure}[h]
    \centering
    \includegraphics[width=0.48\linewidth]{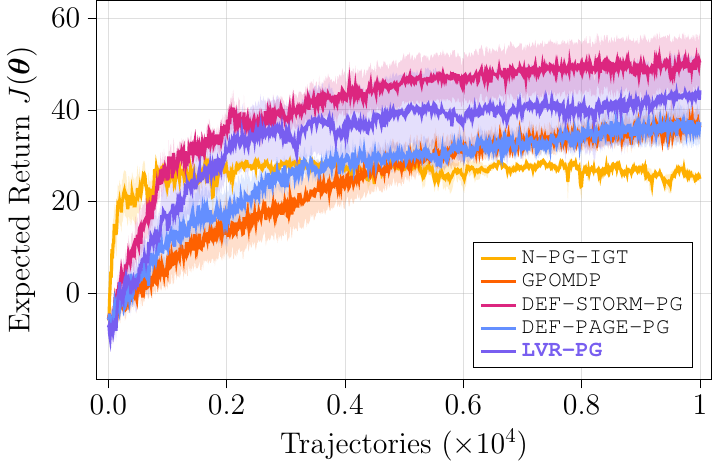}
    \caption{Learning curves on Swimmer-v4. Curves show mean evaluation returns over five runs; shaded regions indicate the estimated standard error of the mean.}
    \label{fig:swimmer_eval}
\end{figure}

Figure~\ref{fig:swimmer_eval} shows that \defstormpg achieves the highest returns over the evaluated training budget. In this environment, the momentum-only configuration of \alg provides the strongest empirical sample efficiency among the tested configurations. While learning faster in the initial part of the training, \texttt{N-PG-IGT} gets stuck in a suboptimal solution around a cumulative reward of $27$.

\subsection{HalfCheetah-v4}

HalfCheetah-v4 has a $17$-dimensional state space ($\mathcal{S}\subset\mathbb{R}^{17}$) and a six-dimensional action space ($\mathcal{A}\subset\mathbb{R}^6$). The policy controls the locomotion of an articulated agent.

\begin{figure}[h]
    \centering
    \includegraphics[width=0.48\linewidth]{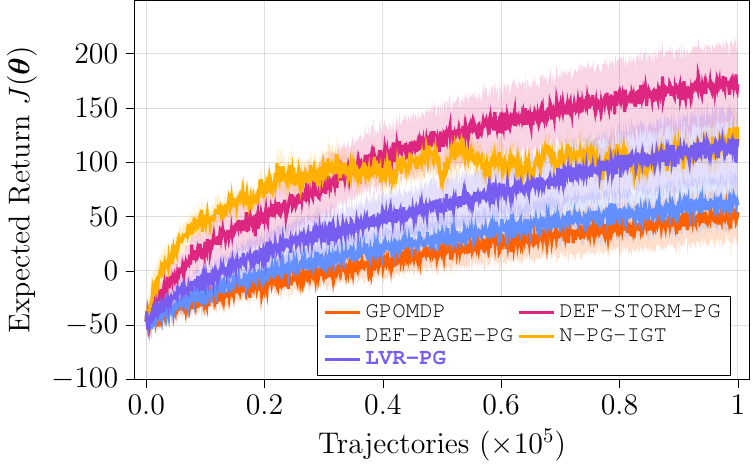}
    \caption{Learning curves on HalfCheetah-v4. Curves show mean evaluation returns over five runs;  shaded regions indicate the estimated standard error of the mean.}
    \label{fig:half_cheetah_eval}
\end{figure}

In Figure~\ref{fig:half_cheetah_eval}, \defstormpg reaches the best return compared to the other algorithms. Again, \texttt{N-PG-IGT} learns faster in the first part of the learning, but early saturates to a lower cumulative reward compared to \defstormpg. The experiment also shows that  \defpagepg, and \gpomdp learn slower in the environment, and achieve a lower cumulative return after observing $10^5$ trajectories.

\subsection{Hopper-v4}

Hopper-v4 is a locomotion task with an $11$-dimensional state space ($\mathcal{S}\subset\mathbb{R}^{11}$) and a three-dimensional action space ($\mathcal{A}\subset\mathbb{R}^3$).

\begin{figure}[h]
    \centering
    \includegraphics[width=0.48\linewidth]{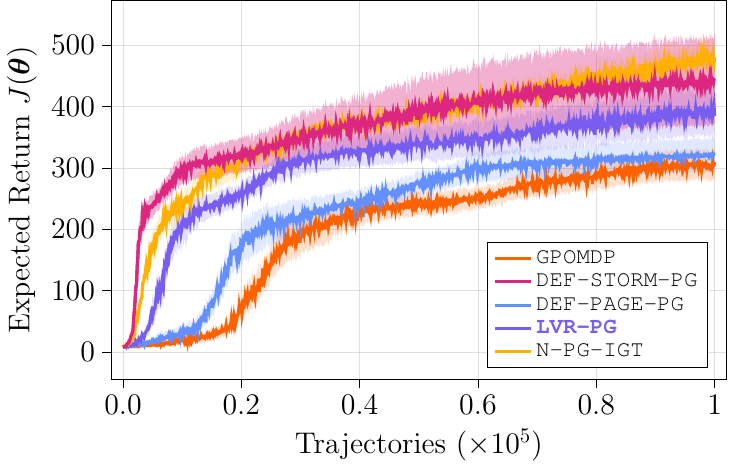}
    \caption{Learning curves on Hopper-v4. Curves show mean evaluation returns over five runs;  shaded regions indicate the estimated standard error of the mean.}
    \label{fig:hopper_eval}
\end{figure}

Figure~\ref{fig:hopper_eval} shows that \defstormpg and \texttt{N-PG-IGT} achieve better performance compared to the other algorithms. Again, this shows that momentum updates are empirically more efficient. Here, we see how \defpagepg and \gpomdp learn consistently slower than the algorithms that implement momentum, as they start collecting reward after observing $2\times 10^4$ trajectories.

\subsection{Gradient-Estimator Variability on CartPole-v1}

We complement the return comparisons with a diagnostic of per-trajectory gradient variability on CartPole-v1. For a batch of $m_k$ trajectories, let $\mathbf{g}_{k,i}$ denote the unweighted, baseline-free \gpomdp contribution evaluated at the current policy parameters. We compute the trace of their sample covariance,
\begin{align*}
    \widehat{\mathbb{V}\text{ar}}\left(\estGradK\right)
    &= \frac{1}{m_k-1}\sum_{i=1}^{m_k}
       \left\|\mathbf{g}_{k,i}-\overline{\mathbf{g}}_k\right\|_2^2,
    &\overline{\mathbf{g}}_k
    &= \frac{1}{m_k}\sum_{i=1}^{m_k}\mathbf{g}_{k,i}.
\end{align*}
We plot the mean of $\log_{10}\widehat{\mathbb{V}\text{ar}}\left(\estGradK\right)$ across independent runs. The contributions are evaluated at the current policy parameters using the batches collected by each algorithm: on-policy batches for \gpomdp and \texttt{N-PG-IGT}, and defensive-mixture batches for the configurations of \alg. Thus, the statistic summarizes the variability of the per-trajectory gradient contributions used in this diagnostic. The experiment uses the configurations of Tables~\ref{tab:environment_hyperparams} and~\ref{tab:unified_hyperparams}.

\begin{figure}[h]
    \centering
    \includegraphics[width=0.48\linewidth]{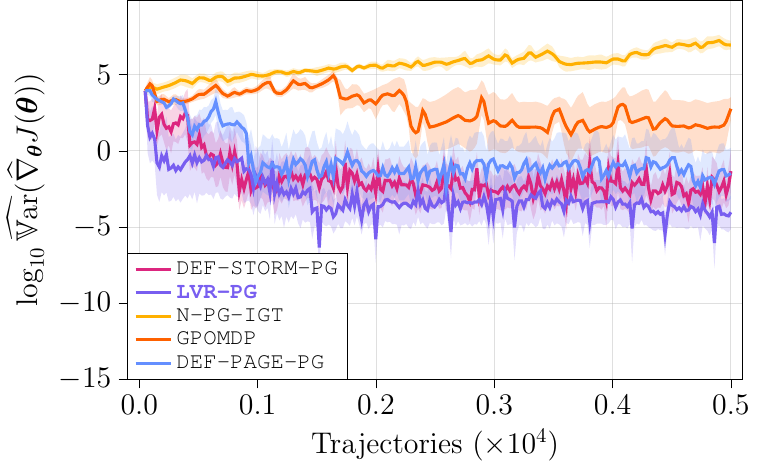}
    \caption{Per-trajectory gradient-variability diagnostic on CartPole-v1 over $5\times10^3$ training trajectories. Values are shown on a base-$10$ logarithmic scale and averaged over five independent runs; shaded regions indicate the estimated standard error of the mean.}
    \label{fig:cartpole_variance_new}
\end{figure}

Figure~\ref{fig:cartpole_variance_new} shows lower per-trajectory gradient variability for \defstormpg and the hybrid \alg configuration than for the baselines over much of training. \defpagepg also achieves lower values than \gpomdp, while the values for \texttt{N-PG-IGT} increase over the observed budget. 

\subsection{Hyperparameter Ablations on CartPole-v1}
\label{subsec:cartpole}

We examine the sensitivity of \defpagepg to $p\in\{0.1,0.5,0.9\}$ and of \defstormpg to $\alpha\in\{0.1,0.5,0.9\}$. Both studies use the CartPole-v1 environment with the configuration presented in Tables~\ref{tab:environment_hyperparams} and~\ref{tab:unified_hyperparams}.

\begin{figure*}[h]
    \centering
    \includegraphics[width=0.48\linewidth]{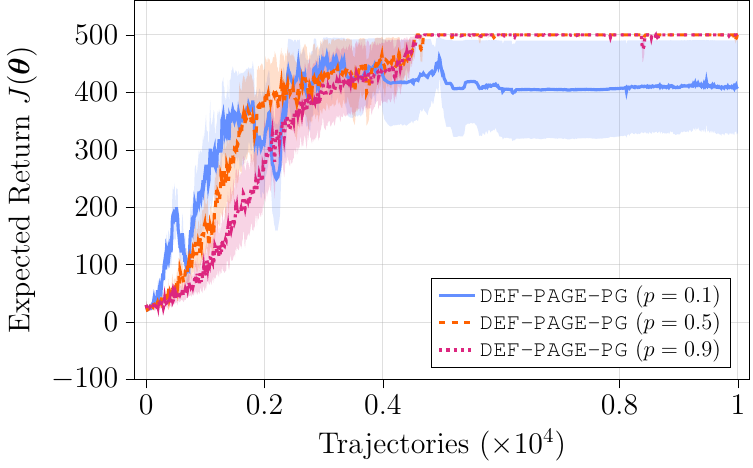}
    \hfill
    \includegraphics[width=0.48\linewidth]{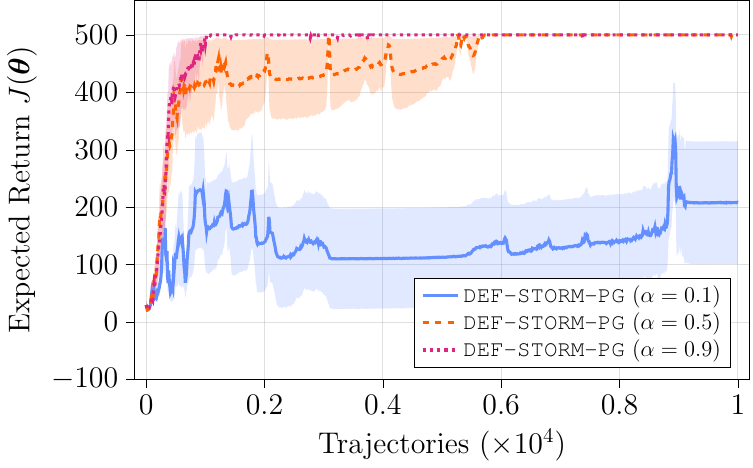}
    \caption{Hyperparameter ablations on CartPole-v1 over $10^4$ training trajectories. \textbf{Left:} Refresh probability $p\in\{0.1,0.5,0.9\}$ for \defpagepg. \textbf{Right:} Recursive-update parameter $\alpha\in\{0.1,0.5,0.9\}$ for \defstormpg. Curves show mean evaluation returns over five runs, shaded regions indicate the estimated standard error of the mean.}
    \label{fig:ablation_cartpole}
\end{figure*}

Figure~\ref{fig:ablation_cartpole} (left) shows that \defpagepg is not much less sensitive to the choice of $p$, although frequent refreshes improve its final performance. Both $p=0.5$ and $p=0.9$ reach returns close to $500$, whereas $p=0.1$ settles near $400$. The configuration $p=0.9$ maintains high returns while selecting recursive mini-batch updates with small probability. This suggests that, without momentum, relying heavily on recursive mini-batch updates is less effective in the tested setting. The improvement with more frequent full-batch refreshes is consistent with using larger on-policy batches to control gradient-estimation variance.

Figure~\ref{fig:ablation_cartpole} (right) shows that \defstormpg is very sensitive to the choice of $\alpha$. The configuration $\alpha=0.9$ rapidly reaches returns close to $500$, while $\alpha=0.5$ requires substantially more trajectories, and $\alpha=0.1$ exhibits persistent fluctuations and lower returns, failing to solve the task in the environment. Thus, a large $\alpha$ substantially improves empirical sample efficiency in this setting. Since the historical correction is weighted by $1-\alpha$, the best-performing configuration retains a smaller correction from the previous estimate.

\subsection{Hyperparameter Grid Search for \alg on CartPole-v1}
\label{apx:grid_search}

We evaluate hybrid configurations of \alg with $\alpha\in\{0.01,0.5,0.9\}$ and $p\in\{0.05,0.3,0.5\}$. The policy architecture and training settings match the ablation study in Section~\ref{subsec:cartpole}. Each configuration uses five independent runs.  Figure~\ref{fig:alg_heatmap} summarizes normalized AUC and terminal return, with terminal return computed as the mean of the last three evaluations in each run.

\begin{figure*}[h]
    \centering
    \resizebox{\textwidth}{!}{%
    \begin{tikzpicture}[
        x=2cm, y=2cm,
        cell/.style={rectangle, minimum size=2cm, align=center, font=\small, text=black},
        whitecell/.style={cell, text=white}
    ]
    
    \definecolor{lvr_col}{RGB}{120,94,240}
    \definecolor{defstorm_col}{RGB}{220,38,127}
    
    \node at (1, 2.9) {\textbf{Panel (a):} Normalized AUC.};
    
    \node[rotate=90] at (-1.1, 1) {Refresh Probability ($p$)};
    \node at (1, -0.9) {Momentum parameter ($\alpha$)};
    
    \node at (-0.65, 2) {$0.05$};
    \node at (-0.65, 1) {$0.3$};
    \node at (-0.65, 0) {$0.5$};
    
    \node at (0, -0.6) {$0.01$};
    \node at (1, -0.6) {$0.5$};
    \node at (2, -0.6) {$0.9$};
    
    \node[cell, fill=lvr_col!3] at (0,2) {$12.7$ \\ $\pm 2.4$};
    \node[cell, fill=lvr_col!43] at (1,2) {$212.8$ \\ $\pm 131.3$};
    \node[whitecell, fill=lvr_col!76] at (2,2) {$379.3$ \\ $\pm 60.9$};
    
    \node[whitecell, fill=lvr_col!69] at (0,1) {$343.9$ \\ $\pm 48.8$};
    \node[whitecell, fill=lvr_col!84] at (1,1) {$422.3$ \\ $\pm 26.5$};
    \node[whitecell, fill=lvr_col!88] at (2,1) {$438.8$ \\ $\pm 20.1$};
    
    \node[whitecell, fill=lvr_col!83] at (0,0) {$415.1$ \\ $\pm 55.3$};
    \node[whitecell, fill=lvr_col!80] at (1,0) {$399.4$ \\ $\pm 95.9$};
    \node[whitecell, fill=lvr_col!91] at (2,0) {$453.5$ \\ $\pm 7.2$};

    \begin{scope}[xshift=9.5cm]
        \node at (1, 2.9) {\textbf{Panel (b):} Last-Iterate Return};
        
        \node[rotate=90] at (-1.1, 1) {Refresh Probability ($p$)};
        \node at (1, -0.9) {Momentum parameter ($\alpha$)};
        
        \node at (-0.65, 2) {$0.05$};
        \node at (-0.65, 1) {$0.3$};
        \node at (-0.65, 0) {$0.5$};
        
        \node at (0, -0.6) {$0.01$};
        \node at (1, -0.6) {$0.5$};
        \node at (2, -0.6) {$0.9$};
        
        \node[cell, fill=defstorm_col!2] at (0,2) {$9.3$ \\ $\pm 0.2$};
        \node[cell, fill=defstorm_col!26] at (1,2) {$129.1$ \\ $\pm 186.3$};
        \node[whitecell, fill=defstorm_col!67] at (2,2) {$336.4$ \\ $\pm 202.4$};
        
        \node[whitecell, fill=defstorm_col!76] at (0,1) {$381.9$ \\ $\pm 176.2$};
        \node[whitecell, fill=defstorm_col!100] at (1,1) {$500.0$ \\ $\pm 0.0$};
        \node[whitecell, fill=defstorm_col!98] at (2,1) {$490.3$ \\ $\pm 19.0$};
        
        \node[whitecell, fill=defstorm_col!100] at (0,0) {$500.0$ \\ $\pm 0.0$};
        \node[whitecell, fill=defstorm_col!85] at (1,0) {$422.8$ \\ $\pm 151.3$};
        \node[whitecell, fill=defstorm_col!100] at (2,0) {$500.0$ \\ $\pm 0.0$};
    \end{scope}

    \end{tikzpicture}%
    }
    \caption{Hybrid hyperparameter grid on CartPole-v1. \textbf{Left:} Mean normalized AUC. \textbf{Right:} Mean terminal return, averaged over the last three evaluations in each run. Entries summarize five independent runs; shaded regions indicate the estimated standard error of the mean.}
    \label{fig:alg_heatmap}
\end{figure*}
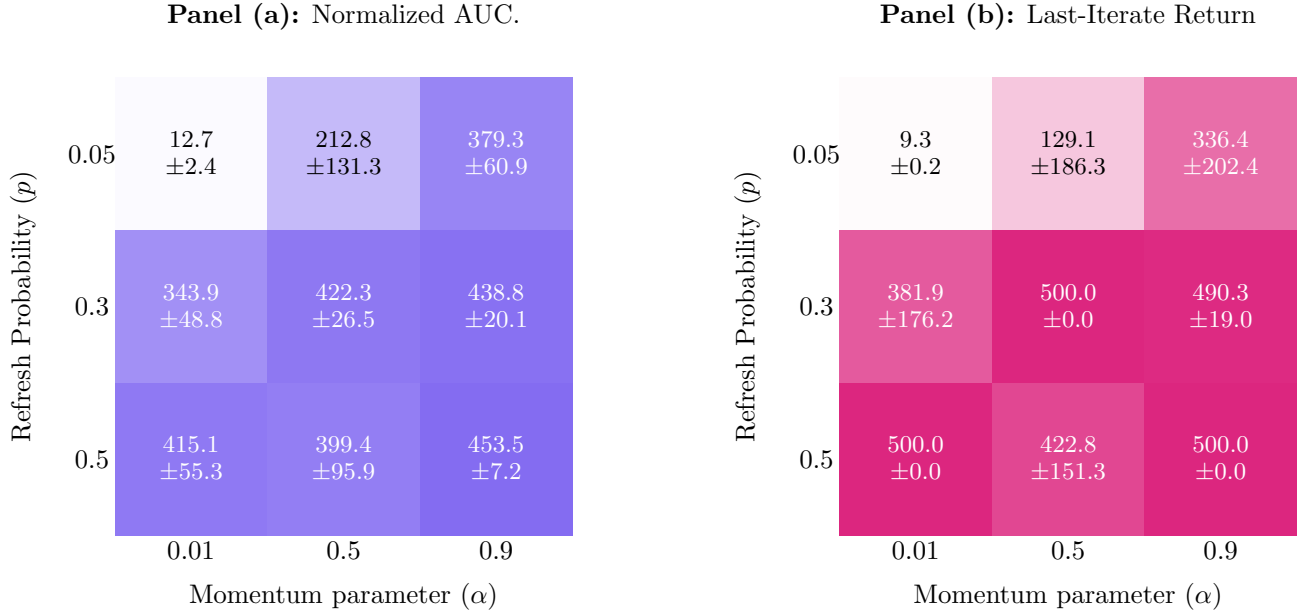

The grid shows sensitivity to both parameters. At $p=0.05$, increasing $\alpha$ raises the mean AUC from $12.7$ to $379.3$ and the mean terminal return from $9.3$ to $336.4$. More frequent refreshes substantially improve the low-$\alpha$ configuration. Three configurations, $(0.5,0.3)$, $(0.01,0.5)$, and $(0.9,0.5)$, attain a mean terminal return of $500$ at the displayed precision, while their mean AUC values differ. The configuration $(\alpha,p)=(0.9,0.5)$ achieves the largest mean AUC in the grid, $453.5\pm7.2$, combining high terminal performance with the strongest performance over the training budget.

These results identify effective choices within the hybrid configuration and show that its performance depends on the interaction between $\alpha$ and $p$. Together with the experiments presented in Appendix~\ref{subsec:cartpole}, they motivate selecting \defstormpg as the primary representative of \alg, since it achieves consistently high mean AUC across the tested environments without periodic full-batch refreshes.

\subsection{Computational Resources}
All the experiments were run on a 2023 14-inch MacBook Pro equipped with an Apple M2 Pro chip and $16~\text{GB}$ of RAM.

\end{document}